%% file: ms.tex
\documentclass[11pt]{article}
\usepackage{latexsym,amsfonts,amsmath,amsthm,amssymb, epsfig}
\usepackage{color}
\usepackage{mathtools}
\usepackage{bbm}
\usepackage{geometry}
\usepackage{enumitem}
\usepackage{subcaption}
\usepackage{tikz}
\usepackage{multicol}
\usepackage{booktabs}
\usepackage{hyperref}
\usepackage[nocompress]{cite}
\newtheorem{thm}{Theorem}
\newtheorem{lem}[thm]{Lemma}
\newtheorem{prop}[thm]{Proposition}
\newtheorem{cor}{Corollary}

\theoremstyle{plain}
\newtheorem{rem}{Remark}

\newtheorem{dfn}[thm]{Definition}

\usepackage[linesnumbered,algoruled,boxed,lined]{algorithm2e}
\makeatletter
\g@addto@macro{\@algocf@init}{\SetKwInOut{Parameter}{Parameters}}
\makeatother

\title{Stable Recovery of Entangled
Weights: Towards Robust Identification of Deep Neural Networks from Minimal Samples}
\author{Christian Fiedler$^{1,2}$, Massimo Fornasier$^3$, Timo Klock$^{4,5}$, and Michael Rauchensteiner$^6$}
\date{%
    $^1$Intelligent Control Systems Group, Max Planck Institute for Intelligent Systems, Heisenbergstr. 3, 70569, Stuttgart, Germany\\ Email: \href{mailto:fiedler@is.mpg.de}{fiedler@is.mpg.de}\\%
    $^2$Institute for Data Science in Mechanical Engineering, RWTH-Aachen University, Dennewartstra\ss e 27, 50868, Aachen, Germany\\%
    $^3$Department of Mathematics, Bolzmannstra\ss e 3, 85748, Garching, Germany,\\ Email: \href{mailto:massimo.fornasier@ma.tum.de}{massimo.fornasier@ma.tum.de}\\%
    $^4$Department of Numerical Analysis and Scientific Computing, Simula Research Laboratory, Oslo, Norway, \\Email: \href{mailto:timo@simula.no}{timo@simula.no}\\%
    $^5$Department of Mathematics, University of San Diego, California, San Diego, US\\
    $^6$Department of Mathematics, Bolzmannstra\ss e 3, 85748, Garching, Germany, \\Email: \href{mailto:michael.rauchensteiner@ma.tum.de}{michael.rauchensteiner@ma.tum.de}\\[2ex]%
    \today
}
\input{Macros.tex}

\newcommand{\Maf}{\kappa} 

\newcommand{\numsamples}{N_H}

\newcommand{\auxspace}{\mathcal{W}}

\begin{document}
\maketitle
\begin{abstract}
In this paper we approach the problem of unique and stable identifiability of generic deep artificial neural networks with pyramidal shape and smooth activation functions from a finite number of input-output samples.
More specifically we introduce the so-called {\it entangled weights},
which compose weights of successive layers intertwined with suitable diagonal and invertible matrices depending on the activation functions and their shifts.
We prove that entangled weights are completely and stably approximated by an efficient and robust algorithm as soon as
$\mathcal O(D^2 \times m)$ nonadaptive input-output samples of the network are collected,
where $D$ is the input dimension and $m$ is the number of neurons of the network. Moreover,
we empirically observe that the approach applies to networks with up to $\mathcal O(D \times m_L)$ neurons, where $m_L$ is the number of output neurons at layer $L$.
Provided knowledge of layer assignments of entangled weights and of remaining scaling and shift parameters, which may be further heuristically obtained by least squares, the entangled weights identify the network completely and uniquely.
To highlight the relevance of the theoretical result of stable recovery of entangled weights, we present numerical experiments, which demonstrate that multilayered networks with generic weights can be robustly identified and therefore {\it uniformly} approximated by the presented algorithmic pipeline. In contrast backpropagation  cannot generalize stably very well in this setting, being always limited by relatively large uniform error. In terms of practical impact, our study shows that we can relate input-output information uniquely and stably to network parameters, providing a form of explainability. Moreover, our method paves the way for compression of overparametrized networks and for the training of minimal complexity networks.
\end{abstract}

\tableofcontents

\section{Introduction}
Deep learning has become an extremely successful approach, performing state-of-the-art on various applications such as speech recognition \cite{DBLP:journals/corr/HannunCCCDEPSSCN14}, image recognition \cite{NIPS2012_4824,7780459}, language translation \cite{NIPS2017_7181}, and as a novel method for scientific computing \cite{Berner2020,elbraechter2020dnn}. Also in unsupervised machine learning, deep neural networks have shown great success, for instance in image and speech generation \cite{pmlr-v48-oord16,oord2016wavenet}, and in reinforcement learning for solving control problems, such as mastering Atari games \cite{nature15} or beating human champions in playing Go \cite{silver2017mastering}.
Deep learning is about realizing complex tasks as the ones mentioned above, by means of highly parametrized functions, called deep artificial neural networks $f:\mathbb R^D \to \mathbb R^{m_L}$. In this paper we consider classical feed-forward artificial neural networks of the type
\begin{equation}\label{nndef}
f(x)=g_{L}(W_L^\top g_{L-1}(W_{L-1}^\top \dots g_1(W_1^\top x)\dots),
\end{equation} where
the functions $g_\ell(\cdot)=g(\cdot + \tau_\ell)$, $\tau_\ell \in \mathbb R^{m_\ell}$, are sufficiently smooth and shifted activation functions and the matrices
$W_{\ell} \in \bbR^{ m_{\ell-1} \times m_\ell}$ collect the weights of each layer $\ell = 1,\ldots,L$, see Definition \ref{dfn:feedforward_networks} below for a more precise definition.
In practical applications, the number of layers $L$, determining the depth of the network, and the dimensions $m_{\ell-1}\times m_\ell$ of the weight matrices $W_\ell$ are typically determined through heuristic
considerations, whereas the weight matrices and the shifts  are learned based
on training data.

In support of deep learning comes the practical evidence of being able to  outperform other methods, but also the recent theoretical discoveries \cite{mhaskar2020function,shaham2018provable,Berner2020,elbraechter2020dnn,elbrachter2020deep,daubechies2019,PETERSEN2018296,cloninger2020relu,devore2020neural,mhaskar2016deep} that show that deep artificial network can approximate  high dimensional functions without incurring in the curse of dimensionality, i.e., one does not need an exponential number of parameters (weights of the network) with respect to the input dimension in order to approximate  high-dimensional functions. While the approximation properties of neural networks are becoming more understood and transparent, the training phase based on suitable optimization processes remains a (black-)box with some level of opacity. Recent results are shedding some light on this important phase of the employment of neural networks, at least in some simple cases, e.g., of linear neural networks or shallow neural networks, \cite{arora2019convergence,NIPS2019_8960,bah2019,MeiE7665,zhong2017recovery}.

\subsection{Backpropagation and fundamental issues of training}\label{bpg}

Backpropagation of error \cite{werbos74,10.5555/104279.104293,7fa6b6a5cde14bcfbd7ab3a8f19d0d56} is the most frequently used algorithm to train deep neural networks and is justified by its
tremendous empirical success. All the practical advances recalled above are due to the efficacy of this method.  Backpropagation refers to employing stochastic gradient descent or its variations  \cite{sun2019optimization} to minimize certain loss functions (e.g. least squares, Kullback-Leibler divergence, or
Wasserstein distances) of the weights\footnote{More precisely ``backpropagation'' is a recursive way of application of the chain rule needed to compute the gradient of the loss with respect to weights, but the term ``backpropagation'' is often used also to describe any algorithmic optimization which uses such gradients. In many cases such gradients are by now computed by symbolic calculus.}, usually measuring the misfit of input-output information over a finite number of labeled training samples.
The practical efficacy of deep learning is currently ensured in the so-called overparametrized regime by considering fitting a large amount of data with a larger
amount of parameters, resulting in a high dimensional optimization problem. Furthermore,
for sufficiently overparametrized networks it is known that gradient descent is
guaranteed to achieve zero, or very small training loss in some scenarios \cite{du2019gradient}. However, training deep networks  features a surprising phenomenon which stands in contrast to conventional wisdom in statistics: despite data fitting, increasing the number of parameters of the model beyond the number of training examples often reduces the generalization error, i.e., the prediction error on new unseen data, while from classical theory one would expect that overfitting leads to a blow-up of the generalization error \cite{zhang2016understanding}.
Due to the large number of parameters of neural networks and the potential variety of local and global optima, the result of the training is still difficult to explain and interpret, although a regularization effect is expected, i.e., an implicit bias towards low-complexity networks.
In fact, there are many different networks with different weights, which are essentially equivalent with respect to the loss function. Hence,  except for simple cases, e.g., \cite{arora2019convergence,NIPS2019_8960,bah2019,moroshko2020implicit,neyshabur2015,soudry2018implicit,woodworth2020kernel}, it is yet unknown how, through implicit bias, the information of the training set is eventually encoded in the parameters of the network.

The generalization ability of networks trained by such methods can be challenged by adversarial attacks \cite{Goodfellow2015ExplainingAH,42503}. This clearly show the lack of stability in uniform norm of networks trained by backpropagation, because small perturbations of the input can significantly modify the network output.
The lack of uniform stability by backpropagation can also be experienced in the realizable regime simply by considering as a datum a given (pre-trained) neural network $f$ and by attempting its approximation with another network $\tilde f$ by minimizing its mean-squared misfit over a finite number of samples $\{Y_j = f(X_j): j=1,\dots,N\}$. While the overall approximation of $f$ by $\tilde f$ will be generally good in mean-squared error (with possibly even zero loss on the data), the uniform error saturates to a less satisfactory value, see for instance Figure \ref{fig:gd_comparison} below. This means that the networks are equivalent with respect to the optimization performed by backpropagation, but there exist inputs for which the two networks still differ significantly.
Ensuring uniform stability of networks, which is the mathematical synonym of stable generalization, remains in fact an important open question.

Another aspect that is considered still quite problematic in training by backpropagation is the need of a large size of the set of labeled training data. In fact, it is thought that humans learn predominantly in an unsupervised way, without the need of much labeled data. In the typical human learning, first a context is built in a semi-supervised fashion and then suddenly the learning happens with little more effort \cite{doi:10.1002/sce.37303405110,hinton1988learning}.

\subsection{Scope of the paper: robust identifiability in realizable regime}

It is well-known that generic data are realizable by a network as soon as the network has a number of connections $\overline W= \sum_{\ell=1}^{L-1} m_{\ell+1}\times m_{\ell}$, which  essentially scales with the number $N$ of data, i.e., $\overline W = \mathcal O(N)$  \cite{DBLP:conf/nips/YunSJ19,1189626,doi:10.1137/20M1314884,zhang2016understanding}.
 In this paper we address the three issues mentioned in Section \ref{bpg},
 \begin{itemize}
 \item[(i)]   the explainability and interpretability of the weights of a network,
 \item[(ii)] its uniform stability,
 \item[(iii)] and the amount of  input-output data needed for its identification,
 \end{itemize}
 in such realizable regime. Namely, we approach the problem of the  {\it unique}  and {\it stable} identification  of a given generic neural network from a minimal  number of input-output samples, essentially scaling  with the size of the network, i.e., $N=\mathcal O(\overline W)$. The identification is intended up to equivalences given by natural symmetries such as permutations of neurons and, in case of symmetric activation functions, possible sign changes of weights and shifts \cite{DBLP:journals/corr/abs-2006-11727}.
 There is by now plenty of evidence that gradient descent and its variants could achieve exact realizability on the training data by overparametrized networks, see, e.g., the survey \cite{sun2019optimization} and reference therein. Nevertheless, for the reasons argued in Section \ref{bpg}, we do not use backpropagation for network identification in the regime $N=\mathcal O(\overline W)$ and we make an effort of providing a fully explainable and transparent procedure. Before entering in the details of the procedure, let us review the importance of network identifiability.

Robust identification of neural networks is indeed a task of relevant theoretical and practical interest.  Essentially it says that given a sufficiently generic network no other network, smaller or larger, up to the above mentioned equivalences, can in fact realize the same input-output mapping exactly.
 Robustness also implies that if a larger network performs  an input-output mapping, then it may be reduced to a minimal and potentially significantly smaller network performing  approximately the same function. Among relevant consequences of robust identifiability we mention
\begin{itemize}
\item Explainability: identifiability means to unveil how input-output information relates to the weights of the network, and it is a mathematical characterization, which is in turn a form of explainability. In fact the unique and stable representation of the network by its weights is a precise encoding that actually tells everything about its input-output relationship;
\item Compression of networks: after training of largely overparameterized neural networks, their embedding in smaller (mobile) devices requires ``miniaturization'' of the network, by taking advantage of the expected intrinsic low-complexity due to possible implicit bias. Hence, the identification of the smallest network matching approximately the given large network is of great practical importance, see also \cite{Blcskei2019OptimalAW,gale2019state,shevchenko2020landscape,ye2020good};
\item Reliable use of neural networks for scientific computing: neural networks are known to efficiently approximate solutions of partial differential equations and are recently in focus as new discretization methods for scientific computing, see, e.g.,  \cite{Berner2020,elbraechter2020dnn} and reference therein. It remains
open how to reliably compute best approximating networks and so far backpropagation is the only method used in practice, with no guarantees of optimal solutions. Hence, the robust identification with theoretical guarantees of best approximating networks of minimal complexity is of utmost importance and remains an open issue for the reliable use of neural networks in scientific computing.
\end{itemize}
In this paper we lay the groundwork for the development of such potential applications of robust identification and their further investigation.

\subsection{State of the art}

The unique identifiability of neural networks has been considered in the literature for over three decades \cite{SUSSMANN1992589,Albertini93uniquenessof,Fefferman1994ReconstructingAN,DBLP:journals/corr/abs-1906-06994, DBLP:journals/corr/abs-2006-11727}.  Despite the long standing of the problem, most of the known results are obtained under the fundamental theoretical assumption of being able to access exactly all possible inputs-outputs of the network.
Except for \cite{rolnick2020reverseengineering}, which considers piecewise linear networks, no results are based on a finite number of samples.
In particular, no constructive and stable procedure has been provided for more general networks. The results in
\cite{SUSSMANN1992589,Albertini93uniquenessof} apply to networks with a single hidden layer. The seminal work \cite{Fefferman1994ReconstructingAN} by Fefferman is about identifiability of fully connected deep networks and it is based on the unique encoding of the network architecture within the countable set of poles of the neural network function $f$ if we consider it as a meromorphic function on $\mathbb C^D$.  The recent work by   Vla\v{c}i\'c and B\"{o}lcskei \cite{DBLP:journals/corr/abs-2006-11727}, which partially builds upon  \cite{Fefferman1994ReconstructingAN}, presents a comprehensive result on the identification of deep neural networks without so-called clone nodes and piecewise $C^1$ activation functions $g$, whose derivative has bounded variation. This implies in particular that the activation function is bounded  (hence it is of sigmoidal type) and it can be uniformly approximated by functions $\sigma$, which are linear combinations of shifted (and scaled) $\tanh$ and constants.
The crucial properties of such functions are that they are $i$-periodic with an infinite and discrete set of poles.
The identifiability result is very general and requires a vast set of tools and techniques from analytic continuation from complex analysis in several variables (polydisks techniques), algebraic techniques from Lie groups (characterization of Tori etc.), and graph theory. 

A neural network of the type \eqref{nndef} remains fully determined by a finite number of parameters and although its identification is known since the 1980s to be in general an NP-hard problem \cite{JUDD1988177,BLUM1992117}, it is not at all expected to generically require an infinite amount of training samples as assumed in the above mentioned results. As we already noted above, the simple employment of backpropagation at this point does not offer an adequate solution, because of the lack of uniform stability and the fact that one has no control on the number of samples needed for the identification.\\

Inspired by older work dating back to the 1990s \cite{BUHMANN1999103,CHUI1992131}, the recent papers \cite{fornasier2018identification,fornasier2019robust,Fornasier2012,anandkumar15,Lin20, MondMont18,zhong2017recovery} have explored the connection between differentiation of shallow networks (one hidden layer) and symmetric tensor decompositions.
Namely, given a network of the type
$$
f(x) = \sum_{j=1}^{m_1} g_j(\langle w_j, x\rangle),
$$
corresponding to $L=1$ and $m_L=1$, its higher order derivatives or, more practically, its finite difference approximations $\Delta^k f(x)$ can be (approximately) decomposed into a nonorthogonal symmetric tensor decomposition of the weights
$$
\Delta^k_\varepsilon f(x) \approx \nabla^k f(x) = \sum_{j=1}^{m_1} g_j^{(k)}(\langle w_j,  x\rangle) \underbrace{w_j\otimes \dots \otimes w_j}_{k \mbox{ times}}.
$$
In other words, differentiation exposes the weights, which would  otherwise be ``hidden inside'' the neurons. It has to be noted that for $m_1>1$ one single differentiation ($k=1$) would not suffice in order to identify individual weights, but it would allow to find the active subspace only \cite{fornasier2018identification}.
In \cite{anandkumar15} stable $1$-rank
decompositions of third order symmetric tensors ($k = 3$) \cite{Anandkumar2014GuaranteedNT} have been
used for the weights identification of one hidden layer neural
networks.
In \cite{fornasier2018identification} the authors show that using second order derivatives ($k = 2$)
actually suffices and the corresponding error estimates
reflect positively the lower order and potential of improved
stability. While the computation of finite differences requires active sampling, in the above mentioned
papers also passive sampling has been considered, under the assumption that one disposes
of an estimation of the probability density of the input. Once the weights are recovered, it is possible to provably identify also the functions $g_j$ either by Fourier methods as in \cite{anandkumar15} or by direct estimation as in \cite{fornasier2018identification}.

Unfortunately,  higher order differentiation
of deeper networks (two or more hidden layers) generates nonsymmetric tensors  and the identification of the weights by tensor decompositions may become unstable and, in general, NP-hard.
Hence, it may seem that this technique has significant limitations in that it cannot be applied to deep networks. However,  three of us recently made a surprising discovery, which actually allowed to extend the results
to two hidden layer neural networks \cite{fornasier2019robust} of scalar output, i.e., $L=2$ and $m_L=1$.
The approach is based on the observation that Hessians $\nabla^2f(x)$, after an appropriate algebraic manipulation, can be rewritten in terms of a suitable non-orthogonal decomposition of rank-$1$ matrices. Namely, it holds
$\nabla^2f(x)= W_1 S^{[1]}(x) W_1^\top + V_2(x) S^{[2]}(x) V_2(x)^\top$, where $V_2(x) = W_1 G_1(x) W_2$ for suitable diagonal and invertible matrices $G_1(x)$. Denoting also $V_1(x)=V_1=W_1$, we can rewrite the Hessians as $\nabla^2f(x)= \sum_{\ell=1}^2 V_\ell(x) S^{[\ell]}(x) V_\ell(x)^\top$.

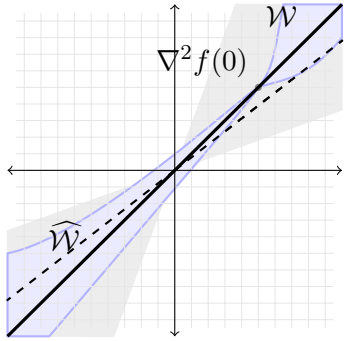
\begin{figure}
\mbox{
 \parbox{10cm}{
	\centering
	\begin{tikzpicture}[scale=1.1]
	\filldraw[gray!14] (0,0) -- (200:2.13) -- (-2,-2) -- (250:2.13) -- (0,0);
	\filldraw[gray!14] (0,0) -- (20:2.13) -- (2,2) -- (70:2.13) -- (0,0);
	\draw[thick, dashed] (0,0) -- (38:2.55);

	\draw[thick, color=blue!30, fill = blue!8]
	(1,1) .. controls (0.8, 0.7) and (0.6,0.5) .. (-1.5,-2) --
	(-2,-2) -- (-2,-1) --
	(-2,-1) .. controls (-1,-0.8) and (0.8, 0.9)  .. (1,1)--
	(1,1) .. controls (1.2, 1.25) and (1.2,1.5) .. (1.3,2)--
	(2,2)--(2,1.6)--
	(2,1.6) .. controls (1.6,1.1) and (1.3, 1.1) .. (1,1);
	\draw[step=0.2cm, gray!20, very thin] (-1.9,-1.9) grid (1.9,1.9);
	\draw[<->] (0,-2) -- (0,2);
	\draw[<->] (-2,0) -- (2, 0);
	\filldraw[gray] (1,1) circle (1pt) node[color=black,above left]{$\nabla^2 f(0)$} ;
	\draw[very thick] (-2,-2) -- node[above left,pos=0.9]{$\mathcal W$}(2,2);
	\draw[thick, dashed] (0,0) -- node[above,left]{$\widehat{\mathcal W}$} (218:2.55);
	\draw[thick, dashed] (0,0) -- (38:2.55);
	\end{tikzpicture}
 }
 \hfill
\parbox{6cm}{\small
Description: Illustration of the relationship between $\mathcal W$ (black line) for $x^\ast=0$ and $\operatorname{span}\left\lbrace \nabla^2 f(x) \middle| x \in \mathbb R^{D} \right\rbrace$ (light blue region) given by two non-linear cones that fan out from $\nabla^2 f(0)$. There is no reason to believe that the these cones are symmetric around $\mathcal W$. The gray cones show the maximal deviation of $\widehat{\mathcal W}$ from $\mathcal  W$.
}}
\caption{Geometrical visualization of Hessians' distribution.}\label{fig:geometryCW}
\end{figure}

In view of the Lipschitz continuity of $x\to V_\ell(x)$, if $x$ is sampled from a tightly concentrated distribution $\mu_X$, say around a point $x^\ast$, then the Hessians   cluster around a very specific subspace
$\mathcal W = \Span{w_i\otimes w_i}$ generated by symmetric rank one matrices. The components $w_i$ of these tensor products are only partially generated by the original weights $W_1$ and  in part are suitable compositions/mixtures of the weights of the first and second layers
$V_2(x^\ast) = W_1 G_1(x^\ast) W_2$ that we call the {\it entangled weights}.
Under the assumption that the set $\{\nabla^2 f(x_i): x_i \sim \mu_X\}$ is sufficiently rich,  the subspace $\mathcal W$ can be robustly approximated $\widehat{\mathcal W} \approx \mathcal W$ by principal component analysis of (approximate) Hessians, see Figure \ref{fig:geometryCW}. This approach does work for generic and fully nonlinear networks, for instance, with sigmoidal-type of activation functions.
We notice, however, that the method fails for networks for which $\{\nabla^2 f(x_i): x_i \sim \mu_X\}$ does not contain enough information, for instance, for piecewise linear networks (e.g., with ReLU activation function),  where $\nabla^2 f(x)=0$ almost everywhere. An extension of our approach to encompass  (leaky) ReLU networks has been recently presented in \cite{Lin20}, which, remarkably, explores the use of passive sampling also for two hidden layers networks.
An alternative {\it ad hoc} approach for ReLU networks is also \cite{rolnick2020reverseengineering}.

Once the subspace $\widehat{\mathcal W} \approx \mathcal W$ is computed, the problem of identification of the network parameters can be transformed to discovering symmetric rank-one matrices $w_i\otimes w_i$
near the subspace $\widehat{\mathcal W} \approx \mathcal W$. For this task, we introduced and analyzed in \cite{fornasier2019robust,fornasier2018identification} a robust nonconvex optimization program given by
\begin{equation}\label{nonconvprog1}
\arg\max \|M\| \mbox{ subject to } \|M\|_F \leq 1, \quad M \in \widehat{\mathcal W},
\end{equation}
where $\N{\cdot}$ is the operator norm and $\N{\cdot}_F$
is the Frobenius norm. We proved, under suitable incoherence assumptions, i.e., $|\langle w_i,w_j\rangle|$ is sufficiently small for $i\neq j$, that local maximizers of \eqref{nonconvprog1}
are in fact approximately the matrices $w_i\otimes w_i$.
Within the subspace $\mathcal W$, the hierarchy or layer attribution of weights to
the first layer $V_1 = W_1$ and entangled weights $V_2$ relative to the second layer is in fact inaccessible.
For sigmoidal type of activation functions $g$, we  devised a heuristic method based on the asymptotic behavior of the function
\begin{equation}\label{hierarchy}
t \rightarrow \N{\nabla f(tw_i)}_2
\end{equation}
to attribute the computed components to first
layer or entangled weights of the second layer: if the function decays at infinity, then $w_i$ is an entangled weight of the second layer, otherwise it is a weight of the first layer. Once the entangled weights are recovered and attributed correctly to their reference layer, we showed numerically that the remaining and much fewer parameters of the network can be identified by least squares and use of standard gradient descent. The resulting networks coincide - exactly - with the original one in full.

\subsection{Contributions of this paper}

The scope of this  paper is to  extend the approach in \cite{{fornasier2019robust}} to deeper, vector-valued networks, i.e., $L\geq 1$ and $m_L \geq 1$.
We consider networks $f:\mathbb R^D \to \mathbb R^{m_L}$ as in \eqref{nndef} and we denote $f=(f_1,\dots,f_{m_L})$ the vector components of the network. We shall show in Proposition \ref{prop:nn_derivatives} that Hessians $\nabla^2 f_p(x)$, or finite difference approximations $\Delta_\varepsilon^2 f_p(x)$ thereof, have (approximately) the form
$$\Delta_\varepsilon^2 f_p(x) \approx \nabla^2 f_p(x) = \sum_{\ell=1}^L V_\ell(x) S_p^{[\ell]}(x) V_\ell(x)^\top, \quad p \in [m_L],$$
for entangled weight matrices $V_{\ell}(x)= \left(\prod_{k=1}^{\ell-1} W_k G_k(x)\right)W_\ell$.
In view of the Lipschitz continuity of $x\to V_\ell(x)$, by sampling Hessians $\Delta_\varepsilon^2 f_p(x)$ from a distribution $x\sim \mu_X$ tightly concentrating, e.g., at $x^*$, we show that they
cluster around a subspace $\mathcal W = \operatorname{span}\{w_i\otimes w_i\}$, irrespectively of the output $p \in [m_L]$.  Furthermore, the spanning rank-1 basis elements $w_i\otimes w_i=v^{[\ell]}_i(x^*)
\otimes v^{[\ell]}_i(x^*)$ are precisely made of entangled weight vectors $v^{[\ell]}_i(x^*)$, columns of $V_{\ell}(x^*)$. Next we show by Theorem \ref{thm:approx_bound} that the subspace $\mathcal W$ can  be
stably approximated $\widehat{\mathcal W} \approx \mathcal W$ by  PCA (Algorithm \ref{alg:subspace_approx}) of the point cloud $\{\Delta_\varepsilon^2 f_p(x_i): x_i \sim \mu_X\}$.

  Then, differently from the approach used in \cite{fornasier2018identification,fornasier2019robust} where the matrix optimization program \eqref{nonconvprog1} was employed, the entangled weights are here discovered within the subspace $\widehat{\mathcal W} \approx \mathcal W$ by the robust nonconvex program
\begin{equation}\label{nonconvprog2}
\max_{u \in  \mathbb S^{D-1}}  \|P_{\widehat{\mathcal W}}(u\otimes u)\|_F^2.
\end{equation}
This program has been considered implicitly in \cite[Lemma 16 and Lemma 17]{fornasier2019robust} and it has been proposed independently in \cite{fiedlerlearning, kileel2019subspace} in the context of neural networks and tensor decompositions, respectively. It can be solved efficiently by a projected gradient ascent iteration over the sphere $\mathbb S^{D-1}$, called {\it  subspace power method} (Algorithm \ref{alg:recover_weights}), which was extensively analyzed for tensor decompositions \cite{kileel2019subspace}. The advantage of \eqref{nonconvprog2} over \eqref{nonconvprog1} is that the theoretical analysis of its robustness is mathematically easier, see Theorem \ref{thm:summary_theorem} below; moreover, as the optimization is over vectors instead of matrices, \eqref{nonconvprog2} comes with a significantly improved algorithmic complexity. In fact, such a method is highly scalable with respect to the size of the network and the dimension $D$.
Furthermore, we extend in Proposition \ref{prop:existence_of_equivalent_network} the reparametrization result of \cite{fornasier2019robust} to deep networks, showing that entangled weights can be used for loss-free reparametrization of the original network, leaving much fewer residual parameters related to scaling of the weights and shifts of activation functions undetermined.
In order to highlight the relevance of the results of robust recovery of entangled weights and to provide for them a proper context, we give empirical demonstration of complete identification of multilayer networks.
In particular, as in \cite{fornasier2019robust}, we found reliable heuristics, which allow to classify the entangled weights in terms of their attribution to different layers (Algorithms \ref{alg:clustering}-\ref{alg:detect_last_layer}). So far this approach is limited to the case of $L\leq 3$. The full identification of the reparametrized network can be again performed by a least squares fit over the remaining - scaling and shift - parameters, see \eqref{eq:new_GD_recover_network_functional}. We do not dispose yet of a proof of such convergence though, but it is consistently observed in the numerical experiments and it is subject of current investigation.

\subsection{Organization of the paper}
The paper is organized as follows. In Section \ref{sec:ffnets_and_entangled_weights} we introduce the
networks considered in this work and we recall the definition of entangled weights, alongside the mentioned
reparametrization result. Based on these preliminary results, we describe in Section \ref{sec:algorithm}
our algorithmic pipeline for reconstructing a deep network from a finite number of input-output samples. Section \ref{sec:experiments} shows
extensive numerical experiments for each part of the pipeline, whereas Sections
\ref{sec:subspace_approx} and \ref{sec:rank_one_recov} provide theoretical analyses
of the subspace approximation of $\CW$ and a robust analysis of the optimal
program \eqref{nonconvprog2} under suitable incoherence conditions. Thus, the latter two
sections give theoretical justification for our approach to entangled weight recovery.
We close the paper in Section \ref{sec:conclusion} with a conclusion, future directions,
and open problems.

\subsection{Notation}\label{sec:notation}
\newcommand{\lind}[1]{[#1]}
\newcommand{\dprime}{{\prime\prime}}

 Given any integer $m \in \mathbb N$, we use the symbol $[m]:=\{1,2,\dots,m \}$ for indicating the index set of the first $m$ integers. We denote by $B_1^d$ the Euclidean unit ball in $\mathbb R^d$, by $\mathbb S^{d-1}$ the Euclidean sphere, by $\textrm{Unif}(\bbS^{d-1})$ its uniform probability measure, and by $\textrm{Sym}(\bbR^{d\times d})$ the space
 of symmetric matrices in $\bbR^{d\times d}$. We denote by $\ell_q^d$ the $d$-dimensional Euclidean space endowed with the norm $\|x\|_{\ell_q^d} =\left (  \sum_{j=1}^d |x_j|^q \right )^{1/q}$. For $q=2$ we often write indifferently  $\|x\| = \|x\|_{2}= \|x\|_{\ell_2^d}$. For a matrix $M$ we denote $\sigma_k(M)$ its $k^{th}$ singular value.
 The spectral norm of a matrix is denoted $\|\cdot\|$.
 We may also denote with $\|\cdot\|_{p \to q}$ the operator norms from $\ell_p^d$ to $\ell_q^m$ spaces.
 Given a closed set $C$ we denote by $P_C(x)$ the possibly set-valued orthogonal projection $P_C(x) \in \argmin_{z \in C}\N{x-z}_2$ ($C$ may be a subspace of $\bbR^d$ or $\textrm{Sym}(\bbR^{d\times d})$, or spheres thereof, or a subspace intersected with a Euclidean ball or sphere). Whenever we use $P_C$, the projection will
 be uniquely defined.
For vectors $x_1,\dots, x_k \in  \mathbb R^d$ we denote the tensor product $x_1 \otimes \dots \otimes x_k$ as the tensor of entries $({x_1}_{i_1} \dots {x_k}_{i_k})_{i_1,\dots,i_k}$. For the case of $k=2$ the tensor product $x \otimes y$ of two vectors  $x,y \in  \mathbb R^d$ equals the matrix $x y^T =  (x_i y_j)_{ij}$. For any matrix $M \in \R^{m \times n}$
\begin{align*}
\opvec(M) := (m_{11}, m_{21}, \dots, m_{m1}, m_{12}, m_{22}, \dots, m_{mn})^T \in \R^{mn}
\end{align*}
is its vectorization, which is the vector created by  unfolding the columns of $M$. Moreover, we denote by $M^\dagger$ the Moore-Penrose pseudoinverse of the matrix $M$. The kernel and range of a matrix $M$ are denoted by $\ker(M)$ and $\range(M)$ respectively.
Furthermore, we denote by $\Id_d$ the $d\times d$ identity matrix.
For a suitably regular function $g$ we denote by $g^{(k)}$ its $k$-th derivative.
The gradient of a differentiable function $f: \bbR^n \rightarrow \bbR$ is denoted by $\nabla f$
and its Hessian by $\nabla^2 f$. Finally, we define the sub-Gaussian norm of a random variable $Z \in \bbR$ by
$\textN{Z}_{\psi_2} = \inf\{s > 0 : \bbE\exp(\SN{X/s}^2) \leq 2\}$ and of a random vector $Z \in \bbR^d$ by
$\textN{Z}_{\psi_2} = \sup_{v \in \bbS^{d-1}}\textN{v^\top Z}_{\psi_2}$.

\section{Feedforward neural networks and entangled weight vectors}
\label{sec:ffnets_and_entangled_weights}
This section defines feedforward neural networks as considered in this work and
recalls the important concept of entangled weight vectors, which has been previously
introduced and used in \cite{fornasier2019robust}. The main purpose und usage of entangled weights is a loss-free
reparametrization of the original network function by a new set of weights, which can, contrary to standard weights, be exposed by differentiating the network function, as a form of linearization. The exposure of the entangled weights allows in turn their robust identification. In the case of
shallow networks the definition coincides with standard network weights, but to extend
network identification based on derivative information to deeper networks,
the concept of entangled weights seems indispensable.

We begin with the  definition of feed forward networks considered in this work.

\begin{dfn}[Feedforward neural network]
\label{dfn:feedforward_networks}
Let $L,m_0,\ldots,m_L \in \bbN$ with $D=m_0$. For $\ell \in [L]$,
consider weight matrices
\begin{equation*}
W_\ell = \left(w^{\lind{\ell}}_1|  \ldots |w^{\lind{\ell}}_{m_\ell}\right)\in \bbR^{m_{\ell-1} \times m_\ell},
\end{equation*}
shifts $\tau_\ell \in \R^{m_\ell}$, and let $g : \bbR \rightarrow \bbR$ be an activation function. A feedforward
neural network with $m_L$ outputs is a function $f : \bbR^D \rightarrow \bbR^{m_L}$ computed via the recursive rule
$y^{\lind{0}}(x) = x$,
\begin{align*}
y^{\lind{\ell}}(x) &= g(W_{\ell}^\top y^{\lind{\ell-1}} + \tau_{\ell} ), \quad \ell \in [L]
\end{align*}
and $f(x) = y^{\lind{L}}(x),$
where $g$ is meant to be applied componentwise to non-scalar inputs. The components of $f$ are denoted by $f_p$ for $p \in [m_L]$ and we often
write $g_{\ell}(\cdot) = g(\cdot + \tau_{\ell})$. It will often be  useful to refer to the number of
neurons of the network as $m=m_1+\dots+m_L$.
\end{dfn}

Let us now introduce entangled weight vectors. For each $x \in \bbR^D$ we first define
diagonal matrices $G_\ell(x) := \opdiag\left(g_\ell^\prime(W_\ell^\top y^{[\ell-1]}(x)) \right) \in \bbR^{m_\ell \times m_\ell}$,
which depend on the shifted activation function $g_\ell(\cdot)=g(\cdot + \tau_\ell)$ and the original weight matrices $W_\ell$, for $\ell \in [L]$. Then, the
$(i,\ell)$-th {\it entangled weight} vector of a network $f : \bbR^D\rightarrow \bbR^{m_L}$
at a location $x \in \bbR^D$ is defined as
\begin{align}
\label{eqn:entangled_weights}
v^{\lind{\ell}}_i(x) &:= \left(\prod_{k=1}^{\ell-1} W_k G_k(x)\right)w_i^{\lind{\ell}}.
\end{align}
The index $\ell$ indicates the layer attribution of the entangled weight and the index $i$ represents some ordering within the layer.
We can also write the set of entangled weights at layer $\ell$ in matrix form by defining
\begin{align}
\label{eqn:entangled_weights_matrix}
V_{\ell}(x) &:= \left(\prod_{k=1}^{\ell-1} W_k G_k(x)\right)W_\ell.
\end{align}
Note that the definition of entangled weights at $\ell = 1$ simply coincides with $W_1$,
implying that weights and entangled weights coincide for shallow networks. For
deeper networks however, entangled weights for $\ell > 1$
are linear combinations of first layer  weights $W_1$
with mixture coefficients given by weight matrices $W_{2},\ldots,W_{\ell}$ and diagonal
matrices $G_{1},\ldots,G_{\ell-1}$. Thus, they generally differ substantially from $W_1,\ldots,W_{L}$,
in particular because entangled weights at layer $\ell$ are elements of $\bbR^D$ rather than $\bbR^{m_{\ell-1}}$.

On first sight the definition \eqref{eqn:entangled_weights} may seem artificial
and lack a clear connection to either derivatives of the network function $f$ or
loss-less reparametrizations of $f$ mentioned in the beginning
of this section. The next two results clarify these relations and thus motivate the definition of entangled weights. First, we show that Hessians of $f$ at $x$
can be decomposed in terms of entangled weight vectors. This is a key result, because it implies that the entangled weights are exposed
by differentiation, similarly as first layer weights in shallow nets \cite{anandkumar15,fornasier2018identification,fornasier2019robust,zhong2017recovery}.

\begin{prop} \label{prop:nn_derivatives}
Let $f$ be a feedforward network as in Definition \ref{dfn:feedforward_networks}. The Hessian of $f_p$ reads
 \begin{align*}
  \nabla^2 f_p(x) & = \sum_{i=1}^{m_1} S^{\lind{1}}_{p,i}(x) \left( w^{\lind{1}}_i \otimes w^{\lind{1}}_i \right)
    + \sum_{\ell=2}^L \sum_{i=1}^{m_\ell}  S^{\lind{\ell}}_{p,i}(x) \left( v^{\lind{\ell}}_i(x) \otimes v^{\lind{\ell}}_i(x) \right) = \sum_{\ell=1}^L V_\ell(x) S_p^{\lind{\ell}}(x) V_\ell(x)^\top,
 \end{align*}
where the scalar $S^{\lind{\ell}}_{p,i}(x)$ is the $i$-th entry of the diagonal matrix
 \begin{equation*}
  S^{\lind{\ell}}_p(x) = \diag(g_\ell^\dprime(W_\ell^\top y^{\lind{\ell-1}}(x)))
  \diag\left( \left( \prod_{k=\ell+1}^L W_k G_k(x) \right)e_p\right),
 \end{equation*}
  where $e_p$ is a vector of the canonical basis.
\end{prop}
\begin{proof}
  Let $Jh(x)$ denote the Jacobian matrix of any vector-valued function $h$. By the chain
  rule we get
   $J y^{\lind{\ell}}(x) = G_\ell(x)W_\ell^\top G_{\ell-1}(x) \cdots W_1^\top$ for all $\ell \in [L]$,
  which implies
  \begin{equation} \label{eq.nn_deriv.grad}
   \nabla f_p(x) = W_1 G_1(x) \cdots W_L G_L(x) e_p,\quad p \in [m_L],
  \end{equation}
  where $e_p$ is a vector of the canonical basis.
  Now let us evaluate some preliminary expressions for later use in the Hessian of $f_p$.
 We recall that, whenever $g_{\ell}$ or its derivatives are applied to vectors, we assume that they act componentwise, otherwise they act as univariate scalar functions.
  Using the chain rule we obtain
\begin{equation}
  \label{eq.nn_deriv.aux1}
  \begin{aligned}
   J\left( g^\prime_\ell(W_\ell^\top y^{\lind{\ell-1}}(x)) \right)
   & = \opdiag( g_\ell^\dprime( W_\ell^\top y^{\lind{\ell-1}}(x)) W_\ell^\top  J y^{\lind{\ell-1}}(x)  \\
   & = \opdiag( g_\ell^\dprime( W_\ell^\top y^{\lind{\ell-1}}(x)) W_\ell^\top G_{\ell-1}(x) W_{\ell-1}(x)^\top \ldots G_1(x)W_1^\top.
 \end{aligned}
\end{equation}
  Furthermore, for any $\bbR^{m_\ell}$-valued differentiable function $r(x)$ we have, using the product and chain rule, that
  \begin{equation}
    \label{eq.nn_deriv.aux2}
  \begin{aligned}
   &J(G_\ell(x) r(x) )
    = J\begin{pmatrix}
		   g_\ell^\prime(\langle w^{\lind{\ell}}_i, y^{\lind{\ell-1}}(x)\rangle) r_i(x)
                   \end{pmatrix}_{i \in [m_\ell]} \\
   &\quad = \begin{pmatrix}
        g_\ell^\prime(\langle w^{\lind{\ell}}_i, y^{\lind{\ell}}(x) \rangle) J r_i(x)
       \end{pmatrix}_{i \in [m_\ell]}
       +
       \begin{pmatrix}
        J \left( g_\ell^\prime(\langle w^{\lind{\ell}}_i, y^{\lind{\ell}}(x) \rangle) \right)
        r_i(x)
       \end{pmatrix}_{i \in [m_\ell]}  \\
   &\quad = G_\ell(x) J r(x) +
   \opdiag(r(x)) J\left( g_\ell^\prime(W_\ell^\top y^{\lind{\ell-1}}(x)) \right)
 \end{aligned}
\end{equation}
  By using \eqref{eq.nn_deriv.grad} - \eqref{eq.nn_deriv.aux2},
  the linearity of $J$,  and the definition of matrices $V_1=W_1$ and $S^{[1]}(x)$,   we obtain the explicit form of the Hessian of $f_p$
  \begin{align*}
   &\nabla^2 f_p(x)  = J ( \nabla f_p(x) )  = J \left( \left( \prod_{\ell=1}^L W_\ell G_\ell(x) \right)e_p \right)  = W_1 J \left( G_1(x) \left( \prod_{\ell=2}^L W_\ell G_\ell(x) \right) e_p \right) \\
   &\quad = W_1 \opdiag\left( \left(\prod_{\ell=2}^L W_\ell G_\ell(x) \right) e_p \right) J \left( g_1^\prime(W_1^\top x) \right)
    + W_1 G_1(x) J \left( \left( \prod_{\ell=2}^L W_\ell G_\ell(x) \right) e_p \right) \\
   &\quad = W_1 \opdiag\left( \left(\prod_{\ell=2}^L W_\ell G_\ell(x) \right) e_p \right) \opdiag \left( g_1^\dprime(W_1^\top x) \right) W_1^\top
    + W_1 G_1(x) J \left( \left( \prod_{\ell=2}^L W_\ell G_\ell(x) \right) e_p \right) \\
   &\quad = V_1 S_p^{\lind{1}}(x) V_1^\top + W_1 G_1(x) W_2 J \left( G_2(x) \left( \prod_{\ell=3}^L W_\ell G_\ell(x) \right) e_p \right)
  \end{align*}
  Repeating these steps leads to
  \begin{align*}
   \nabla^2 f_p(x) & = \sum_{\ell=1}^{L-1} V_\ell(x) S^{\lind{\ell}}_p(x) V_\ell(x)^\top + \left(\prod_{\ell=1}^{L-1} W_\ell G_\ell(x) \right)
    J \left( W_L G_L(x) e_p \right) \\
   & =  \sum_{\ell=1}^{L-1} V_\ell(x) S_p^{\lind{\ell}}(x) V_\ell(x)^\top
    + V_L(x) \opdiag(g_L^\dprime(W_L^\top y^{\lind{L-1}}(x)) \opdiag(e_p) W_L^\top J y^{\lind{L-1}}(x) \\
   & =  \sum_{\ell=1}^{L} V_\ell(x)  S_p^{\lind{\ell}}(x) V_\ell(x)^\top.
  \end{align*}
\end{proof}

It still remains to clarify the role of entangled weights in the  loss-less reparametrization
of the original network $f$. Let us assume we have access to matrices of the type
\begin{align}
\label{eq:entangled_weights_lol}
\widetilde V_{\ell} = \prod\limits_{k=1}^{\ell-1}(W_k D_k)W_{\ell}\pi_{\ell}S_{\ell} \in \bbR^{D \times m_{\ell}},\qquad \ell \in [L],
\end{align}
where $D_1,\ldots,D_{L-1}, S_1,\ldots,S_L$ are arbitrary invertible diagonal matrices
and $\pi_{1},\ldots,\pi_{L}$ are permutation matrices. In particular, we have in mind
$D_i = G_i(x)$ for some $x \in \bbR^D$, where the diagonal matrices $G_1(x),\ldots,G_{L-1}(x)$ are invertible,
so that $\widetilde V_{\ell}$ equals $V_{\ell}(x)$ after rescaling of
the columns by some scaling matrix $S_\ell$ and permuting entangled
weights (columns of $V_{\ell}(x)$) by $\pi_{\ell}$. Having access to such matrices,
we can derive the following reparametrization result.
\begin{prop}
\label{prop:existence_of_equivalent_network}
Let $f$ be a feedforward network as in Definition \ref{dfn:feedforward_networks}
and assume we have access to matrices $\widetilde V_{1},\ldots, \widetilde V_{L}$
as in \eqref{eq:entangled_weights_lol} for invertible diagonal matrices $D_\ell, S_\ell$ and permutation matrices $\pi_{\ell}$.
Furthermore, assume $\rank(\widetilde V_{\ell}) = m_{\ell}$ for all $\ell \in [L]$.
Then the feedforward network $\tilde f$ defined by
weight matrices $\widetilde W_1^\top = S_1^{-1}\widetilde V_1^\top$ and
\begin{align*}
\widetilde W_{\ell+1}^\top = S_{\ell+1}^{-1} \widetilde V_{\ell+1}^\top (\widetilde V_{\ell}^\top)^{\dagger} S_{\ell} \widetilde D_{\ell}^{-1},\quad \ell \in [L-1],
\end{align*}
with $\widetilde D_{\ell} = \pi_{\ell}^\top D_{\ell}\pi_{\ell}$, shifts $\widetilde \tau_{\ell} = \pi_{\ell}^\top \tau_{\ell}$, and
activation functions $g$ satisfies $\widetilde f \equiv \pi_{L}^\top \circ f$.
\end{prop}

Proposition \ref{prop:existence_of_equivalent_network} clarifies that
$f$ can be expressed, without loss, by a new set of weight matrices $\widetilde W_\ell$,
which are completely defined by matrices $\widetilde V_{\ell}$ up to diagonal matrices
$S_1,\ldots, S_L$ and $\widetilde D_1,\ldots, \widetilde D_{L-1}$, and of course also by the
shift vectors $\tau_1,\ldots,\tau_L$. The degress of freedom of the latter parameters $S_\ell, D_\ell, \tau_\ell$ are $\CO(\sum_{\ell=0}^{L}m_{\ell})$, which is in fact much fewer
than the number of total parameters $\CO(\sum_{\ell=0}^{L-1}m_{\ell+1}m_{\ell})$ in the
original network. Hence, if we would dispose of matrices of the type $\widetilde V_{\ell}$, then we would obtain also a significant (loss-less) de-parametrization of the network.

\begin{rem}
It is important to notice that Proposition \ref{prop:existence_of_equivalent_network} comes with the condition
$\rank(\widetilde V_{\ell}) = m_{\ell}$ (this can be relaxed to the more technical condition $S_{\ell}\pi_{\ell}^\top D_\ell y^{\lind{\ell}}(x) \in \ker(V_{\ell})^\perp$ for all $x\in \bbR^D$),
which restricts the applicability of the result to pyramidally shaped networks with
$D = m_0 \geq m_1 \geq \ldots \geq m_{L}$. Considering the fact that we replace
weight matrices $W_{\ell} \in \bbR^{m_{\ell-1} \times w_{\ell}}$ by scaled versions of
$\widetilde V_{\ell+1}^\top \widetilde V_{\ell}^{\dagger}$, it is clear that the rank
of the modified weight matrix is limited by $\min\{m_{\ell}, m_{\ell+1}, D\}$ just by
evaluating the matrix dimensions. Therefore, we can not hope for loss-free reparametrization
in the case of general non-pyramidic networks. A fairly straight-forward extension is possible
however when the entire network $f$ is not fully connected, but consists of several pyramidically
shaped subnetworks. In that case, by considering each subnetwork separately
and then forming a modified weight matrix with blocks corresponding to each subnetwork
loss-free reparametrization is possible again.
\end{rem}

\begin{proof}[Proof of Proposition \ref{prop:existence_of_equivalent_network}]
For readability the dependence on the input  $x$ may be not written in the following.
Let $\widetilde y^{\lind{\ell}}$ be the output function at layer $\ell$ of the network
with weights $\widetilde W_{1},\ldots,\widetilde W_{L}$ and shifts $\widetilde \tau_{1},\ldots,\widetilde \tau_{L}$.
We shall prove by induction that $\widetilde y^{\lind{\ell}} \equiv \pi_{\ell}^\top y^{\lind{\ell}}$  for all $\ell \in [L]$.

\noindent
\textbf{Induction start.}
Using the definition of $\widetilde W_1$ and $S_1^{-1}\widetilde V_1^\top = \pi_1^\top W_1^\top$, we obtain for the first layer
\begin{align*}
\widetilde y^{\lind{1}} &= g(\widetilde W_1^\top y^{\lind{0}} + \widetilde\tau_{1}) =
g(S_1^{-1} \widetilde V_1^\top  y^{\lind{0}} + \pi_1^\top \tau_1) =
g(\pi_1^\top W_1^\top  y^{\lind{0}} + \pi_1^\top \tau_1)  = \pi_1^\top y^{\lind{1}} .
\end{align*}
\textbf{Induction step.}
Now, let us perform the induction step and assume $\widetilde y^{\lind{\ell}} = \pi_{\ell}^\top y^{\lind{\ell}}$. We first note that \eqref{eq:entangled_weights_lol} implies
$\widetilde V_{\ell}S_{\ell}^{-1}\pi_{\ell}^\top = \Pi_{k=1}^{\ell-1}(W_kD_k)W_\ell$ and therefore we obtain the relation
\begin{equation}
\label{eq:V_eqn}
\begin{aligned}
\widetilde V_{\ell+1}&= \Pi_{k=1}^{\ell}(W_kD_k)W_{\ell+1}\pi_{\ell+1}S_{\ell+1}\\
&= \Pi_{k=1}^{\ell-1}(W_kD_k)W_{\ell}D_{\ell}W_{\ell+1}\pi_{\ell+1}S_{\ell+1} \\
&= \widetilde V_{\ell}S_{\ell}^{-1}\pi_{\ell}^\top D_{\ell}W_{\ell+1}\pi_{\ell+1}S_{\ell+1}.
\end{aligned}
\end{equation}
Furthermore, since $\widetilde V_{\ell}$ has independent columns we have $\widetilde V_{\ell}^{\dagger} \widetilde V_{\ell} = \Id_{m_{\ell}}$ and from \eqref{eq:V_eqn} we derive the identity
$\widetilde V_{\ell+1}^\top (\widetilde V_{\ell}^{\top})^\dagger = S_{\ell+1}\pi_{\ell+1}^\top W_{\ell+1}^\top D_{\ell}\pi_{\ell} S_{\ell}^{-1}$. Then,
using the definition $\widetilde W_{\ell+1}$ and the identity
for $\widetilde V_{\ell+1}^\top (\widetilde V_{\ell}^{\top})^{\dagger}$, the output of layer $\ell+1$ can be written as
\begin{align*}
\widetilde y^{\lind{\ell+1}} &= g(\widetilde W_{\ell+1}^\top\widetilde y^{\lind{\ell}} + \widetilde \tau_{\ell+1}) = g(S_{\ell+1}^{-1}\widetilde V_{\ell+1}^\top (\widetilde V_{\ell}^{\top})^\dagger S_{\ell}\widetilde D_{\ell}^{-1} \widetilde y^{\lind{\ell}} + \widetilde \tau_{\ell+1})\\
&=g(S_{\ell+1}^{-1}S_{\ell+1}\pi_{\ell+1}^\top W_{\ell+1}^\top D_{\ell}\pi_{\ell} S_{\ell}^{-1}S_{\ell}\widetilde D_{\ell}^{-1} \widetilde y^{\lind{\ell}} + \widetilde \tau_{\ell+1})
=g(\pi_{\ell+1}^\top W_{\ell+1}^\top D_{\ell}\pi_{\ell} \widetilde D_{\ell}^{-1} \widetilde y^{\lind{\ell}} + \widetilde \tau_{\ell+1}).
\end{align*}
Furthermore, by the induction hypothesis we have
$\widetilde y^{\lind{\ell}} = \pi_{\ell}^\top y^{\lind{\ell}}$. Using additionally
$\pi_{\ell}\widetilde D_{\ell}^{-1}\pi_{\ell}^\top = D_{\ell}^{-1}$,
it follows that
\begin{align*}
\widetilde y^{\lind{\ell+1}} &= g(\pi_{\ell+1}^\top W_{\ell+1}^\top D_{\ell}\pi_{\ell} \widetilde D_{\ell}^{-1} \pi_{\ell}^\top y^{\lind{\ell}} + \widetilde \tau_{\ell+1}) = g(\pi_{\ell+1}^\top W_{\ell+1}^\top y^{\lind{\ell}} + \pi_{\ell+1}\tau_{\ell+1})
= \pi_{\ell+1}^\top y^{\lind{\ell+1}}.
\end{align*}
\end{proof}

\section{Reconstruction pipeline based on entangled weight vectors}
\label{sec:algorithm}
In this section we present the numerical approach for reconstructing a network
function $f : \bbR^D \rightarrow \bbR^{m_L}$ of pyramidal shape, i.e., which satisfies
$D = m_0 \geq \ldots \geq m_L$. Based on Proposition \ref{prop:existence_of_equivalent_network}, the
function $f$ can be reparametrized using few parameters when
having access to scaled and possibly column-permuted entangled weight matrices $V_{1}(x^{\ast}),\ldots,V_{L}(x^\ast)$
for some $x^\ast \in \bbR^D$. This motivates the splitting of the problem into two phases,
where the first phase aims at identifying entangled weight matrices  $V_{1}(x^{\ast}),\ldots,V_{L}(x^\ast)$
(or rather scaled and permuted versions thereof as in Eqn. \eqref{eq:entangled_weights_lol}),
and the second phase recovers the remaining scale information and shift parameters
using vanilla gradient descent for a least squares fit. Although gradient descent
is still employed as in standard backpropagation, we stress that the least squares problem
involves much fewer free parameters and, more importantly, will be fed with a priori
computed entangled weight information. As shown in the experiments, the subsequent gradient
descent step converges fairly quickly to a network which is uniformly close to $f$,
contrary to using backpropagation for learning the network entirely from scratch.

Recovering $V_{1}(x^{\ast}),\ldots,V_{L}(x^\ast)$ in the first phase of the pipeline
is done via a three-step algorithm. The first step, called \textbf{Building the context},
approximately identifies the subspace
\begin{equation} \label{eq:approx_subspace_W}
 \CW := \Span{v^{[\ell]}_{i}(x^\ast) \otimes v^{[\ell]}_i(x^\ast) \mid i=1,\ldots,m_\ell, \: \ell=1,\ldots,L, }
\end{equation}
spanned by outer products of entangled weight vectors by leveraging that Hessians
expose these vectors, see Proposition \ref{prop:nn_derivatives}. The second step,
called \textbf{Weight recovery}, recovers the spanning elements $\{v^{[\ell]}_{i}(x^\ast)  \mid i \in [m_\ell], \: \ell \in [L]\}$
within $\CW$ using a suitably defined nonconvex program. Finally, since layer assignment
information for each recovered weight is lost when working with the subspace $\CW$,
the last step, \textbf{Weight assignment}, assigns each recovered weight to the corresponding layer.

Let us provide additional details for each of the four algorithmic  steps in the following.
\subsection{Building the context}
\label{subsec:building_the_context}
The starting point for approximating the subspace $\CW$ is the structure of the Hessians as in
Proposition \ref{prop:nn_derivatives}
\begin{equation}
\label{eq:hessian_formulation_alg}
 \nabla^2 f_p(x) = \sum_{\ell=1}^L V_\ell(x) S_p^{\lind{\ell}}(x) V_\ell(x)^\top,
\end{equation}
which makes explicit that entangled weights $V_1(x),\ldots,V_L(x)$ are naturally exposed by differentiation. 
In the case of shallow networks, i.e., for
$L = 1$, we have only $V_{1}(x) = W_1$ and
we can exactly recover $\CW$ by taking the span of ca. $m_1$ linearly
independent Hessian matrices of $f$: more precisely, when sampling
$\numsamples \approx m_1/m_L=1$ locations $X_1,\ldots,X_{\numsamples}$ at random, e.g., uniformly from the unit sphere, for generic networks
the resulting Hessians $\nabla^2 f_1(X_1),\ldots, \nabla^2 f_{m_L}(X_{\numsamples})$ are linearly independent with high probability. 
Furthermore, if we only have access to point queries of the network,
Hessians can be instead approximated by finite difference approximations $\Delta_{\epsilon}^2 f_p(x) \approx  \nabla^2 f_p(x)$,
thus incurring a small error while approximating $\CW$, depending on the finite difference stepsize $\epsilon$.
We refer to \cite{fornasier2018identification} for further details on the shallow network case.

\begin{algorithm}[t]

	\KwIn{Neural network $f:\R^D \to \R^{m_L}$ with layers of width $m_1, \dots, m_L$, number of locations $N_H$, probability distribution $\mu_X$ with $\bbE_{X\sim \mu_X}X  = x^{\ast}$}

		Draw $x_1,\ldots,x_{N_H}$ independently from $\mu_X$ \\
		Build submatrices $\widehat M_p=
			(\opvec ( \Delta_{\epsilon}^2 f_p(x_1) ) | \ldots |    \opvec(\Delta_{\epsilon}^2  f_p(x_{N_H})))$ for $p \in [m_L]$  \\
		Combine submatrices by joining them along their columns $M = \left( \widehat M_1 | \dots | \widehat M_{m_L}\right)$ \\
		Perform SVD $\widehat M= \widehat U\widehat\Sigma \widehat V^\top$ with $\widehat\Sigma$ in descending order.\\
    Let $\widehat U_1$ store the first $m_1+\ldots + m_L$ columns of $\widehat U$ and set $P_{\widehat \CW} := \widehat U_1 \widehat U_1^\top$\\

	\KwOut{$P_{\widehat \CW}$}

	\caption{\textbf{Building the context}}

	\label{alg:subspace_approx}
\end{algorithm}
Judging by the structure of entangled weights for $\ell \geq 2$ and their dependence
on the location $x$ in Proposition \ref{prop:nn_derivatives}, for $L\geq 2$ we cannot expect the Hessians to lie precisely on
a space spanned by  the same rank-$1$ matrices.
As a consequence, the Hessians can not be expected to be contained in $\CW$ anymore,
except of course for the Hessians $\nabla f_1(x^\ast),\ldots,\nabla f_{m_L}(x^\ast)$ whose corresponding
entangled weights span $\CW$. However, by examining the entangled weights more closely,
we can derive the Lipschitz continuity of the matrix valued functions $x\mapsto V_{\ell}(x)$, $\ell \in [L]$,
with Lipschitz constant depending on the network's complexity (Lemma \ref{lem:aux_bounds}), so that Hessians
still concentrate around the subspace $\CW$ for $x$'s concentrating around $x^\ast$.
This motivates to first sample $\numsamples\approx (m_1+\ldots+m_L)/m_L$ locations $X_1,\ldots,X_{\numsamples}$
from a measure $X\sim \mu_X$ with $\bbE X = x^{\ast}$ and which also concentrates tightly around $x^{\ast}$.
Then one approximates the subspace $\CW$ as the span of the leading
$m_1+\ldots+m_L$ right singular vectors of the matrices $M \in \bbR^{D^2 \times m_L\numsamples}$ or $\widehat M \in \bbR^{D^2 \times m_L\numsamples}$ given by
\begin{align*}
M &:= (M_1|\ldots| M_{m_L})\quad \textrm{with}\quad M_p = (\opvec(\nabla^2 f_p(X_1))|\ldots|\opvec(\nabla^2 f_p(X_{\numsamples}))) \in \bbR^{D^2\times m_L \numsamples},\\
\textrm{or}\quad \widehat M &:= (\widehat M_1|\ldots| \widehat M_{m_L})\quad  \textrm{with}\quad \widehat M_p = (\opvec(\Delta_{\epsilon}^2 f_p(X_1))|\ldots|\opvec(\Delta_{\epsilon}^2 f_p(X_{\numsamples}))) \in \bbR^{D^2\times m_L \numsamples},
\end{align*}
depending on whether we have access to derivatives or just to point evaluations of the network $f$.
Instead of using ordinary singular vectors, related to the ordinary principal component analysis, one may also
consider more robust procedures  \cite{10.1145/1970392.1970395,NIPS2010_4005,NIPS2014_5430}.
A summary of the procedure is given in Algorithm \ref{alg:subspace_approx}
whose theoretical analysis is postponed to Section \ref{sec:subspace_approx} (see Theorem \ref{thm:approx_bound}
for the main result).

Notice that we have thus far not specified the distribution $\mu_X$, in particular the choice of
$x^{\ast}$ and the level of concentration. The location $x^{\ast}$ is best chosen as a point,
where the network $f$ has much second order information, i.e., where $\nabla^{2}f(x)$ for
$x \approx x^{\ast}$ is sufficiently rich to allow for recovering the subspace $\CW$. For instance, if the network
has fairly small first layer shifts $\tau_1\approx 0$, $x^{\ast} = 0$ seems to be a suitable choice. Furthermore,
the recentred distribution $X-x^{\ast}$ should have small sub-Gaussian norm of order $\approx 1/\sqrt{D}$, motivating
for instance the generic choice $\mu_X = R\cdot\textrm{Unif}(\bbS^{D-1})$ for small $R > 0$. Furthermore,
if $m_1 \ll D$, which implies the network is active only on a small subspace of $\bbR^D$,
we can first identify the active subspace, see eg. \cite{fornasier2018identification,fornasier2019robust},
and then view the network as a function $f : \bbR^{m_1}\rightarrow \bbR^{m_L}$ instead.

\subsection{Weight recovery}
\label{subsec:weight_recovery}
In the second step we use the output of \textbf{Building the context}, which is the subspace  $\widehat \CW\approx \CW$ with
corresponding orthoprojector $P_{\widehat \CW}\approx P_{\CW}$, to approximately recover
the spanning rank-one matrices of the subspace $\CW$. Recalling the definition of $\CW$
in \eqref{eq:approx_subspace_W}, these are precisely outer products of entangled $v^{[\ell]}_{i}(x^\ast)$,
thus the recovery of
the spanning rank-one matrices yields the  entangled weights. We stress that
we lose sign and scale information by working only with the subspace information $\CW$,
which is however unproblematic in light of the reparametrization result in Proposition \ref{prop:existence_of_equivalent_network}.

To describe how we approach the recovery problem, let us first assume we have access to the exact subspace $\CW$. The
key property, which allows for recovering its
 spanning elements, 
 is that they are \emph{uniquely identifiable} by being matrices of rank one within $\CW$.
More precisely, we can first observe that any global maximizer of the nonconvex program
\begin{align}
\label{eq:optimal_program_no_perturbation}
\max_{\N{u} \leq 1} \Phi(u) := \N{P_{\CW}(u\otimes u)}_F^2,
\end{align}
must be a rank-one matrix in the subspace $\CW$. The program \eqref{eq:optimal_program_no_perturbation}
has been considered implicitly in \cite[Lemma 16 and Lemma 17]{fornasier2019robust} and it has been proposed independently in \cite{fiedlerlearning, kileel2019subspace} in the context of tensor decompositions and neural networks. The extensive analysis in \cite{kileel2019subspace}
has shown that under fairly general conditions on the spanning rank-one elements, i.e., the
entangled weight vectors, there are no more rank-one matrices in the subspace $\CW$
other than original spanning elements. Hence, we can discover in this way the entangled weights up to scale and sign
information within $\CW$
by seeking for global maximizers of the problem \eqref{eq:optimal_program_no_perturbation}.

Returning to the case of disposing of an approximating space $\widehat \CW\approx \CW$  only,
 we thus aim at identifying the entangled weights
by searching global maximizers of the perturbed problem
\begin{align}
\label{eq:optimal_program_perturbation}
\max_{\N{u} \leq 1} \widehat \Phi(u) := \N{P_{\widehat \CW}(u\otimes u)}_F^2.
\end{align}
As we will show in Section \ref{sec:rank_one_recov}, the program exhibits a certain
degree of robustness  such that it is still possible to approximately
recover the entangled network weights under genericity conditions and even in the presence of a perturbation.
However, since the perturbation between $\widehat \CW$ and $\CW$ is highly structured and
not of random type, the recovery guarantees are naturally weaker than in the clean case.

So far we have not described how to find global or even local maximizers of either
program \eqref{eq:optimal_program_no_perturbation} or \eqref{eq:optimal_program_perturbation}.
We actually can use a simple projected gradient ascent algorithm, which has been recently proposed
in \cite{fiedlerlearning, kileel2019subspace}, and starts from randomly sampled
$u_0 \sim \textrm{Unif}(\bbS^{D-1})$ and then iterates
\begin{align}
\label{eq:weight_recovery_iteration}
u_{j} = P_{\bbS^{D-1}}(u_{j-1} + 2\gamma P_{\widehat\CW}(u_{j-1} \otimes u_{j-1}) u_{j-1}),
\end{align}
until convergence (or a convergence criteria is met). Here, $\gamma > 0$ is a gradient step-size. The analysis in
\cite{kileel2019subspace} shows that the algorithm almost surely
avoids saddle points when being randomly initialized, and thus always converges
to a local maximizer of \eqref{eq:optimal_program_perturbation}.
In Section \ref{sec:rank_one_recov} we complement the analysis in \cite{kileel2019subspace} by providing a stability analysis about local maximizers of the
functional $\Phi$ and $\widehat \Phi$, which justifies our use of the method even in case
of a perturbed subspace.

\begin{algorithm}[t]

	\KwIn{$P_{\widehat{\CW}}$, stepsize $\gamma > 0$, repetitions $n$, steps $K$}

		\For{$i = 1\ldots n$}{
			Sample $\widetilde u_0^i \sim\textrm{Unif}(\bbS^{D-1})$\\

			\For{$j = 1 \ldots K$}{
				$\widetilde u_{j}^i =  P_{\bbS^{D-1}}(\widetilde u_{j-1}^i + 2\gamma P_{\widehat \CW}(\widetilde u_{j-1}^i \otimes \widetilde u_{j-1}^i) \widetilde u_{j-1}^i)$\\
			}
			$u^i \leftarrow \textrm{sign}((\widetilde u_{K}^i)_1) \widetilde u_{K}^i$ (sign chosen so that first entry is nonnegative).
		}

	\KwOut{$u^1,\dots,u^n$}

	\caption{\textbf{Weight recovery}}

	\label{alg:recover_weights}

\end{algorithm}

Running iteration \eqref{eq:weight_recovery_iteration} ideally returns a single entangled weight
up to sign and scale (or an approximation thereof). To recover all entangled weights,
we run the iteration several times, say $n=m \ln m + \gamma m + \frac{1}{2} + \mathcal{O}(1 / m)$ times, see, e.g., \cite[Section 8.4]{isaac2013pleasures}, for $m=\sum_{\ell=1}^{L}m_{\ell}$, with randomly sampled initializations $u_0^i \sim \textrm{Unif}(\bbS^{D-1})$. We describe this multiple iteration in Algorithm \ref{alg:recover_weights}.
As clarified in the next section the resulting set $\{u^1,\dots, u^n\}$ returned by Algorithm \ref{alg:recover_weights} is  clustered into $m=\sum_{\ell=1}^{L}m_{\ell}$ groups using the kMeans++ algorithm. The corresponding cluster centers are then used as approximations to the recovered entangled weight vectors.
\\

This concludes the description of the provable recovery of the entangled weights, see Section \ref{sec:subspace_approx}  and Section \ref{sec:rank_one_recov} for the corresponding theoretical analysis. The sections to follow are
included in order to offer a context to our theoretical results and to show the relevance of entangled weights
for the problem of the robust identification of networks. However, for some of the observed numerical
evidences reported below we do not dispose yet of theoretical justifications, which are in the course
of investigation. We will list them in Section \ref{sec:conclusion} as open problems.

\subsection{Weight assignment}
\label{subsec:weight_assignment}

The previous step yields a number of potential entangled weights denoted by $u^1,\dots, u^n$. In the assignment step we need to build out of $u^1,\dots, u^n$  an approximating representative for each entangled weight and we need to assign  to each approximated entangled weight the corresponding layer inside the network.\\
The method sketched below is designed to work for networks with sigmoidal activations, incoherent weights of similar length, and, quite importantly, {several output neurons}. It is devised to distinguish the first and last layer from the inner layers, which means the assignment is only complete for networks up to three layers.
In what follows, it is necessary to know the number of neurons in each layer to perform the assignment, which implicitly reveals the number of layers as well. Inferring the correct network architecture from network queries is an interesting topic for the future. \\\\
As already mentioned, our starting point is a set of approximations of entangled weights. From this point the assignment can be separated into three distinct steps:
\begin{enumerate}
	\item \textbf{Clustering} the vectors $u^1,\dots, u^n$ so that we distill one approximation per entangled weight, which we will call $\tilde v^1, \dots, \tilde v^m$ where $m=\sum_{\ell=1}^{L} m_\ell$ is the number of neurons in the network.
	\item \textbf{Detecting the first layer} by comparing the squared outer product of the entangled weights, $\tilde{v}^1 \otimes \tilde{v}^1, \dots, \tilde{v}^m \otimes \tilde{v}^m$, to Hessians of the network function.
	\item \textbf{Detecting the last layer} can be done by re-applying the idea that was already used during the recovery step. A neural network with $m_L$ outputs can be viewed as $m_L$ networks with $m_L-1$ outputs, by omitting a single output at a time. We use the fact that the entangled weights corresponding to the last layer do not appear in $m_L-1$ of those networks, as opposed to the entangled weights from previous layers, which in principle are shared between all sub-networks.
\end{enumerate}
Let us now detail how these steps are in fact practically realized.

\paragraph{Clustering.}
Assuming the approximations $u^1,\dots, u^n$ are close to the actual entangled weights and every entangled weight was found at least once, then any classic clustering algorithm like kMeans will solve this task. Primarily, because the entangled weights are well separated from each other in our setting, in view of incoherence assumptions. The sign ambiguity can be dealt with by projecting all approximations on one part of the half sphere, as described already in Algorithm \ref{alg:recover_weights}. This clustering step is summarized in Algorithm \ref{alg:clustering}.
\begin{algorithm}[t]

	\KwIn{Approximations of entangled weights $u^1,\dots, u^n$, number of neurons $m$}

	\Begin{
		Run kMeans on $\{{u}^1,\dots, {u}^n\}$ with $m$ clusters centers\\
		Denote those clusters centers by $\tilde{v}^1, \dots, \tilde{v}^m$
	}

	\KwOut{ $\tilde{v}^1, \dots, \tilde{v}^m$ }

	\caption{\textbf{Clustering}}

	\label{alg:clustering}

\end{algorithm}
\paragraph{Detecting the first layer.}
\begin{algorithm}

	\KwIn{Network $f$, approximations of entangled weights $\tilde{v}^1, \dots, \tilde{v}^m$, number of Hessians $N$, input distribution $\mu_X$, number of neurons in first layer $m_1$}

	\Begin{
			Draw $x_1, \dots x_N$ i.i.d from $\mu_X$\\
			\For{$i = 1\ldots m$}{
				Compute $\mathrm{Sim}_1(\tilde{v}^i, \mathcal{H}(\mu_X, N))$\\
			}
      Let $i_1,\ldots,i_{m_1}$ denote the $m_1$ indices maximizing $i\mapsto \mathrm{Sim}_1(\tilde{v}^i, \mathcal{H}(\mu_X, N))$
	}

	\KwOut{Return $\tilde{v}^{i_1}, \dots, \tilde{v}^{i_{m_1}}$ as the set of entangled weights to the first layer}

	\caption{\textbf{Detecting the first layer}}

	\label{alg:detect_first_layer}
	\end{algorithm}
The Hessians of the $p-$th output neuron is of the type
\begin{equation*}
\nabla^2 f_p(x) = V_1(x) S_p^{\lind{1}}(x) V_1(x)^\top + \sum_{\ell=2}^L V_\ell(x) S_p^{\lind{\ell}}(x) V_\ell(x)^\top.
\end{equation*}
The weights of the first layer $V_1(x)= V_1=W_1$ do not actually depend on $x$ and will contribute to any subspace $\CW$ independently of the choice of $x^\ast$.
 All remaining weights appearing in the sums depend on the input $x$ and it is only due to concentration of measure when combining several Hessians that we are able to recover the entangled weights of deeper layers. Hence, we can choose $x$, so that we can sample Hessians in a way which uncorrelates $v^{[\ell]}_i(x)$  from the corresponding entangled weight $v^{[\ell]}_i(x^\ast)$. This uncorrelation can be obtained  by drawing independently the inputs $x_1, \dots, x_N$ according to a suitably chosen distribution $\mu_X$, for example the uniform distribution on the sphere, this time  with large radius $R\gg 0$. Then, in order to detect $V_1=W_1$, we look for those approximated entangled weights, which have the largest correlation to the Hessians of the output neurons at the points $x_1, \dots, x_N$.
Specifically, for the specified distribution $\mu_X$ and chosen number $N$ of Hessians, we use the similarity measure
\begin{align}\label{eq:def_sim}
 \mathrm{Sim}_1(u, \mathcal{H}(\mu_X,N)) := \max_{H\in \cal H} \left|\left\langle u\otimes u,
   \frac{H}{\|H\|_F} \right\rangle_F \right|,
 \end{align}
where $\mathcal{H}:=\mathcal{H}(\mu_X, N) := \{\nabla^2 f_p(x_i) : p \in [m_L],\ x_1, \ldots, x_{N} \sim_{\textrm{iid}} \mu_X\}\subset \bbR^{D\times D}$
contains the $N \times m_L$ Hessian matrices.

\paragraph{Detecting the last layer.}
Given a network $f : \bbR^D \rightarrow \bbR^{m_L}$, let $f_{-p}:\bbR^D\rightarrow \bbR^{m_L-1}$ denote the
subnetwork consisting of all but the $p$-th output. Furthermore, define the corresponding subspace $\CW_{-p}$ as
in \eqref{eq:approx_subspace_W} for $f_{-p}$ (i.e. without the outer product $v_{p}^{\lind{L}}(x^\ast)\otimes v_{p}^{\lind{L}}(x^\ast)$), and
let $P_{-p} := P_{\widehat \CW_{-p}}$ be the orthoprojector constructed according to Algorithm \ref{alg:subspace_approx}
for the network $f_{-p}$.

Under the assumption that all shared weights in layers $1$ to $L-1$ can be recovered reasonably well
from any subnetwork $f_{-p}$, $p \in [m_L]$, we expect all but one outer product
$\tilde{v}^1\otimes \tilde{v}^1, \dots, \tilde{v}^{m}\otimes \tilde{v}^{m}$
to be highly correlated with the subspace $\widehat \CW_{-p}$. The remaining candidate,
which is not correlated with $\widehat \CW_{-p}$, corresponds to the matrix
$v_{p}^{\lind{L}}(x^\ast)\otimes v_{p}^{\lind{L}}(x^\ast)$, i.e., to the $p$-th
entangled weight relative to the last layer.

Algorithmically, we leverage the observation by first computing a score matrix $S \in \bbR^{m_L \times m}$,
where $S_{p,i} := \textN{P_{-p}(\tilde{v}^i\otimes \tilde{v}^i)}_F$. Then, we normalize
the columns of the matrix to norm one (this improves the robustness of the approach if some
weights have not been identified well in the recovery step) and identify the $m_L$ columns of $S$, which have the smallest correlation
with the all-ones vector $\mathbbm{1}_{m_L}$, see Algorithm \ref{alg:detect_last_layer}.

We note that the proposed approach can re-use the Hessian matrices sampled in the Step
`Building the context', so no resampling and additional hyperparameter tuning is required.
Furthermore, if the first layer is detected prior to the last layer, we can remove the corresponding
weights from candidate list $\tilde{v}^1, \dots, \tilde{v}^{m}$, leaving $m-m_1$ potential candidates.
Lastly we add that the approach allows identification of the ordering of weights in the last layer,
which will be an important asset in the `Network completion' step described next.

\begin{algorithm}[t]

	\KwIn{Projection matrices $P_{-1},\dots, P_{-m_L}$, approximations of entangled weights $\tilde{v}^1, \dots, \tilde{v}^{m}$, number of neurons in last layer $m_L$}

	\Begin{
  Compute matrix $S_{p,i} := \textN{P_{-p}(\tilde{v}^i\otimes \tilde{v}^i)}_F \in \bbR^{m_L\times m}$\\
  \For{$i = 1\ldots m$}{
  Compute $\textrm{Sim}_L(i) := \langle \mathbbm{1}_{m_L}, S_{i}/\N{S_{i}}_2\rangle$, where $S_{i}$ is the $i$-th column of $S$
  }
  Let $i_1,\ldots,i_{m_L}$ denote the $m_L$ indices  minimizing $i\mapsto \textrm{Sim}_L(i)$
	}

	\KwOut{Return $\tilde{v}^{i_1}, \dots, \tilde{v}^{i_{m_L}}$ as the set of entangled weights of the last layer}

	\caption{\textbf{Detecting the last layer}}

	\label{alg:detect_last_layer}

\end{algorithm}

\subsection{Network completion}
\label{subsec:network_completion}
Assume now that we have recovered approximations $\widetilde V_1,\ldots, \widetilde V_L$ to the matrices
\begin{align*}
 V_{\ell} = \prod\limits_{k=1}^{\ell-1}(W_k G_k(x^\ast))W_{\ell}\pi_{\ell}S_{\ell} \in \bbR^{D \times m_{\ell}},\qquad \ell \in [L],
\end{align*}
where the diagonal matrix $S_{\ell}$ accounts for missing sign and scaling information in the entangled
weight recovery, and $\pi_{\ell}$ is a permutation matrix accounting for missing order information
in each layer $\ell \in [L]$. According to  Proposition \ref{prop:existence_of_equivalent_network} there remain only $\CO(m_1+\ldots+m_L)$ unidentified parameters  for
reverting the ``entanglement'' of the weights and identifying the network $f$.
\\

Let $\CD_{m}$ denote the set of $m\times m$ diagonal matrices and define
the parameter space $\Omega := \Pi_{\ell = 1}^{L-1} (\R^{m_\ell} \times \CD_{m_\ell}\times \CD_{m_\ell} )\times \CD_{m_L}\times \bbR_{m_L}$.
Any choice of parameters $\omega = ((\widehat \tau_{\ell}, T_{\ell}, R_{\ell})_{\ell \in [L-1]}, \widehat \tau_{L},T_L) \in \Omega$
induces an $L$-layer network $\widehat f: \bbR^{D}\rightarrow \bbR^{m_L}$ defined by  weight matrices $\widehat W_1 = T_1 \widetilde V_{1}$,
\begin{align*}
\widehat W_{\ell+1} = T_{\ell+1}\widetilde V_{\ell+1}^\top (\widetilde V_{\ell})^{\dagger}R_{\ell},\quad \ell \in [L-1],
\end{align*}
and shifts $\widehat \tau_{1},\ldots,\widehat \tau_{L}$. Furthermore, by Proposition \ref{prop:existence_of_equivalent_network}
we know that the parameter choice
\begin{align*}
T_{\ell} = S_{\ell}^{-1},\quad R_{\ell} = S_{\ell}\pi_{\ell}^\top G_\ell(x^\ast)\pi_{\ell},\ \textrm{ and},\ \widehat \tau_\ell = \pi_{\ell}^\top \tau_{\ell}
\end{align*}
recovers the function $\pi_{L}^\top \circ f$, i.e.,
a permutation of the output of the original network $f$. This motivates to reconstruct the network $f$
by fitting the parameters $\omega$, respectively the induced function $\widehat f$,
to a set of input-output queries $\{(X_i,Y_i) : i \in [N_f]\}$  of the network $f$,
where $Y_i = f(X_i)$ and $X_i \sim \CN(0,\frac{1}{D}\Id_{D})$, by using a standard mean-squared error objective.
The parameter fitting can be formulated
as solving the least squares problem
\begin{equation}
\label{eq:new_GD_recover_network_functional}
\min_{\omega \in \Omega}
J(\omega):=\sum_{i=1}^{N_f}\left\| \pi_L^\top Y_i - \hat f(X_i;\omega)\right \|^2.
\end{equation}
We note that the unknown correct permutation $\pi_{L}$ of the last layer can be computed using
the information gathered in the detection of the last layer, see Algorithm \ref{alg:detect_last_layer} and
the accompanying discussion.
Alternatively, we can also add an unknown permutation for the last layer as an additional optimization variable.
We note that due to the identification of the entangled weights and deparametrization of the problem, $
\dim(\Omega) = 3\sum^L_{\ell = 2}m_\ell + 2m_1$, which implies
that the least squares has significantly fewer free parameters compared to
the number
$$
m_1\cdot D + m_1 + \sum_{\ell=2}^L \left( m_{\ell-1}\cdot m_\ell+m_\ell \right)
$$
of original parameters of the entire network. Hence, the preliminary work in this algorithmic pipeline greatly reduces the complexity of the problem with respect to
the usual effort of fitting all parameters at once by means of a backpropagation. More importantly,
we empirically observe that this last optimization yields overwhelmingly accurate uniform approximation of $f$, as soon as the  matrices $V_{\ell}$
are sufficiently well approximated, see Figure \ref{fig:gd_errors}.
This surprising phenomenon may be perceived in contrast to recent results, which state that overparameterization is beneficial for training by gradient descent \cite{du2019gradient}.
We do not present yet here theoretical guarantees of this last empirical risk minimization. Nevertheless its success can be explained - at least locally - by (nested) linearizations of the mean-squared error \eqref{eq:new_GD_recover_network_functional} around the correct parameters and by showing that the linearization is actually uniquely and stably solvable for some probabilistic models of the parameters, see \cite{nguyen2020global,oymak2019moderate,DBLP:journals/corr/SoltanolkotabiJ17} for related techniques. We postpone the detailed analysis to a follow up paper.

\section{Numerical validation of the pipeline}
\label{sec:experiments}
In this section we numerically verify the proposed identification pipeline by extensive
experiments.
We cover all substeps of the pipeline, allowing us to demonstrate and discuss the precise practical realization of the approach.
Let us first briefly introduce the considered network architectures and discuss some hyperparameter
choices in the experiments. A summary of all
hyperparameters with default choices involved in the pipeline can also be found in Table \ref{tab:experiment_notation}.

\begin{figure}[t]
	\centering
	\includegraphics[trim={7.5cm 7.5cm 7.5cm 6.5cm},clip, width=0.95\textwidth]{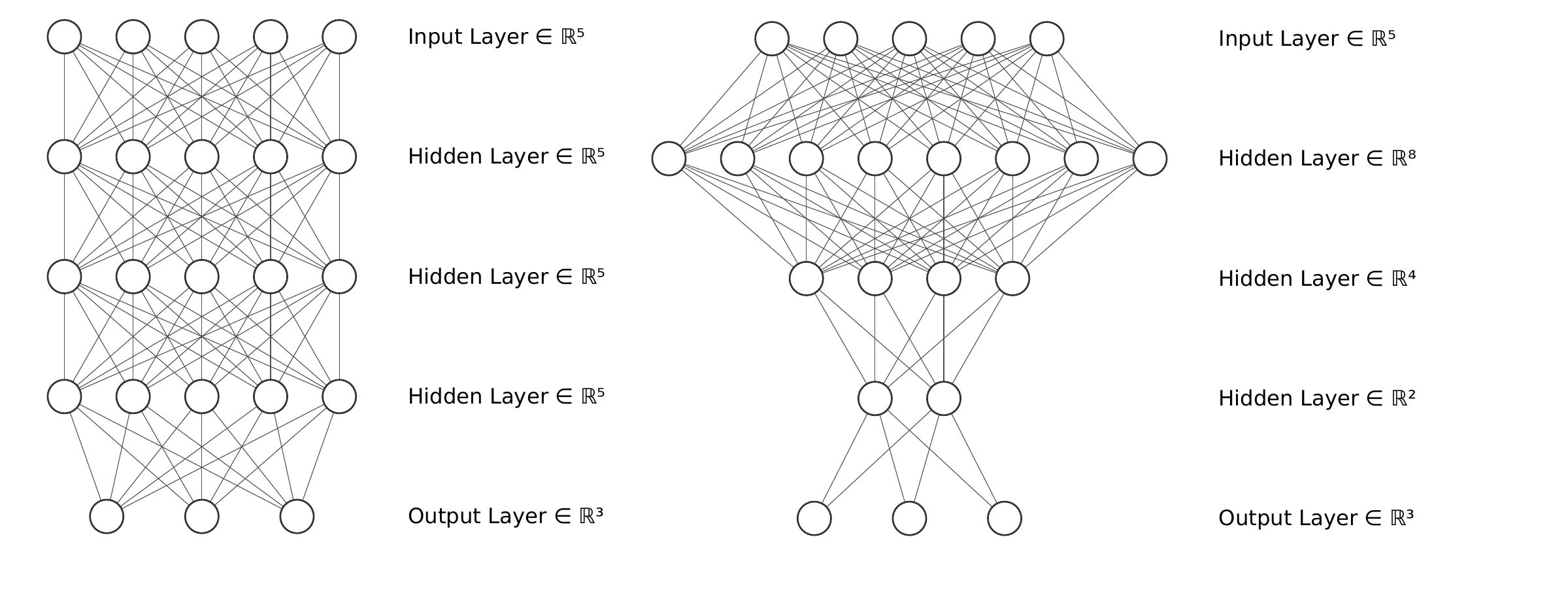}
	\caption{Two exemplary networks specified by the parameters $(D=5, L=4, m_L=3, m=15,c=1)$ on the left and $(D=5, L=4, m_L=3, m=14,c=0.5)$ on the right.}\label{fig:example_networks}
\end{figure}

\begin{table}[t]
\begin{center}
\footnotesize
\begin{tabular}{@{}lllll@{}}
      Network architectures & Step 1 (see Algorithm \ref{alg:subspace_approx}) \\ \toprule
      $D$: ambient dimension (in $\bbN \cap [50,150]$)& $\mu_X$: distribution for Hessians $(R\textrm{Uni}(\bbS^{D-1}))$ \\
      $L$: number of layers (in $\{2,3,4,5\}$) & $R$: radius in $\mu_X$ ($0.01$ unless indicated in Figure)  \\
      $m$: total number of neurons (in $\bbN \cap [50,1500]$) & $N_h$: number of sampled Hessians (max. $20\lceil m/m_L\rceil$) \\
      $m_L$: number of output neurons (in $\bbN \cap [4,10]$) & $\epsilon$: step size finite difference approximation ($0.001$)\\
      $c$: network contraction factor (in $\{0.25,0.5,0.75,1.0\}$) & $x^\ast$: expectation of $\mu_X$ ($0$)\\ \\

      Step 2 (see Algorithm  \ref{alg:recover_weights}) & Step 3 (see Algorithms  \ref{alg:clustering} - \ref{alg:detect_last_layer})\\
      \toprule
       $\gamma:$ gradient ascent step size ($1.5$) & Alg. \ref{alg:clustering} number of cluster centers ($m$)\\
        $K:$ number of steps ($15000$) & Alg. \ref{alg:detect_first_layer} $\mu_X$: distribution for Hessians $(20\sqrt{D}\textrm{Uni}(\bbS^{D-1}))$  \\
        $n:$ number of random initializations $(10000)$ & Alg. \ref{alg:detect_first_layer} $N$: number of Hessians ($400$)\\ \\

      Step 4 (network completion if using GD)& Step 4 (network completion or comparison with SGD)  \\
      \toprule
      learning rate (0.025) & learning rate (0.01)\\
      $N_f$: number of samples in \eqref{eq:new_GD_recover_network_functional} ($\{1000,10000\}$) & number of samples ($D^2 m$)\\
      batch size: all  & batch size ($10^4$)
      \\\bottomrule
\end{tabular}
\end{center}
\caption{Summary of all hyperparameters for the numerical experiments. `Network architectures'
refers to parameters related to the network architecture, Step 1-4 refer to the four steps of the algorithmic pipeline,
and `Comparison with backprop training via SGD' lists parameters for standard backpropagation
via stochastic gradient descent (SGD). A single value in $(\cdot)$-brackets is the default hyperparameter
that is used in all experiments. Otherwise, we list ranges which are tested in experiments.}
\label{tab:experiment_notation}
\end{table}

\paragraph{Network architectures.}
Throughout our experiments we consider network architectures with depths $L \in \{2,\ldots,5\}$,
$\tanh$-activation function, and varying number of total neurons $m=m_1+\ldots+m_L$.
We introduce a shape parameter called \emph{contraction factor} $c \in (0,1]$,
which describes the decrease in the number of neurons from layer to layer, except for the last layer, so that
each network is specified by parameters $(D,L,m_L,m,c)$. Some exemplary networks with different contraction factors are shown in Figure \ref{fig:example_networks}.
In all experiments we sample parameters of the network at random, using $w^{\lind{\ell}}_i \sim \textrm{Unif}(\bbS^{m_{\ell-1}})$
for weights and $(\tau_{\ell})_i \sim \CN(0,0.05)$ for shifts. Analyzing random networks
ensures that all network weights are actually being used in the output, respectively in its Hessian matrices, so that there is a reasonable
hope to fully recover entangled weights and the original network function. Furthermore, randomly sampled
weights are generically incoherent, which is a key requirement of successful rank-one recovery
in $\CW$, according to the theory in Section \ref{sec:rank_one_recov}. We note that our results
empirically further improve if random  weights are replaced by orthogonal weight matrices $W_1,\ldots W_L$
(the latter case is not shown for the sake of brevity), indicating that an increased level of weight incoherence is
beneficial to the accuracy of the proposed algorithms.

\paragraph{Hessian computation.}
The error induced by using finite difference approximations of Hessians
instead of exact Hessians is, to our experience, negligible for the subspace approximation, weight recovery,
and weight assignment step.
Since symbolic differentiation is computationally faster and less memory-intensive, we
therefore conducted the experiments presented in Section \ref{subsec:recovering_entangled_weights} - \ref{subsec:distinction-of-the-entangled-weights}
using symbolic differentation in order to show a wide range of scenarios and parameter choices.

The accuracy of the last pipeline step, `Network completion',
is influenced by the approximation error of entangled weights. Therefore, the experiments in
Section \ref{subsec:network_completion_experiments} are performed using finite difference approximations of Hessians. Specifically,
for a network $f : \bbR^{D}\rightarrow \bbR^{m_L}$ we use
\begin{align}\label{eq:fd_scheme}
\Delta^2_\epsilon f_p(x)_{ij} = \frac{f_p(x + \epsilon e_i +\epsilon e_j) - f_p(x + \epsilon e_i - \epsilon e_j) - f_p(x - \epsilon e_i + \epsilon e_j) + f_p(x - \epsilon e_i - \epsilon e_j)}{4\epsilon^2}
\end{align}
for $i,j \in [D]$, $p\in [m_L]$ and step-size $\epsilon = 0.001$. Here, $e_i$ denotes the $i$-th standard basis vector.

\paragraph{Hyperparameters.}
Each step of the pipeline requires to specify a few hyperparameters, some of which are important
to achieve a good performance. Let us briefly discuss those in the following.

For approximating the subspace $\CW$ we have to specify
a distribution $\mu_X$, so that the corresponding Hessian matrices $\nabla^2 f(X),\ X \sim \mu_X$, carry sufficient information
to detect $\CW$. In all experiments we use the generic choice $\mu_X = R\cdot \textrm{Unif}(\bbS^{D-1})$
for small radius $R>0$, so that $X$ tightly concentrates around the origin $x^\ast = 0$.
This choice is reasonable because the networks considered in our experiments have first
layer shifts $\tau_{1}$ concentrating around zero, and thus we expect the Hessians $\nabla^2 f(X)$ sampled around the origin
to carry much information. The second parameter involved in approximating the subspace $\CW$ is
the number of Hessians $\numsamples$, which we choose in the order of $\CO(m/m_L)$ per output node,
so that we generically have at least $m$ linearly independent Hessians. We note that
choosing $\numsamples$ larger than that does not further benefit the performance, see
also Theorem \ref{thm:approx_bound}.

The performance of recovering the entangled weights within the approximated subspace $\widehat \CW$ improves
if we increase the number of random initializations $n$ and gradient ascent steps $K$ in Algorithm \ref{alg:recover_weights}.
Since the computational burden of this step is relatively
small, we choose  $n = 10000$ and $K = 15000$ throughout. The performance seems
robust with respect to moderate choices of the step size $\gamma$, as long as very small $\gamma$
is compensated with a perhaps larger number of steps $K$. We use $\gamma = 1.5$ in all experiments.

For the layer assignment of the entangled weights, we again have to specify a distribution for detecting the first layer. Recalling that the analyzed networks
have small first layer shifts $\tau_1 \approx 0$, the choice $\mu_X = 20\sqrt{D}\cdot \textrm{Unif}(\bbS^{D-1})$
works reasonably well. Furthermore, we used $N = 400$ Hessians
to compute the similarity measure defined in \eqref{eq:def_sim}.

The remaining parameters are either discussed in the respective subsections below,
or have default values that can be found in Table \ref{tab:experiment_notation}.

\subsection{Recovering entangled weights}
\label{subsec:recovering_entangled_weights}
In this part we focus on the first two steps of the pipeline, which deal with approximating the subspace
$\CW$ as in \eqref{eq:approx_subspace_W}, and recovering the entangled weights at $x^\ast = 0$ from $\CW$ by
their rank-one property. Since we can not recover scale information by the proposed algorithm, we define
the normalized weights by $\tilde v_{i}^{[\ell]}(0) := v_{i}^{[\ell]}(0)/\textN{v_{i}^{[\ell]}(0)}_2$
for $i \in [m_\ell]$, $\ell \in [L]$.
To evaluate our results, we report the following metrics.
\begin{enumerate}
	\def\labelenumi{\arabic{enumi}.}
	\item The subspace distance $\textN{P_{\CW}-P_{\widehat \CW}}_F/\sqrt{m}$, where we scale the error by $\sqrt{m}$ since $\dim(\CW) = m$.
	\item
Denoting  $B_{r}^D(x) = \{x'\in \bbR^D : \N{x-x'}_2 \leq r\}$, we define the  recovery rate of the entangled weights as
	$$ \mathrm{Recov}(u^1,\ldots, u^{n}):= \frac{\#\left(\left\{\tilde v_{i}^{\lind{\ell}}(0) \middle|  i \in [m_{\ell}],\ \ell \in [L]\right\} \cap \left( \bigcup_{k \in [n], s \in \{-1,1\}}B_{0.05}^D(s u^k)\right)\right)}{m}. $$
	\item The rate of false positives, i.e., the fraction of $u^1\ldots u^{n}$ which is not contained in a
  Euclidean ball of radius $0.05$ around any $\pm v_{i}^{\lind{\ell}}$. Denoting $A^C$ as the complement of a set $A$, we define
  $$
  \mathrm{FalsePos}(u^1,\ldots, u^{n}) := \frac{\#\left(\{u_1,\ldots,u^n\} \cap \left( \bigcup_{i \in [m_L],\ \ell \in [L], s \in \{-1,1\}}B_{0.05}^D\left(s\tilde v_{i}^{\lind{\ell}}(0)\right)\right)^C\right)}{n}.
  $$
\end{enumerate}

\paragraph{Exploring different architectures by varying $L$, $c$, and $m$.}
We fix $D = 100$, $m_L = 10$ and consider combinations of number of neuros $m \in \{200,300, \ldots, 1500\}$, contraction factor $c \in \{0.25,0.5,0.75,1\}$, and network depth $L\in \{2,3,4\}$.
Furthermore, we consider different distributions $\mu_X$ by varying $R \in \{0.001,\ldots,10\}$.
Figures \ref{fig:recovery_space_error} - \ref{fig:recovery_ratio_fp} report the recovery results using
the above introduced metrics and show an overall convincing performance of the method for $R \in (0.01,1.0)$.
Specifically, Figures \ref{fig:recovery_ratio_recovery} - \ref{fig:recovery_ratio_fp} show
that we recover all $m$ neurons with high empirical probability if $m < 1000$ without suffering any false positives,
 while approximating the subspace $\CW$ well. Considering the
experiments with $R = 10$, a mistuned radius still allows for recovering many entangled weights
while suffering a rather small fraction of false positives, at least in cases $L \in \{2,3\}$. This is rather
surprising and indicates a degree of robustness of our approach, because we are
not expecting much concentration of entangled weights $v_i^{\lind{\ell}}(x)$ around
$v_i^{\lind{\ell}}(x^\ast)$ for such a choice of $\mu_X$.

A phase transition of all tracked metrics can be observed at around $m = 1000$ neurons. It seems that
fewer layers and low contraction factor $c$ are beneficial for the recovery, but the effect is almost negligible.
As we discuss extensively in the next part, the tipping point $m = 1000$ is related to input and output dimension (here $D = 100$ and $m_L = 10$).

\begin{figure}
	\centering
	\includegraphics[width=\textwidth]{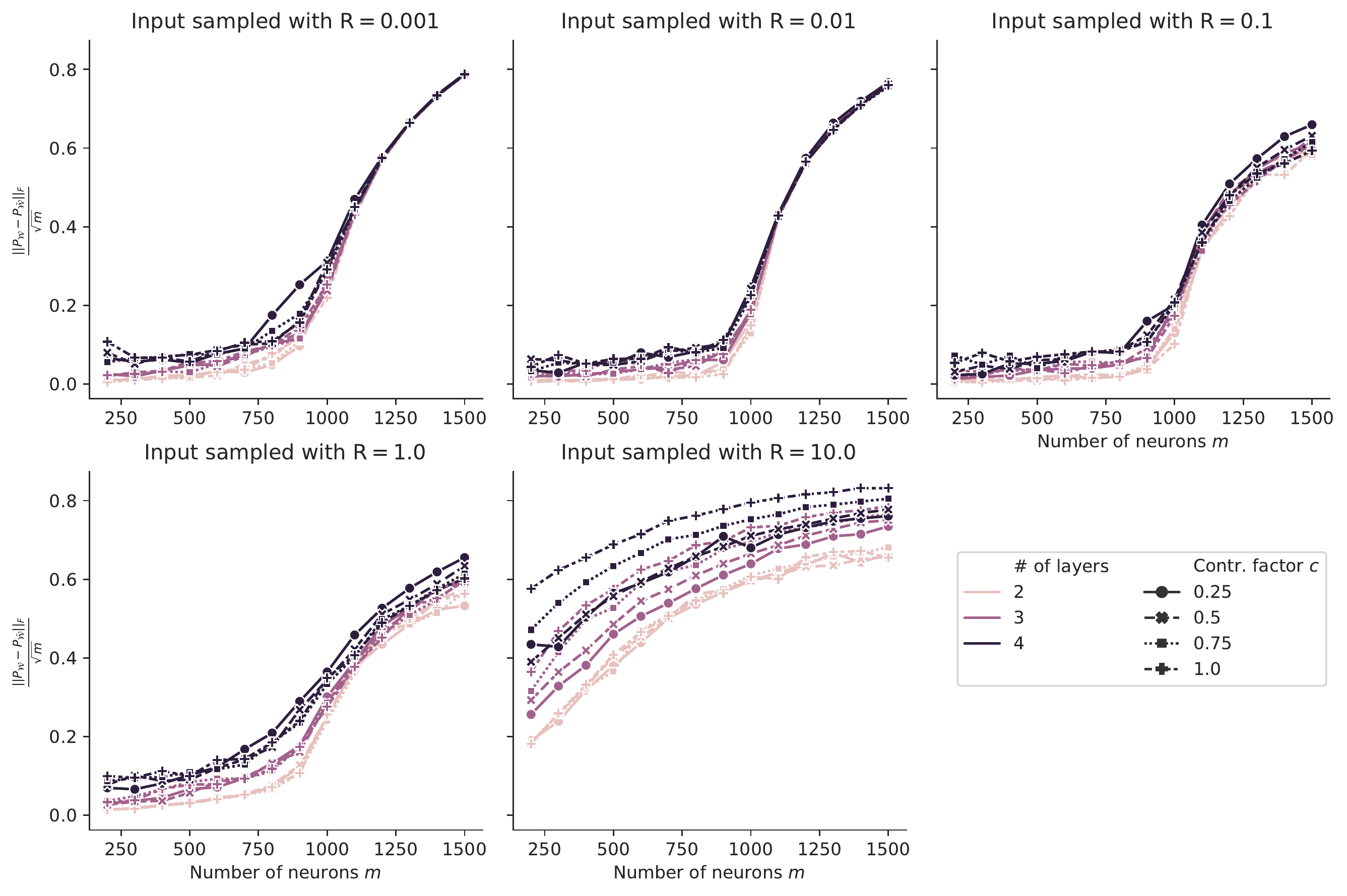}
	\caption{Approximation error of the matrix space measured by the distance of the projections for different architectures and input distributions.}
	\label{fig:recovery_space_error}
\end{figure}
\begin{figure}
	\centering
	\includegraphics[width=\textwidth]{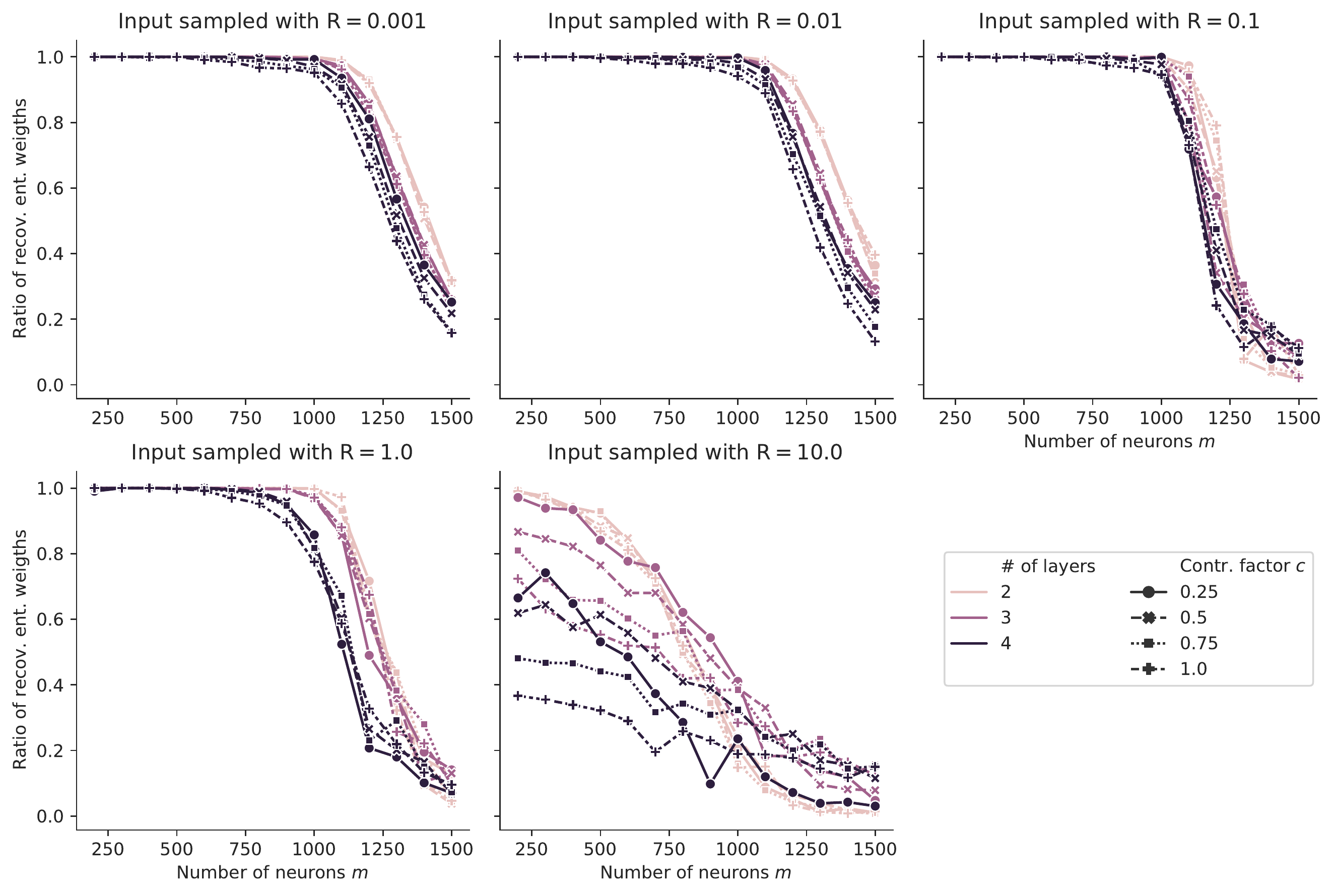}
	\caption{The percentage of overall recovered entangled weights for different architectures and input distributions.}
	\label{fig:recovery_ratio_recovery}
\end{figure}

\begin{figure}
	\centering
	\includegraphics[width=\textwidth]{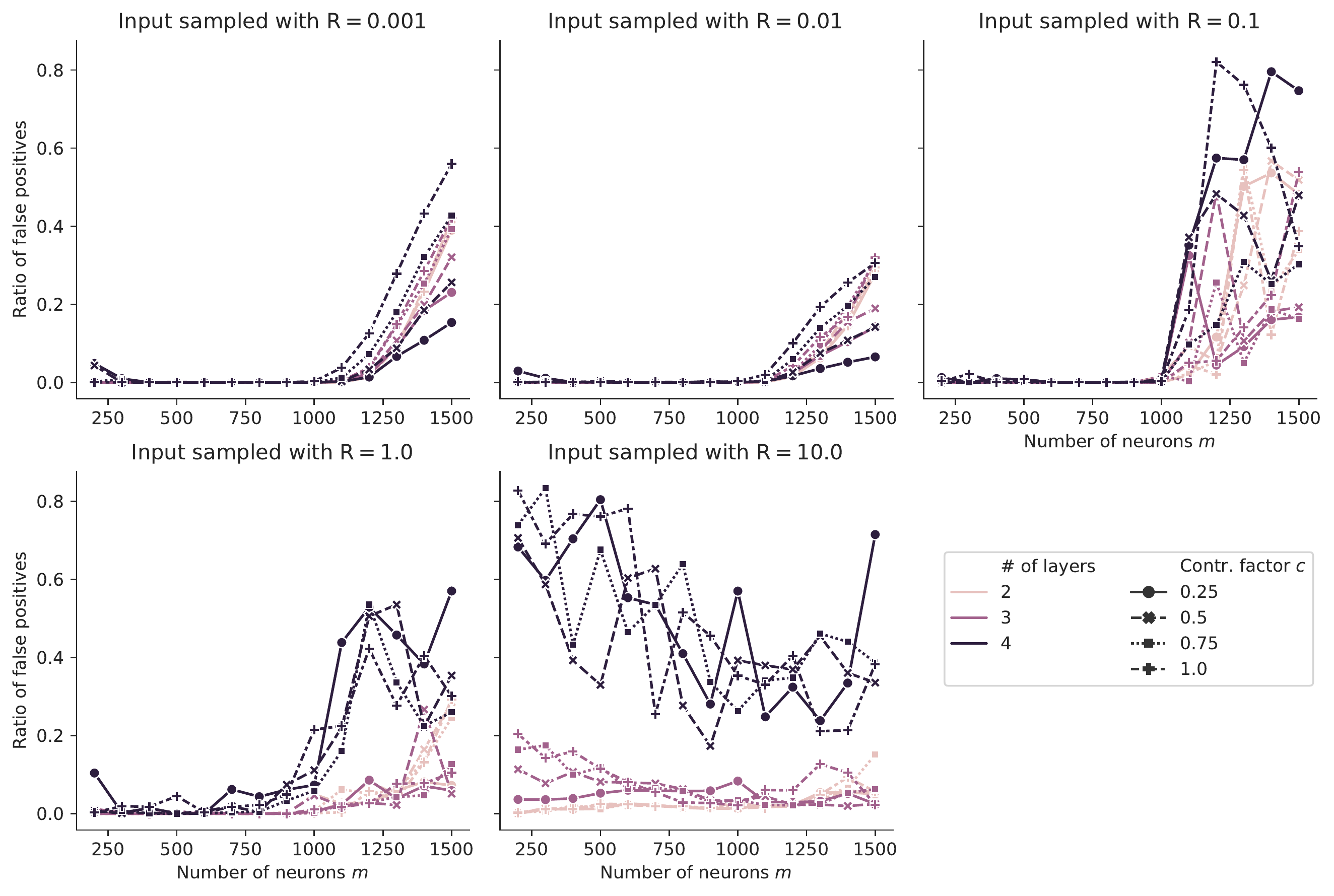}
	\caption{The false positive rate of our recovery for different architectures and input distributions. }
	\label{fig:recovery_ratio_fp}
\end{figure}

\paragraph{Varying input and output dimensions $D$ and $m_L$.}
We now fix $R = 0.01$, $c = 0.5$, and test different input and output dimensions $D$
and $m_L$, as well as different numbers of neurons $m$.
The results, presented in Figures \ref{fig:dim_recoverror} - \ref{fig:dim_spaceerror},
show that $D$ and $m_L$ have a much stronger influence on the recovery results compared
to the number of layers $L$ or the shape parameter $c$ as discussed in the previous paragraph. Namely,
by increasing $D$ and $m_L$, we observe a significant increase of the empirical probability of perfect recovery up to
a number of neurons prescribed roughly by the threshold $m \approx D\times m_L$.

Based on our theoretical understanding from Section \ref{sec:subspace_approx}, the strong influence of $D$ and $m_L$ can be attributed mainly to two factors.
First, for increasing $D$, the same number of entangled weights become increasingly  incoherent, which
plays beneficial towards both the subspace approximation and  for performing
weight recovery within $\CW$ by the subspace power method.
Second, increasing $m_L$ effectively adds additional channels to the network output
and thus increases the information contained in the Hessian distribution $\nabla^2 f(X),\ X \sim \mu_X$.
Intuitively, we can think that each weight may or may not contribute to an output $f_p$
at points sampled from distribution $\mu_X$, but the likelihood that it does not contribute
(or contribute a small amount) to all $m_L$ outputs decreases rapidly when increasing $m_L$.
Therefore, adding network outputs while sharing many of the network weights (up to layer $L$), greatly benefits the weight identification.

\begin{figure}
	\centering
	\includegraphics[width=\textwidth]{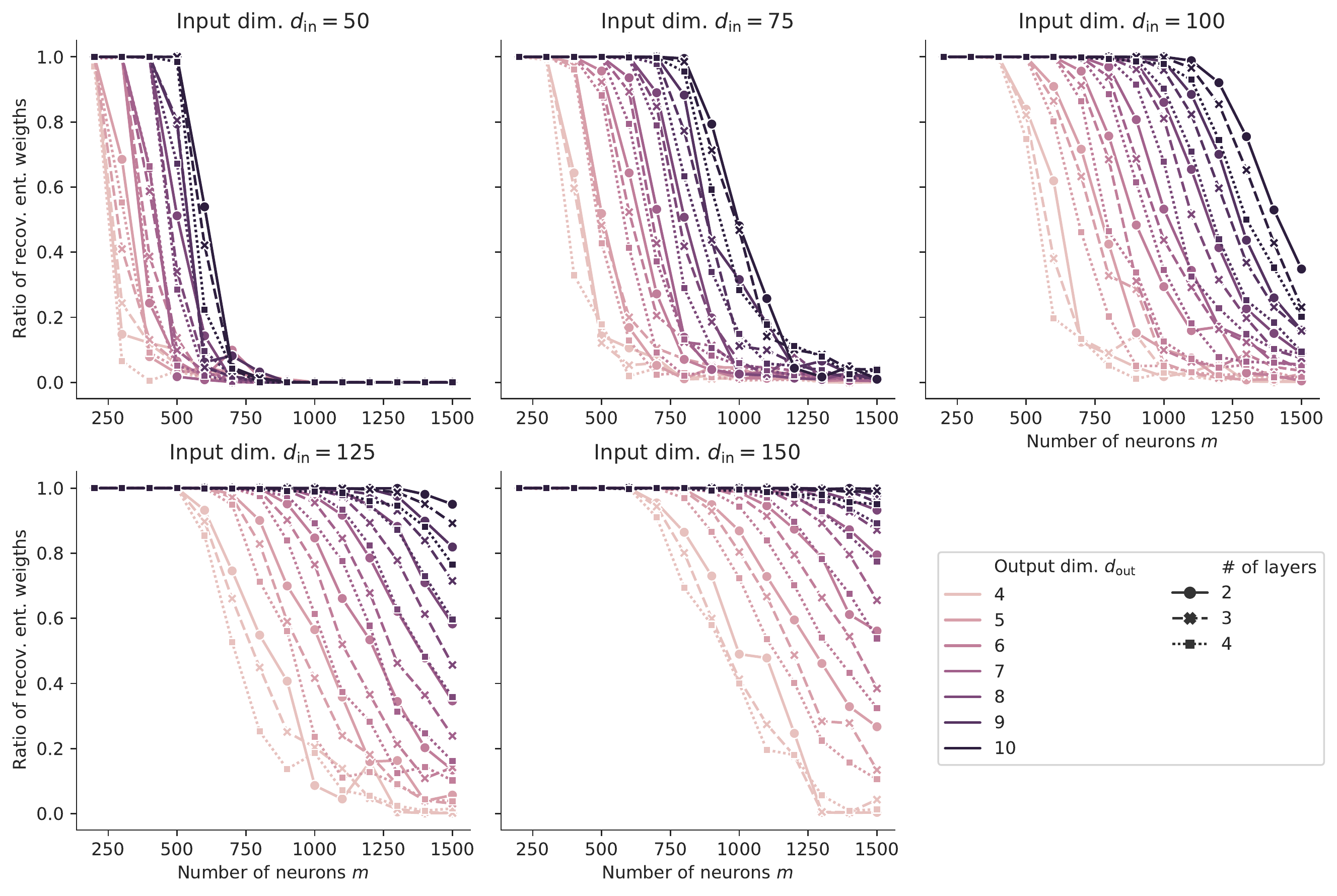}
	\caption{Ratio of recovered weights for different input/output dimensions.}
	\label{fig:dim_recoverror}
\end{figure}

\begin{figure}
	\centering
	\includegraphics[width=\textwidth]{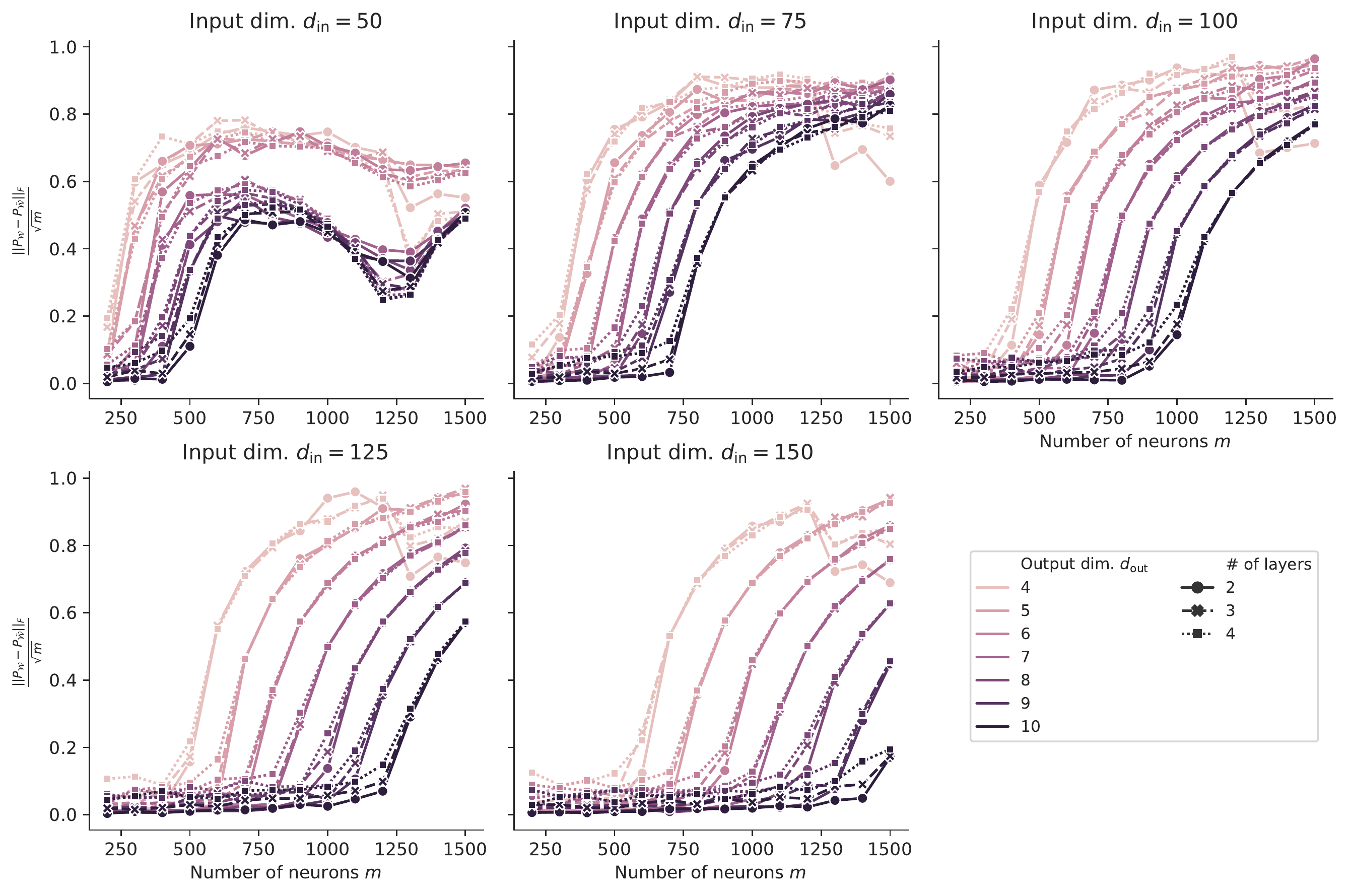}
	\caption{Error of the approximated matrix space for different input/output dimensions.}
	\label{fig:dim_spaceerror}
\end{figure}

\hypertarget{distinction-of-the-entangled-weights}{%
\subsection{Weight assignment}\label{subsec:distinction-of-the-entangled-weights}}
We now test the Algorithms \ref{alg:clustering} - \ref{alg:detect_last_layer} to assign the output of Algorithm \ref{alg:recover_weights}
to three different classes, representing first layer weights, last layer weights, and inner layer weights.
Recall that the clustering step in Algorithm \ref{alg:clustering} reduces the candidates $\{u^1,\ldots,u^n\}$,
obtained in the previous step,
to $m$ vectors $\tilde v^1\ldots \tilde v^{m}$, which ideally correspond to the entangled network weights. We denote
by $\CV_1$ a subset of $\{\tilde v^1\ldots \tilde v^{m}\}$
containing $m_1$ weights assigned to the the first layer by Algorithm \ref{alg:detect_first_layer}, and by $\CV_L$
a subset of $\{\tilde v^1\ldots \tilde v^{m}\}\setminus \CV_1$ containing $m_L$ weights
assigned to the last layer by Algorithm \ref{alg:detect_last_layer}. For a fixed
network $f$ with length normalized entangled weights $\tilde v_{j}^{\ell}(x^\ast)$, we report the error measures
\begin{align*}
E_1 := \max_{\tilde v \in \CV_1} \min_{j \in [m_1], s \in \{-1,1\}}\N{s \tilde v_{j}^{\lind{1}}(x^\ast)-\tilde v}_2,\quad \textrm{ and }\quad
E_L := \max_{\tilde v \in \CV_L} \min_{j \in [m_L], s \in \{-1,1\}}\N{s \tilde v_{j}^{\lind{L}}(x^\ast)-\tilde v}_2,
\end{align*}
corresponding to the worst-case $\ell_2$-error of any of the assigned weights to a corresponding true
normalized entangled weights.  For all considered networks we have
$m_L = 10$ outputs and input dimension $D = 50$. Furthermore, we analyze a variety of different architectures
such as $m_1 \in \{100,\ldots,200\}$ and $m_2 = m_L = 10$
for two layer networks (note that the contraction factor does not matter in this case), over to $m_1 = 50,\  m_2 = 50,\  m_3 = m_L = 10$ (number of neurons $100$ and
$c = 1$) or $m_1 = 125,\  m_2 = 75,\  m_3 = m_L = 10$ (number of neurons 200, $c = 0.6$) for
three layer nets,
or $m_1 = 68,\  m_2 = m_3 = 66,\  m_4 = m_L = 10$ (number of neurons 200, $c = 1.0$)
for four layer nets.

All results reported in Table \ref{tab:accclustering1} show near zero $\ell_2$-error and thus indicate that
assignment to the first, respectively, last layer based on the proposed procedures is effectively optimal
given the output of the the first two steps of the pipeline. The remaining small error
is not caused within the assignment step, but is rather due to inaccuracies in the first two steps of the algorithm.
According to our experience, the proposed assignment routines typically perform well whenever
the first two steps of the pipeline work well, thus implying that assignment to the first and last layer
does not present a critical bottleneck of the approach.
We also note that the results in Table \ref{tab:accclustering1} for $2$-layer networks are a substantial improvement
over \cite{fornasier2019robust}, where we  introduced other heuristics for weight assignment
in $2$-layer nets and near orthonormal weight matrices $W_1$, $W_2$. The difference between
the procedures is that Algorithm \ref{alg:detect_first_layer} is based on Hessian information,
whereas the procedure in \cite{fornasier2019robust} only relies on gradient information of the
network and saturation effects of the sigmoidal activation function, see \eqref{hierarchy}.

\begin{table}
	\footnotesize
	\begin{tabular}{c|cc|cccccc|cccccc}
		& \multicolumn{2}{|c|}{2 layers} & \multicolumn{6}{|c|}{3 layers} & \multicolumn{6}{|c}{4 layers} \\
		Contraction $c$ & \multicolumn{2}{|c|}{1.0}  & \multicolumn{2}{|c}{0.6} & \multicolumn{2}{c}{0.8} & \multicolumn{2}{c|}{1.0} & \multicolumn{2}{|c}{0.6} & \multicolumn{2}{c}{0.8} & \multicolumn{2}{c}{1.0} \\
     \#Neurons $m$  & & & & & & & & & & & & & & \\
		 $\downarrow$ & $E_1$ & $E_L$ & $E_1$ & $E_L$ & $E_1$ & $E_L$ & $E_1$ & $E_L$ & $E_1$ & $E_L$ & $E_1$ & $E_L$& $E_1$ & $E_L$ \\
		\midrule
	 100 &  0.02 &  0.01 &  0.01 & 0.01 &  0.03 & 0.01 &  0.02 & 0.01 &  0.02 & 0.05 &  0.03 & 0.02 &  0.03 & 0.02 \\
		120 & 0.01 &  0.00 &  0.02 & 0.01 &  0.02 & 0.01 &  0.02 & 0.01 &  0.03 & 0.02 &  0.03 & 0.03 &  0.02 & 0.02\\
     140 &  0.01 & 0.00 &  0.02 & 0.07 &  0.02 & 0.02 &  0.02 & 0.01 &  0.02 & 0.02 &  0.04 & 0.04 &  0.03 & 0.02 \\
     160 &  0.01 & 0.00 &  0.02 & 0.01 &  0.02 & 0.01 &  0.02 & 0.01 &  0.03 & 0.02 &  0.03 & 0.03 &  0.04 & 0.03 \\
     180 &  0.01 & 0.00 &  0.02 & 0.01 &  0.02 & 0.02 &  0.02 & 0.01 &  0.03 & 0.02 &  0.03 & 0.02 &  0.03 & 0.04 \\
     200 &  0.01 & 0.00 &  0.02 & 0.01 &  0.02 & 0.01 &  0.02 & 0.02 &  0.03 & 0.03 &  0.03 & 0.02 &  0.03 & 0.02 \\
		\bottomrule
	\end{tabular}
	\caption{Results for the first and last layer weight assignment using Algorithms \ref{alg:clustering} - \ref{alg:detect_last_layer}.
  Each number equals the worst observed $\ell_2$ error between an approximated weight assigned to the first, respectively, last layer
  and a true weight in the first, respectively, last layer. The results are averaged over 5 trials. The input dimension equals $50$,
	and the architectures are described by the number of neurons $m$ and the contraction factor $c$. The results consistently show
  perfect assignment of weights to the first and last layer over various different architectures.}
	\label{tab:accclustering1}
\end{table}

\subsection{Network completion}
\label{subsec:network_completion_experiments}
\begin{figure}
	\centering
	\includegraphics[width=\textwidth]{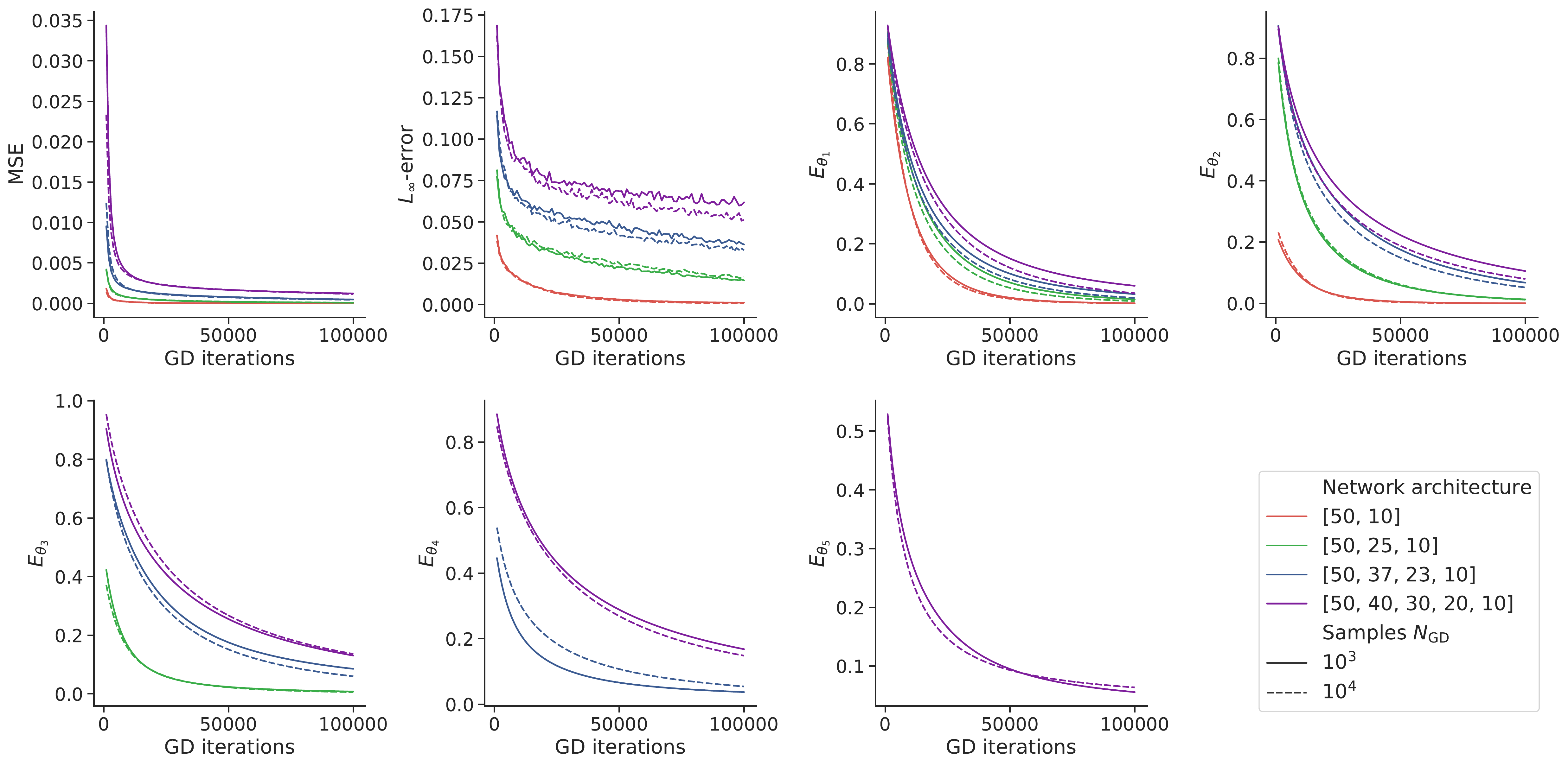}
	\caption{Minimizing the modified objective \eqref{eq:new_GD_recover_network_functional} with ordinary gradient descent to generate the final network approximation from exact entangled weights. The resulting network converges in parameter space, leading to a uniformly low test (generalization) error.}
	\label{fig:gd_errors}
\end{figure}
In this final part we illustrate the benefit of enriching a standard least squares problem,
designed to learn
the original network function and possibly network parameters, with entangled
weight information by means of the reparametrized objective
\eqref{eq:new_GD_recover_network_functional}. Furthermore, we report full pipeline experiments, where each algorithmic step is
performed successively, to assess whether the pipeline allows for identifying the original network
function uniformly well. We concentrate on network architectures with input dimension $D = 50$,
output dimension $m_L = 10$, and a pyramidal shape $m_1\geq \ldots \geq m_L$ as specified under `Network architecture'
in each figure.

We consider two types of experiments. The first experiment is designed to test the efficiency of
the modified objective \eqref{eq:new_GD_recover_network_functional} in isolation, so we assume access to exact entangled weights
of a network and we check whether minimazing the objective \eqref{eq:new_GD_recover_network_functional}
over parameters $\omega \in \Omega$ allows for uniformly learning the network and its parameters.
We use $N_f \in \{10^3,10^4\}$ training samples, initialize shifts $(\hat{\tau}_\ell)_{\ell \in [L]}$ by $0$,
and diagonal matrices ${(T_\ell)_{\ell\in [L]}}, (R_\ell)_{\ell\in [L]\setminus\left\lbrace 1\right\rbrace}$ by identity matrices.
We report the relative mean squared error (MSE) and the relative $L_{\infty}$-error, given by
\[
\textrm{MSE} = \frac{\sum_{i=1}^{N_{\textrm{test}}} (\hat f(Z_i;\omega^*) - Y_i)^2}{\sum_{i=1}^{N_{\textrm{test}}}Y_i^2},\quad
\textrm{E}_{\infty} = \frac{\max_{i \in [N_{\textrm{test}}]}\SN{\hat f(Z_i;\omega^*) - Y_i}}{\max_{i \in [N_{\textrm{test}}]}\SN{Y_i}},
\]
on a randomly sampled test set $\{(Z_i, Y_i:=f(Z_i)) : i \in [N_{\mathrm{test}}]\}$ of size $N_{\mathrm{test}} = 10^5$ as measures of the generalization or extrapolation error. Furthermore,
we report relative  shift errors
\[
E_{\theta_\ell} = \frac{\N{\tilde\tau_\ell - \tau_\ell}^2}{\N{\tau_\ell}^2}, \quad \ell \in [L],
\]
to evaluate whether we also learn the remaining network parameters well.

The second experiment is designed to analyze the efficiency of the proposed pipeline in comparison
with other methods such as stochastic gradient descent for a standard least squares problem. Specifically,
we report $\textrm{MSE}$ and $\textrm{E}_{\infty}$ for the following methods:
\begin{enumerate}[label = \textbf{M\arabic*}]
  \item \label{enum:standard_method} stochastic gradient descent with random initialization on standard least squares, where the network architecture matches
  the architecture of the data generating network;
  \item \label{enum:standard_initial_entangled} as \ref{enum:standard_method}, but the weight matrices $W_1,\ldots,W_L$ of the network are initialized by approximated entangled
	weights, i.e., by $W_\ell = (\widehat V_{\ell-1})^{\dagger} \widehat V_{\ell}^\top \in \bbR^{m_{\ell-1}\times m_{\ell}}$ with $\widehat V_{\ell} \in \bbR^{m_{\ell}\times D}$ storing columnwise
	the approximated entangled weights associated to layer $\ell$;
  \item \label{enum:full_pipeline} ordinary gradient descent applied to the modified objective \eqref{eq:new_GD_recover_network_functional}
  with approximated entangled weights (this method actually corresponds to our full pipeline);
  \item \label{enum:exact_entangled} like \ref{enum:full_pipeline} but using exact entangled weights to test \eqref{eq:new_GD_recover_network_functional} in isolation.
\end{enumerate}
The methods \ref{enum:standard_method} and \ref{enum:standard_initial_entangled} have access to $D^2\times m$ randomly sampled training data points,
which is comparable to the number of queries that is used within our pipeline.
\ref{enum:full_pipeline} and \ref{enum:exact_entangled} use $N_f = 10^4$ randomly sampled training data points
for the network completion step. In methods \ref{enum:standard_initial_entangled} and \ref{enum:full_pipeline},
the weight assignment
is done via an oracle assignment for networks that have more than $L > 3$ layers. This is because we currently
do not dispose of a suitable method for assigning weights for networks with depth  $L > 3$. The remaining steps
are conducted as explained in Section \ref{sec:algorithm}, by using finite difference
approximations to Hessians.

\paragraph{Results: isolated test of modified objective \eqref{eq:new_GD_recover_network_functional}.}
The results of the first experiment, averaged over 10 repititions, are shown in Figure \ref{fig:gd_errors}.
We can clearly see that running ordinary gradient descent on the modified objective \eqref{eq:new_GD_recover_network_functional}
greatly benefits the learning, because all error measures quickly converge
 to $0$. The observed convergence in parameter space and function $L_{\infty}$-error in Figure \ref{fig:gd_errors}
imply that the original network parameters are recovered. As we will see in the next experiment,
this is not the case when trying to fit a neural network via method \ref{enum:standard_method}, i.e., a standard backpropagation from scratch.
We also note that the number of
training samples, ranging from $10^3$ to $10^4$, has negligible effect on the accuracy, which indicates
that a number of training samples $N_f$ equal or slightly above  the number of remaining free parameters
suffices for the network completion step.

\begin{figure}
  \begin{subfigure}{.5\textwidth}
    \centering
    \includegraphics[trim=0 0 0cm 0,width=1\linewidth]{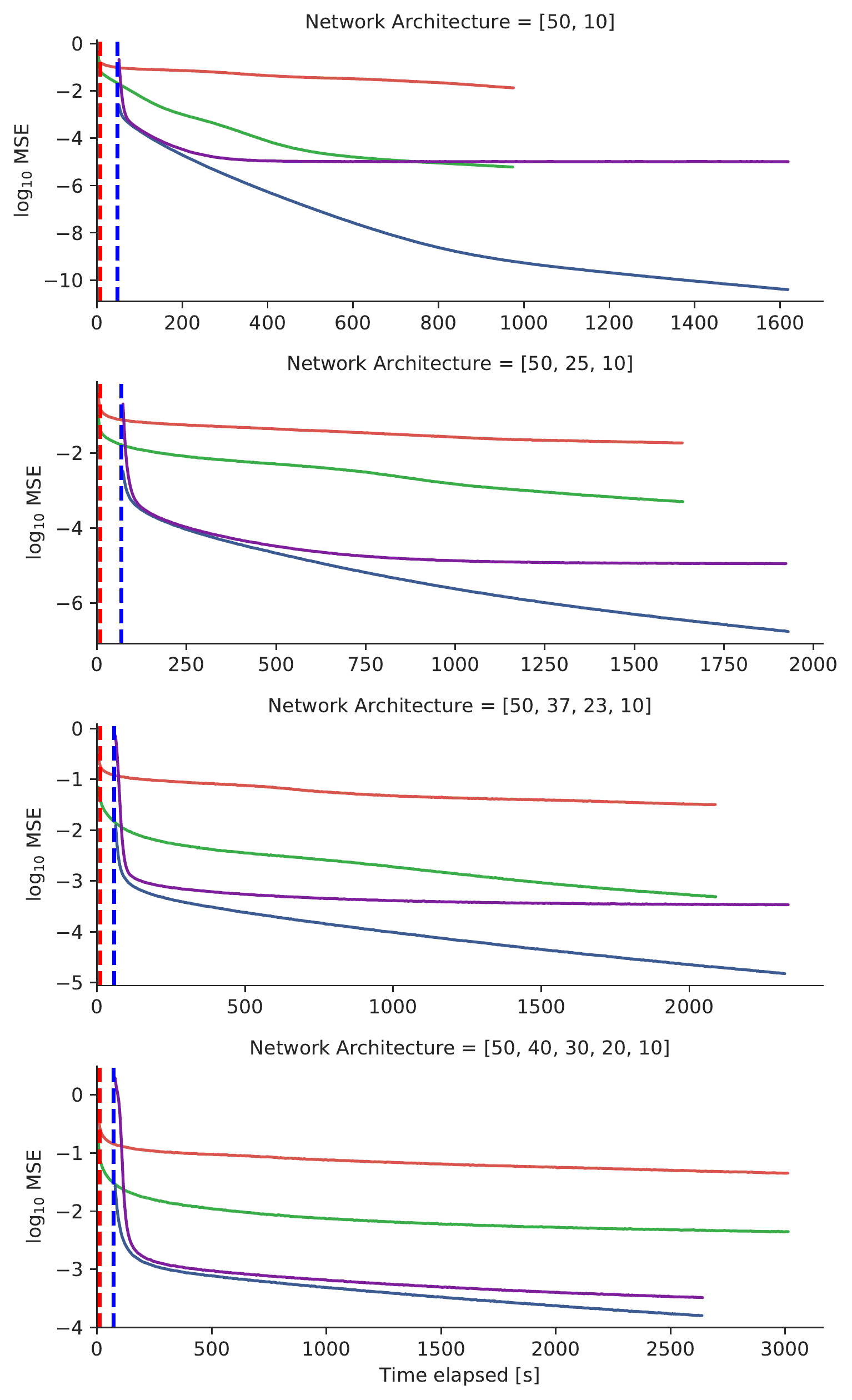}
    \caption{Relative MSE}
    \label{fig:motivation_1}
  \end{subfigure}
  \begin{subfigure}{.5\textwidth}
    \centering
    \includegraphics[trim=0 0 10.215cm 0, clip, width=1\linewidth]{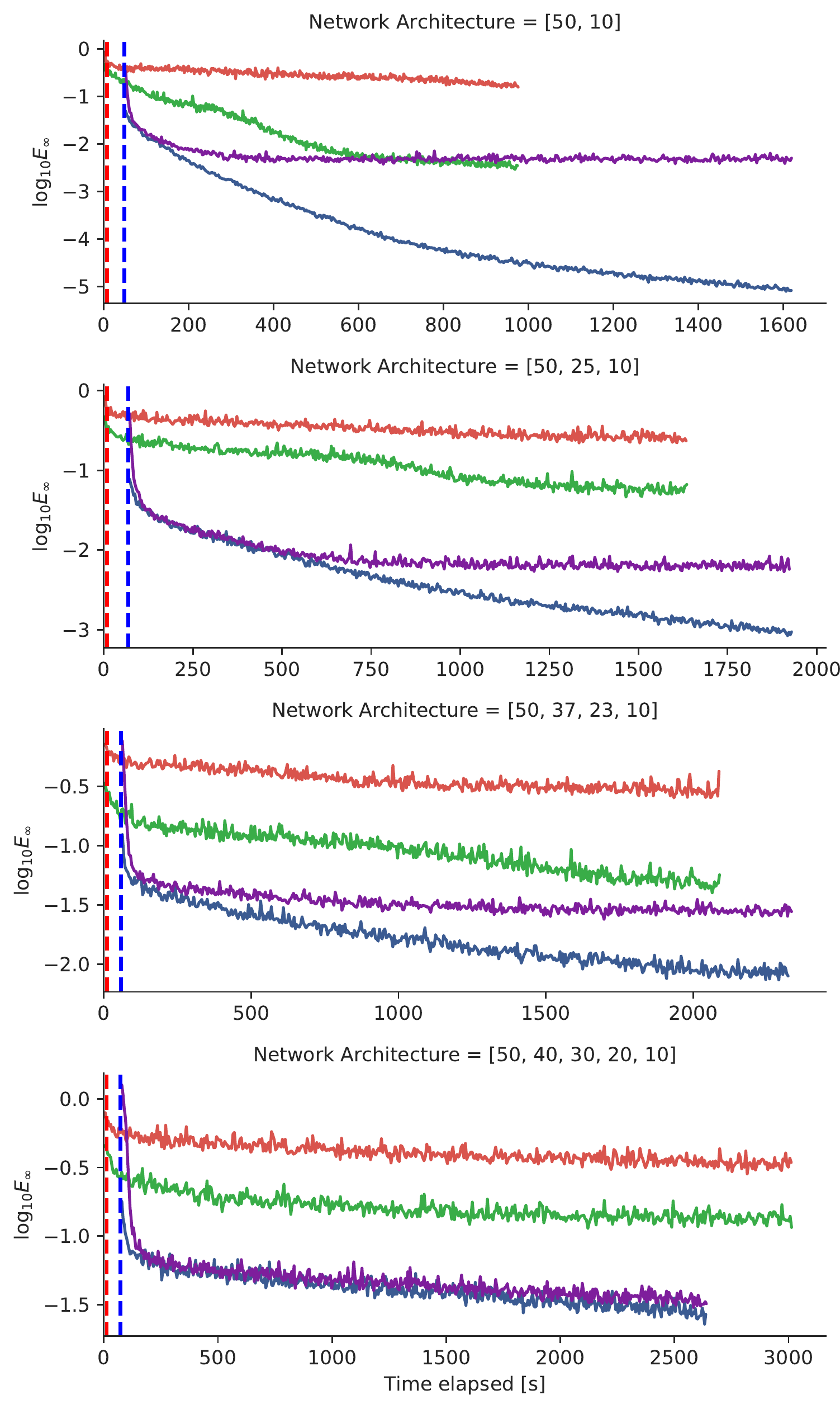}
    \caption{Relative $L_{\infty}$-error}
    \label{fig:motivation_21}
  \end{subfigure}
  \begin{subfigure}{1.0\textwidth}
    \centering
    \includegraphics[trim=15.25cm 10.75cm 0cm 10cm, clip, width=0.4\linewidth]{GD_error_comparison_Linf_redone.pdf}
  \end{subfigure}
\caption{Approximation of the original network function with methods \ref{enum:standard_method} - \ref{enum:exact_entangled}. (Red): Stochastic gradient descent on the standard least squares, randomly initialized.
(Green): Stochastic gradient descent with initialization by approximated entangled weights. (Purple): Gradient descent on the
modified objective \eqref{eq:new_GD_recover_network_functional} using approximated entangled weights. (Blue): Gradient descent
on the modified objective \eqref{eq:new_GD_recover_network_functional} using exact entangled weights. The $x$-axis depicts
the training time in seconds, and the blue and red vertical bars indicates the time consumed by steps 1 and 2 of the pipeline.}\label{fig:gd_comparison}
\end{figure}

\paragraph{Results: full pipeline tests and comparison with standard training \eqref{eq:new_GD_recover_network_functional}.}
Figure \ref{fig:motivation_1} and \ref{fig:motivation_21} report relative function mean-squared error and function $L_{\infty}$-error for the
four different learning methods described in \ref{enum:standard_method} - \ref{enum:exact_entangled}.
First, note that stochastic gradient descent
on the standard objective (red line) fails at achieving both small function mean-squared error or function $L_{\infty}$-error
for all network architectures. These results improve dramatically if we enrich the learning procedure by
initializing the weight matrices as approximated entangled weights, recovered through the proposed pipeline,
instead of using a random initialization (green line). The results for the full pipeline, i.e., for the modified
objective \eqref{eq:new_GD_recover_network_functional} with fixed approximated entangled
weights, are shown in purple. Compared to stochastic gradient descent on the standard least squares,
the results of method  \ref{enum:full_pipeline} are superior, as it achieves substantially smaller error in both
reported measures. Instead, the full pipeline does not quite reach the same accuracy as \ref{enum:exact_entangled}, which runs
gradient descent on  \eqref{eq:new_GD_recover_network_functional} initialized by using exact entangled weights (blue line). This is due to small irreparable errors, accumulated in the first two steps of the full pipeline in \ref{enum:full_pipeline}.

\section{The approximation of entangled weight subspaces}
\label{sec:subspace_approx}
In this section we provide a theoretical analysis of the performance of Algorithm \ref{alg:subspace_approx},
which aims at approximating the matrix subspace
$$
\mathcal{W} = \Span{v^{[\ell]}_{i}(x^\ast) \otimes v^{[\ell]}_i(x^\ast) \mid i=1,\ldots,m_\ell, \: \ell=1,\ldots,L, }
$$
by the leading singular subspace of the matrix
\begin{align}
\label{eq:def_auxiliary_matrix_0}
\widehat M = (\widehat M_1|\ldots|\widehat M_{m_L}) \in \bbR^{D^2\times N_H m_L}, \textrm{ where } \widehat M_p = (\opvec(\Delta_{\varepsilon}^2 f_p(X_1))|\ldots|\opvec(\Delta_{\varepsilon}^2 f_p(X_{\numsamples}))).
\end{align}
As before, we let $m = \sum_{\ell=1}^{L}m_{\ell}$ denote the number of neurons in the network
under consideration.
As a learnability condition, we assume that the $m^{th}$ singular value is strictly positive
\begin{equation}
\label{eq:richness_assumption_initial}
\alpha := \sigma_{m}\left(\sum_{p=1}^{m_L}  \int \opvec(\nabla^2 f_{p}(X)) \otimes \opvec(\nabla^2 f_{p}(X)) \mathrm{d}\mu_X \right)>0.
\end{equation}
This assumption asserts that we have chosen a distribution $\mu_X$,
which offers sufficient second order information for recovering the subspace $\CW$. It should be clear from this condition that all networks whose Hessians $\nabla^2 f_p(X)$ vanish almost everywhere, for instance networks with piecewise linear activation functions, will not be identifiable by this method, see, e.g., \cite{rolnick2020reverseengineering} for an alternative approach in this case.
Moreover,
\eqref{eq:richness_assumption_initial} excludes reducible networks that contain subgroups of neurons
whose contributions cancel out, i.e., do not eventually contribute to the network output, and can thus not be recovered using
any number of network output queries, see the related example in \cite[Section 2]{fornasier2019robust}.
The main result about the approximation of $\CW$ follows now.

\begin{thm} \label{thm:approx_bound}
Consider a neural network $f: \bbR^D \rightarrow \bbR^{m_L}$ as in Definition \ref{dfn:feedforward_networks} with at least $3$-times differentiable activation function
$g$. Let $X\sim \mu_X$ be sub-Gaussian with mean $\bbE[X] = x^{\ast}$ and let
$X_1,\ldots,X_{\numsamples}$ be independent copies of $X$. Define the shorthands
\begin{align}
\nonumber
\hat{C}_\epsilon &:= \max_{p \in [m_L]}\sup_{x \in \supp(\mu_X)} \| \Delta_\epsilon^2 f_p(x) - \nabla^2 f_p(x) \|_F,\\
\label{eq:assumption_kappa_der_act}
\kappa &:= \max_{k \in [3]} \textN{g^{(k)}}_{L_{\infty}(\bbR)},\\
\nonumber
\bar{C} &:= C\kappa^L  \left(\prod_{k=1}^{L} \| W_k \|\right)\sum_{\ell=2}^{L}\left(\prod_{k=1}^{\ell} \| W_k \|\right)\sum_{k=1}^{\ell-1}\kappa^{\ell+k-2}\N{W_k}_{2\rightarrow \infty} \left(\prod_{j=1}^{k-1} \| W_k \|\right),
\end{align}
where $\hat{C}_\epsilon$ describes the accuracy of the employed Hessian approximation scheme, $\kappa$ and
$\bar{C}$ describe the network complexity, and $C > 0$ is a universal numeric constant.
If $\alpha$ in \eqref{eq:richness_assumption_initial} satisfies $\alpha >  \max\{2\bar C^2 m_L D \textN{X-x^\ast}_{\psi_2}^2, 2m_L\hat{C}_\epsilon^2\}$,
the subspace $\widehat{\auxspace}$ associated to the $m$ largest singular values of $\widehat M$ satisfies
\begin{equation}
\label{eq:subspace_error_W}
\| P_{\widehat \CW} - P_{\CW} \|_F \leq 2\sqrt{m_L}
\frac{\hat{C}_\epsilon + \bar{C}\sqrt{D}\N{X-x^{\ast}}_{\psi_2}}{\sqrt{\frac{\alpha}{2} - m_L\hat{C}_\epsilon^2}}
\end{equation}
with probability at least
\begin{equation*}
 1 - 2\exp\left(-C\numsamples\right)
   -\bar{m}\exp\left(-C\frac{\numsamples\alpha}{m_L \Maf^{2L} \left( \prod_{\ell=1}^L \| W_\ell \| \right)^2
     \left( \sum_{\ell=1}^L  \Maf^{\ell-1} \prod_{k=1}^\ell \| W_k \| \right)^2 }\right).
\end{equation*}
\end{thm}

Before we prove Theorem \ref{thm:approx_bound}, a couple of comments are in order.
The error \eqref{eq:subspace_error_W} consists of two terms. The first term
accounts for approximating the Hessians by $\Delta_\epsilon^2 f_p(x) \approx \nabla^2 f_p(x)$
and the second term bounds the mismatch of the Hessians $\nabla^2 f_p(x)$ with $\CW$, since $\nabla^2 f_p(x)$ is not
contained exactly in $\CW$, but is just close it.
We do not further characterize $\hat{C}_\epsilon$, as it depends on the mechanism to
construct approximate Hessians $\Delta_\epsilon^2 f_p(x)$ (e.g., finite differences), and a corresponding
error bound can be readily inserted.
Most of the proof instead focuses
on bounding $\N{P_{\CW^\perp}\nabla^2 f_p(X)}_F$ for $X \sim \mu_X$, leading to the
second term in \eqref{eq:subspace_error_W}.

In the numerical experiments reported in the previous section , we used $\mu_X = R\cdot \textrm{Unif}(\bbS^{D-1})$ for
some small radius $R>0$. The corresponding sub-Gaussian norm is $\textN{X}_{\psi_2}\leq R/\sqrt{D}$ and
inserting it into \eqref{eq:subspace_error_W} gives
\begin{equation*}
 \| P_{\widehat \CW} - P_{\CW} \|_F \leq 2\sqrt{m_L}
 \frac{ \hat{C}_\epsilon + \bar{C}R}{\sqrt{\frac{\alpha}{2} - m_L\hat{C}_\epsilon^2}}.
\end{equation*}
The bound suggests that taking $R$ as small as possible is beneficial for the subspace approximation
accuracy, which is in agreement with the numerical experiments. We should note however
that parameter $\alpha$ in \eqref{eq:richness_assumption_initial}
also depends on $R$ in a nontrivial way, and choosing $R$ too small
may lead to a distribution $\mu_X$ that does not carry enough information for a stable recovery.

Lastly, $\bar C$ encodes the complexity of the network function $f$
in terms of the norms of the weight matrices. We note that $\N{W_\ell}_{2\rightarrow \infty} = 1$
if the weights are unit norm, and that $\N{W_\ell}_2 \approx 1$ if $W_\ell$ are  appropriately scaled random matrices. Such kinds of weights are, for instance, used
as a typical initialization of backpropagation for the training of  neural networks.

We prove Theorem \ref{thm:approx_bound} as a result of a more general proposition that compares
the leading singular subspace of $\widehat M$ with a certain auxiliary space $\CW^\ast\subseteq \CW$. Namely,
we introduce the matrix $M^* := (M_1^*|\ldots| M_{m_L}^*) \in \bbR^{D^2\times N_H m_L}$,
where the blocks are given by
\begin{align}
\label{eq:def_auxiliary_matrix_1}
M_p^* = (\opvec(P_{\CW}(\nabla^2 f_p(X_1)))|\ldots |\opvec(P_{\CW}(\nabla^2 f_p(X_{\numsamples}))))\in \bbR^{D^2\times N_H},
\end{align}
and denote the auxiliary space $\CW^\ast = \range(M^\ast) \subseteq \CW$ of dimension $\bar m := \dim(\CW^\ast)$.
By construction we always have $\bar m \leq m$, and under the assumptions in Theorem \ref{thm:approx_bound},
we can show $\bar m = m$ and
thus $\CW^\ast = \CW$. However,
we can still recover the subspace $\CW^\ast \subseteq \CW$
by using the leading $\bar m$ eigenvectors of $\widehat M$ in the case $\bar m < m$ as described
in the following result.

\begin{prop} \label{prop:approx_bound}
Consider a neural network $f: \bbR^D \rightarrow \bbR^{m_L}$ as in Definition \ref{dfn:feedforward_networks} with at least $3$-times differentiable activation function
$g$ and denote $\kappa > 0$ as in \eqref{eq:assumption_kappa_der_act}. Let $X\sim \mu_X$ with mean $\bbE[X] = x^{\ast}$ be sub-Gaussian and let
$X_1,\ldots,X_{\numsamples}$ be independent copies of $X$. Furthermore,
let $\CW^* = \range(M^*) \subseteq \CW$ with $\bar m := \dim(\CW^\ast)$. Assume
\begin{equation}
\label{eq:richness_assumption_auxiliary}
\alpha^\ast := \sigma_{\bar m}\left(\sum_{p=1}^{m_L}  \int \opvec(\nabla^2 f_{p}(X)) \otimes \opvec(\nabla^2 f_{p}(X)) \mathrm{d}\mu_X \right) > 2m_L\hat{C}_\epsilon^2.
\end{equation}
The subspace $\widehat{\auxspace}$ associated to the $\bar m$ largest singular values  of $\widehat M$ satisfies
 \begin{equation*}
  \label{eq:subspace_error}
  \| P_{\widehat \CW} - P_{\auxspace^\ast} \|_F \leq 2\sqrt{m_L}
  \frac{ \hat{C}_\epsilon + \bar{C}\sqrt{D}\N{X-x^{\ast}}_{\psi_2}}{\sqrt{\frac{\alpha^\ast}{2} - m_L\hat{C}_\epsilon^2}}
\end{equation*}
with probability at least
\begin{equation*}
 1 - 2\exp\left(-C\numsamples\right)
   -\bar{m}\exp\left(-C\frac{\numsamples\alpha^\ast}{m_L \Maf^{2L} \left( \prod_{\ell=1}^L \| W_\ell \| \right)^2
     \left( \sum_{\ell=1}^L  \Maf^{\ell-1} \prod_{k=1}^\ell \| W_k \| \right)^2 }\right).
\end{equation*}
The constants $C, \hat{C}_\epsilon, \bar{C}$ are as in Theorem \ref{thm:approx_bound}.
\end{prop}

\subsection{Proof of Proposition \ref{prop:approx_bound}}
Let us give a quick proof sketch before going into the details.
To show that the leading singular space of $\widehat M$ approximates the subspace $\CW^\ast$, the main step is to compare $\widehat M$
with the auxiliary matrix $M^*$.
The comparison of $M^*$ and $\widehat M$ goes by first introducing a second
auxiliary matrix, denoted by
\begin{align}
\label{eq:def_M_matrix}
  M := (M_1|\ldots|M_{m_L}) \in \bbR^{D^2 \times N_H m_L},\ \textrm{where } M_p := (\opvec(\nabla^2 f_p(X_1))|\ldots|\opvec(\nabla^2 f_p(X_{\numsamples}))),
\end{align}
and then decomposing the error into
\begin{align}
\label{eq:aux_decomposition_initial}
\N{M^* - \widehat M}_F \leq \N{M^* - M}_F  + \N{M-\widehat M}_F.
\end{align}
The second term can be easily bounded by the accuracy of our Hessian approximations measured by
the discrepancy $\hat{C}_\epsilon$ in Theorem \ref{thm:approx_bound} or Proposition \ref{prop:approx_bound}.
The first term in \eqref{eq:aux_decomposition_initial} is more challenging
and uses the specific form of the Hessians of $f$, derived in Proposition \ref{prop:nn_derivatives}, and given by
$$
\nabla^2 f(x) = \sum_{\ell=1}^L V_\ell(x) S_p^{\lind{\ell}}(x) V_\ell(x)^\top.
$$
Namely, we will first establish the Lipschitz-continuity of the entangled weight
matrix $V_{\ell}(x)$ as a function of $x$ (see Lemma \ref{lem:aux_bounds}) and then
show proximity between $M$ and $M^*$ using standard concentration arguments.
The combined bound for $\textN{M^* - \widehat M}_F$ is presented in Lemma \ref{lem:bound_diff}.
In the final step of the proof, we use a Wedin-bound argument to compare
the leading singular space of $M^*$ and $\widehat M$, which is bounded by $\textN{M^* - \widehat M}_F$
divided by  the $\bar{m}$-th largest singular value of $\widehat M$. The required
lower bound on $\sigma_{\bar m}(\widehat M)$ is derived by first rewriting $\alpha^*$ in
assumption \eqref{eq:richness_assumption_auxiliary} as $\alpha^* = \sigma_{\bar m}(\bbE[MM^\top])$,
then carrying this bound with high probability to $\sigma_{\bar m}(MM^\top)$,
and finally to $\sigma_{\bar m}(\widehat M \widehat M^\top)$.

\paragraph{Proof details}
We begin by establishing the Lipschitz continuity of the entangled weights $V_{\ell}(x)$
as a function of $x$, as well as some uniform matrix bounds that we require in the
following.
\begin{lem} \label{lem:aux_bounds}
Consider a neural network as in Definition \ref{dfn:feedforward_networks} with at least $3$-times differentiable activation function
$g$ and assume there exists $\kappa > 0$ so that \eqref{eq:assumption_kappa_der_act} holds.
For any $x,x' \in \textrm{supp}(\mu_X)$ we have
\begin{align}
\label{eq:subspace_approx:bound_V}
     \| V_\ell(x) \| &\leq \Maf^{\ell-1} \prod_{k=1}^\ell \| W_k \|,\\
\label{eq:subspace_approx:bound_Vdiff}
     \| V_\ell(x) - V_\ell(x^\prime) \|
     &\leq  \left(\prod_{k=1}^{\ell} \| W_k \| \right)\sum_{k=1}^{\ell-1}\kappa^{\ell+k-2}\N{W_k}_{2\rightarrow \infty}\left(\prod_{j=1}^{k-1} \| W_j \| \right) \N{x-x'},\\
\label{eq:subspace_approx:bound_S}
      \| S_p^{\lind{\ell}}(x) \|_F &\leq \Maf^{L-\ell + 1}  \prod_{k=\ell+1}^L \| W_k \|.
\end{align}
\end{lem}
\begin{proof}
Let us first note that for any $\ell \in [L]$ and arbitrary $x$ we have
by definition of $G_{\ell}$
\begin{align}
\label{eqn:bound_spectral_G}
\N{G_{\ell}(x)} = \max_{i \in [m_\ell]}\SN{g_{\ell}'\left(\langle w^{\lind{\ell}}_i, y^{\lind{\ell-1}}(x)\rangle\right)}
=\max_{i \in [m_\ell]}\SN{g'\left(\langle w^{\lind{\ell}}_i, y^{\lind{\ell-1}}(x)\rangle) - \tau^{\lind{\ell}}_i\right)}
\leq \kappa,
\end{align}
where $\tau^{\lind{\ell}}_i$ denotes the $i$-th entry of the $\ell$-th shift vector $\tau_{\ell}$.
Then, for the statement \eqref{eq:subspace_approx:bound_V} note that $\| V_1(x) \| = \| W_1 \|$
by definition, and for $\ell \geq 2$, by matrix norm submultiplicativity, we have
\begin{align*}
 \| V_\ell(x) \| & = \left\| W_1 \prod_{k=2}^{\ell} G_{k-1}(x) W_k \right\| \leq \| W_1 \| \left( \prod_{k=1}^{\ell-1} \| G_k(x) \| \right) \left( \prod_{k=2}^\ell \| W_k \| \right) \leq \Maf^{\ell-1} \prod_{k=1}^\ell \| W_k \|.
\end{align*}
For \eqref{eq:subspace_approx:bound_Vdiff} we first notice that the Lipschitz continuity of the activation function in each layer, given by \eqref{eq:assumption_kappa_der_act},  implies
\begin{align} \label{eq:subspace_approx:aux0}
  \| y^{\lind{\ell}}(x) - y^{\lind{\ell}}(x^\prime) \|
 &\leq \| g(W_{\ell}^\top y^{\lind{\ell-1}}(x) + \tau_{\ell}) - g(W_{\ell}^\top y^{\lind{\ell-1}}(x^\prime) + \tau_{\ell}) \|\\
 &\leq \kappa \N{W_{\ell}}\| y^{\lind{\ell-1}}(x) - y^{\lind{\ell-1}}(x^\prime) \| \leq \ldots \leq \kappa^{\ell}\left(\prod_{k=1}^\ell \| W_k \| \right)\N{x-x^\prime},
\end{align}
where we repeated the argument in the first step $\ell$-times and then used $y^{\lind{0}}(x) = x$, $y^{\lind{0}}(x^\prime) = x^\prime$.
This means the $\ell$-th output layer is still Lipschitz continuous, with Lipschitz constant given by spectral norms
of weight matrices and $\kappa^{\ell}$.
Since $g'$ is also Lipschitz continuous with Lipschitz constant $\kappa$ as in \eqref{eq:assumption_kappa_der_act},
a similar argument applies to the matrix function $x\mapsto G_\ell(x)$, i.e., we have
\begin{equation}
\label{eqn:aux_G_bound_lipschitz}
\begin{aligned}
 &\| G_\ell(x) - G_\ell(x^\prime) \|
  = \| g_\ell^\prime(W_\ell^\top y^{\lind{\ell-1}}(x)) - g_\ell^\prime(W_\ell^\top y^{\lind{\ell-1}}(x^\prime)) \|_\infty \\
 &\qquad\qquad \leq \kappa \N{W_\ell}_{2\rightarrow \infty}\| y^{\lind{\ell-1}}(x) - y^{\lind{\ell-1}}(x^\prime) \| \leq \kappa^{\ell} \N{W_\ell}_{2\rightarrow \infty}\left(\prod_{k=1}^{\ell -1} \| W_k \| \right)\N{x-x^\prime}.
\end{aligned}
\end{equation}
Now let us address the Lipschitz continuity of entangled weights $V_{\ell}(x)$ as a function of $x$. For $\ell = 0$
we simply have $ \| V_1(x) - V_1(x^\prime) \| = \| W_1 - W_1 \| = 0$, so nothing needs to be done. For
$\ell \geq 2$, we first use the triangle inequality, norm submultiplicativity, and the relation
$V_{\ell}(x) = V_{\ell-1}(x)G_{\ell-1}(x)W_\ell$ to get
\begin{align*}
 \| V_\ell(x) - V_\ell(x^\prime) \|
 & = \left\| V_{\ell-1}(x)G_{\ell-1}(x)W_\ell - V_{\ell-1}(x^\prime )G_{\ell-1}(x^\prime)W_\ell \right\| \\
 &\leq \N{W_{\ell}}\left(\left\| V_{\ell-1}(x)- V_{\ell-1}(x^\prime)\right\|\N{G_{\ell-1}(x)} + \N{V_{\ell-1}(x^\prime )}\N{G_{\ell-1}(x^\prime) - G_{\ell-1}(x) }\right)\\
&\leq \kappa  \| W_\ell \| \left( \| V_{\ell-1}(x) - V_{\ell-1}(x^\prime) \|  +
   \kappa^{\ell-3} \left(\prod_{k=1}^{\ell - 1} \| W_k \| \right) \| G_{\ell-1}(x) - G_{\ell-1}(x^\prime) \|\right),
\end{align*}
where we used \eqref{eq:subspace_approx:bound_V} and \eqref{eqn:bound_spectral_G} to bound $\N{G_{\ell-1}(x)}$ and $\N{V_{\ell-1}(x')}$
in the last step.
Repeating the computation for $\| V_{\ell-1}(x) - V_{\ell-1}(x^\prime) \|$ until we reach $\ell = 1$
where $\| V_1(x) - V_1(x^\prime) \| = 0$, we obtain
\begin{align*}
 \| V_\ell(x) - V_\ell(x^\prime) \| \leq \kappa^{\ell-2} \left(\prod_{k=1}^{\ell} \| W_k \| \right) \ \sum_{k=1}^{\ell-1} \| G_{k}(x) - G_{k}(x^\prime) \|,
\end{align*}
and the result \eqref{eq:subspace_approx:bound_Vdiff} follows by using Lipschitz property of $G_{k}$ as shown in \eqref{eqn:aux_G_bound_lipschitz}.
Finally, for the third result \eqref{eq:subspace_approx:bound_S} we first note
that \eqref{eq:assumption_kappa_der_act} implies
$\| \opdiag(g_\ell^\dprime(W_\ell^\top y^{\lind{\ell-1}}(x) )\| \leq \kappa$. Then, inserting
the definition of $ S^{\lind{\ell}}_p(x)$ and using norm submultiplicativity combined with the bound on $\N{G_\ell(x)}$ in \eqref{eqn:bound_spectral_G}
we get
\begin{align*}
 \N{S^{\lind{\ell}}_p(x)}_F & \leq \N{ \opdiag(g_\ell^\dprime(W_\ell^\top y^{\lind{\ell-1}}(x) )}
    \left\| \opdiag\left( \left(\prod_{k=\ell+1}^L W_k G_k(x)\right)e_p \right)\right\|_F \\
    & \leq \Maf \left\| \left( \prod_{k=\ell+1}^L W_k G_k(x) \right) e_p \right\| _2  \leq \Maf \left( \prod_{k=\ell+1}^L \| W_k\| \| G_k(x) \| \right)
     \leq \Maf^{L- \ell + 1} \prod_{k=\ell+1}^L \| W_k \|.
\end{align*}
\end{proof}

\noindent
Having these preliminary bounds, next we can  bound the error $\textN{\widehat M - M^{\ast}}_F$.
\begin{lem} \label{lem:bound_diff}
Consider a neural network $f : \bbR^D\rightarrow \bbR^{m_L}$ as in Definition \ref{dfn:feedforward_networks} with at least $3$-times differentiable activation function
$g$ and assume there exists $\kappa > 0$ satisfying the derivative bound \eqref{eq:assumption_kappa_der_act}. Let $X \sim \mu_X$
be a sub-Gaussian distribution with mean $\bbE[X] = x^{\ast}$ and let $X_1,\ldots,X_{\numsamples}$ be independent copies of $X$. There exists a uniform
constant $C$ so that with probability at least $1-\exp(-C\numsamples)$, the matrices $\widehat M$, $M^{\ast}$,
and $M$ as defined in \eqref{eq:def_auxiliary_matrix_0}, \eqref{eq:def_auxiliary_matrix_1}, \eqref{eq:def_M_matrix} satisfy
\begin{align*}
\|M - M^\ast \|_F &\leq \sqrt{\numsamples m_L}\bar{C}\sqrt{D}\N{X-x^{\ast}}_{\psi_2},\\
 \|\widehat{M} - M^\ast \|_F &\leq \sqrt{\numsamples m_L}\left(\hat{C}_\epsilon + \bar{C}\sqrt{D}\N{X-x^{\ast}}_{\psi_2}\right),
\end{align*}
where the constants $\hat{C}_\epsilon$ and $\bar{C}$ are as in Theorem \ref{thm:approx_bound}.%
\end{lem}
\begin{proof}
As described in the proof sketch, we first use the triangle inequality to get  $\| \widehat{M} - M^\ast \|_F \leq \| \hat{M} - M \|_F + \| M - M^\ast \|_F$
and from the definition of the Frobenius norm, we immediately get
$\| \widehat{M} - M \|_F \leq \sqrt{m_L\numsamples} \hat{C}_\epsilon$. Hence, we can
focus on $ \| M - M^\ast \|_F$ in the remainder of the proof.
Using the definition of the Frobenius norm and the Hessian formulation in Proposition
\ref{prop:nn_derivatives} we first obtain
\begin{align*}
 \| M - M^\ast \|_F^2 & = \sum_{p=1}^{m_L} \sum_{i=1}^{\numsamples} \| \nabla^2 f_p(X_i) - P_{\auxspace}\nabla^2 f_p(X_i) \|_F^2 \\
 & \leq \sum_{p=1}^{m_L} \sum_{i=1}^{\numsamples} \left( \sum_{\ell=2}^L \left\| V_\ell(X_i)  S_p^{\lind{\ell}}(X_i) V_\ell(X_i)^\top
      - P_{\auxspace}\left(  V_\ell(X_i) S_p^{\lind{\ell}}(X_i) V_\ell(X_i)^\top \right) \right\|_F \right)^2,
\end{align*}
where we can forget the term $\ell = 1$ in the sum because of $P_{\CW^\perp }(V_1(X_i)) = P_{\CW^\perp}(W_1) = 0$. Then,
by the minimizing property of the orthogonal projection, we can replace $P_{\auxspace}(V_\ell(X_i) S_p^{\lind{\ell}}(X_i) V_\ell(X_i)^\top)$
by an arbitary matrix in $\auxspace$ as the difference will  only increase. Specifically, we can choose the matrix
$V_\ell(x^\ast) S_p^{\lind{\ell}}(X_i) V_\ell(x^\ast)^\top\in \CW$,
and then decompose the error to get for all $i \in [\numsamples]$
\begin{align*}
& \left\| V_\ell(X_i)  S_p^{\lind{\ell}}(X_i) V_\ell(X_i)^\top
     - P_{\auxspace}\left(  V_\ell(X_i) S_p^{\lind{\ell}}(X_i) V_\ell(X_i)^\top \right) \right\|_F\\
&\quad\leq \left\| V_\ell(X_i)  S_p^{\lind{\ell}}(X_i) V_\ell(X_i)^\top
          - V_\ell(x^\ast) S_p^{\lind{\ell}}(X_i) V_\ell(x^\ast)^\top  \right\|_F\\
&\quad \leq \left\| V_\ell(X_i) - V_\ell(x^\ast) \right\| \|   S_p^{\lind{\ell}}(X_i)  \|_F
 (\|V_\ell(X_i) \| + \| V_\ell(x^\ast)  \|)\\
&\quad \leq 2\kappa^L  \left(\prod_{k=1}^{L} \| W_k \|\right)\left(\prod_{k=1}^{\ell} \| W_k \|\right)\sum_{k=1}^{\ell-1}\kappa^{\ell+k-2}\N{W_k}_{2\rightarrow \infty} \left(\prod_{j=1}^{k-1} \| W_k \|\right) \N{X_i-x^{\ast}}.
\end{align*}
Here, the last step follows from the bounds derived in Lemma \ref{lem:bound_diff}.
Summing the leading factor from $\ell = 2$ to $L$ gives precisely $\bar C$ (up to a universal constant) so that we obtain in total
\begin{align*}
\| M - M^\ast \|_F^2 \leq \bar C^2 m_L \sum_{i=1}^{\numsamples}\N{X_i - x^{\ast}}^2.
\end{align*}
Note now that the random variable $Z_i := X_i - x^{\ast}$ is sub-Gaussian (since $X\sim \mu_X$ is sub-Gaussian).
It is then straight-forward to deduce that $\N{Z_i}$ is sub-Gaussian with sub-Gaussian
norm $\sqrt{D}\N{Z_i}_{\psi_2}$, and that $\N{Z_i}^2$ is sub-exponential
with sub-exponential norm given by $D\N{Z_i}_{\psi_2}^2$. Using elementary properties of the
subexponential norm, see, e.g., \cite[Proposition 2.7.1]{HDP18}, it follows that
\begin{align*}
\bbE \sum_{i=1}^{\numsamples}\N{Z_i}^2 = \sum_{i=1}^{\numsamples}\bbE \N{Z_i}^2 \leq C\numsamples D\N{Z_i}_{\psi_2}^2
\end{align*}
for some universal constant $C$. Furthermore, using Bernstein's inequality \cite[Theorem 2.8.1]{HDP18}
with $t=C\numsamples \textN{\N{Z_i}^2}_{\psi_1} = C\numsamples D \N{Z_i}_{\psi_2}^2$, we get the
concentration result
\begin{align*}
\sum_{i=1}^{\numsamples}\N{Z_i}^2 \leq C\numsamples D\N{Z_i}_{\psi_2}^2 + \SN{\sum_{i=1}^{\numsamples}\N{Z_i}^2 - \bbE \N{Z_i}^2}
\leq 2C \numsamples D\N{Z_i}_{\psi_2}^2
\end{align*}
with probability at least $1-2\exp\left(-C'\numsamples\right)$ for some other universal constant $C' > 0$.
Combining this with the bound on $\| M - M^\ast \|_F^2$ yields on the same probability event
\begin{align*}
\| M - M^\ast\|_F^2 \leq \bar C^2 m_L \numsamples D\N{Z_i}_{\psi_2}^2.
\end{align*}
\end{proof}
\noindent
Before completing the proof of Theorem \ref{thm:approx_bound}, we derive the lower bound
for $\sigma_{\bar m}(\widehat M)$.
\begin{lem} \label{lem:bound_min_sval}
Assume the setting of Lemma \ref{lem:bound_diff} and the learnability condition \eqref{eq:richness_assumption_auxiliary}.
 Then there exists a universal constant $C > 0$ so that
 \begin{equation*}
\sigma_{\bar{m}}(M)\geq \sqrt{\frac{\alpha^\ast N_H}{2}}\quad \textrm{ and }\quad     \sigma_{\bar{m}}(\widehat{M}) \geq \sqrt{\numsamples}\left( \sqrt{\frac{\alpha^\ast}{2} - m_L\hat{C}_\epsilon^2}\right)
 \end{equation*}
 with probability at least
 \begin{equation*}
  1 - \bar{m}\exp\left(-C\frac{\numsamples\alpha^\ast}{m_L \Maf^{2L} \left( \prod_{\ell=1}^L \| W_\ell \| \right)^2
      \left( \sum_{\ell=1}^L  \Maf^{\ell-1} \prod_{k=1}^\ell \| W_k \| \right)^2 }\right)
 \end{equation*}
\end{lem}
\begin{proof}
First note $\sigma_{\bar{m}}(\widehat M) = \sqrt{\sigma_{\bar{m}}(\widehat M\widehat M^\top)}$ and that  Weyl's inequality immediately implies
\begin{align*}
\sigma_{\bar{m}}(\widehat{M}\widehat{M}^\top) \geq \sigma_{\bar{m}}(MM^\top) - \| MM^\top - \widehat{M}\widehat{M}^\top \| \geq \sigma_{\bar{m}}(M M^\top) - m_L\numsamples\hat{C}_\epsilon^2,
\end{align*}
where we used $\| MM^\top - \widehat{M}\widehat{M}^\top \|\leq \| M - \widehat{M} \|_F^2$. Recalling the definition of $M$, we further note
 \begin{align*}
\sigma_{\bar{m}}(MM^\top ) = \sigma_{\bar{m}}\left(\sum_{i=1}^{\numsamples} \sum_{p=1}^{m_L} \opvec(\nabla^2 f_p(X_i)) \otimes \opvec(\nabla^2 f_p(X_i))\right),
 \end{align*}
 which means $\sigma_{\bar{m}}(MM^\top)$ is the $\bar{m}$-th eigenvalue of a sum of $\numsamples$ independent and identically distributed
 random matrices. Taking into account $ \sigma_{\bar{m}}(\bbE MM^\top) = \numsamples \alpha^\ast$
 by assumption \eqref{eq:richness_assumption_auxiliary},
 we can use a matrix Chernoff bound \cite[Theorem~4.1]{tropp_gittens} to deduce that for all $t \in [0,1]$ we have
 \begin{equation*}
\bbP\left(\sigma_{\bar{m}}(MM^\top)
   \geq t\numsamples \alpha^\ast\right) \geq 1 - \bar{m}\exp\left(\frac{-(1-t)^2 \numsamples \alpha^\ast}
   {2\max_{x \in \supp(\mu_X)} \| \sum_{p=1}^{m_L} \opvec(\nabla^2 f_{p}(x)) \otimes \opvec(\nabla^2 f_{p}(x)) \| }
   \right).
 \end{equation*}
 In order to make the last expression more explicit, we can use the Hessian formulation in Proposition  \ref{prop:nn_derivatives},
the general inequality $\N{AB}_F\leq \N{A}_F\N{B}$, and the bounds in Lemma  \ref{lem:aux_bounds} to get
 \begin{align*}
  \| \nabla^2 f_{p}(x) \|_F & \leq \sum_{\ell=1}^L \| V_\ell(x)  S^{\lind{\ell}}_{p}(x) V_\ell(x)^\top \|_F  \leq \sum_{\ell=1}^L \|V_\ell(x) \|^2 \| S^{\lind{\ell}}_{p}(x) \|_F \\
    & \leq \sum_{\ell=1}^{L}\left(\Maf^{\ell-1} \prod_{k=1}^\ell \| W_k \|\right)^2 \Maf^{L-\ell + 1}  \prod_{k=\ell+1}^L \| W_k \|
    \leq \kappa^L \prod_{k=1}^L \| W_k \|\sum_{\ell=1}^{L}\Maf^{\ell-1} \prod_{k=1}^\ell \| W_k \|,
 \end{align*}
 universally for all $p \in [m_L]$ and $x \in \supp(\mu_X)$.
 Using $\max_{x \in \supp(\mu_X)} \| \sum_{p=1}^{m_L} \opvec(\nabla^2 f_{p}(x)) \otimes \opvec(\nabla^2 f_{p}(x)) \|\leq m_L \| \nabla^2 f_{p}(x) \|_F^2 $
 and $t = \frac{1}{2}$ yields the result.
\end{proof}

\noindent
We can now give the proof of Proposition \ref{prop:approx_bound}.
\begin{proof}[Proof of Proposition \ref{prop:approx_bound}]
 Let $\widehat{U}\widehat{\Sigma}\widehat{V}^\top $ and $U^\ast\Sigma^\ast{V^\ast}^\top $ be the singular value decompositions
 of $\widehat{M}$ and $M^\ast$, respectively, and let $\widehat{U}_1$, $U^\ast_1$ be the matrices that contain the first
 $\bar{m}$ columns of $\widehat{U}$, $U^\ast$, respectively. From Wedin's bound \cite{stewart1991perturbation,wedin1972perturbation}
 we then have
 \begin{equation*}
  \| P_{\widehat{\auxspace}} - P_{\auxspace^\ast} \|_F \leq \frac{2\|\widehat{M}-M^\ast\|_F}{\beta}
 \end{equation*}
 for any $\beta>0$ that satisfies
 \begin{align*}
  \beta & \leq \min_{\substack{1 \leq i \leq \bar{m}\\ \bar{m} < j}} |\sigma_i(\widehat{M}) - \sigma_j(M^\ast)|\qquad \textrm{ and }\qquad \beta  \leq \min_{1 \leq i \leq \bar{m}} \sigma_i(\widehat{M}).
 \end{align*}
 Since $\dim(\CW^\ast) = \dim(\range(M^\ast)) = \bar{m}$ we have $\sigma_j(M^\ast) = 0$ for all $j > \bar m$ and thus the two constraints for $\beta$ are actually equivalent.
 We may use $\beta=\sigma_{\bar{m}}(\widehat{M})$.  Then applying the union bound and using Lemma \ref{lem:bound_diff}  and Lemma \ref{lem:bound_min_sval}
leads to the desired result.
\end{proof}

\subsection{Proof of Theorem \ref{thm:approx_bound}}
\label{subsec:proof_main_theorem_subspace}
We show that Theorem \ref{thm:approx_bound} is implied by Proposition \ref{prop:approx_bound}
under the assumptions made in Theorem \ref{thm:approx_bound} by proving
$\bar m = \dim(\range(M^\ast)) = m$, or, $\CW = \CW^\ast$.

First, note that by Lemma \ref{lem:bound_diff}
we still have $\|M - M^\ast \|_F \leq \bar{C}\sqrt{\numsamples m_LD}\N{X-x^{\ast}}_{\psi_2}$
with the probability as described in the lemma. Furthermore, following the proof of Lemma \ref{lem:bound_min_sval},
we get $\sigma_{m}(M)\geq \sqrt{\alpha N_H/2}$ with the probability described in
Lemma \ref{lem:bound_min_sval} after replacing $\alpha^\ast$ by $\alpha = \sigma_m(\bbE MM^\top)$.
Then, using Weyl's eigenvalue bound we have
\begin{align*}
\sigma_m(M^\ast {M^\ast}^\top) &\geq \sigma_m(MM^\top) -\N{M^\ast {M^\ast}^\top - M M^\top}_2 \geq N_H\frac{\alpha}{2} - \N{M^\ast - M}_F^2\\
&\geq N_H\left(\frac{\alpha}{2}-\bar C^2 m_L D\N{X-x^\ast}_{\psi_2}^2\right).
\end{align*}
Using the lower bound $\alpha > 2\bar C^2 m_L D\N{X-x^\ast}_{\psi_2}^2$, the right hand side is strictly
positive and thus $\sigma_m(M^\ast) > 0$, respectively, $\bar m =  m$.
Lastly, we note that the events giving $\|M - M^\ast \|_F \leq \bar{C}\sqrt{\numsamples m_LD}\N{X-x^{\ast}}_{\psi_2}$
and $\sigma_{m}(M)\geq \sqrt{\alpha N_H/2}$ are the same events that are required for Proposition \ref{prop:approx_bound},
implying that Theorem \ref{thm:approx_bound} holds with the same probability as Proposition \ref{prop:approx_bound}.
\qed

\section{Guarantees about entangled weight recovery}
\label{sec:rank_one_recov}
Recall that the outer products of entangled network weights $v_{i}^{\lind{\ell}}(x^\ast)$
are the spanning elements of the matrix subspace $\CW \subseteq \textrm{Sym}(\bbR^{D\times D})$ as defined in \eqref{eq:approx_subspace_W}.
In this section we show how a subspace approximation $\widehat \CW \approx \CW$
can be used to approximately recover the spanning elements of $\CW$, or the entangled network weights at $x^\ast$,
by searching for local maximizers of
\begin{equation}
\label{eq:opt_functional_sec_5}
\max_{\N{u} = 1} \Phi_{\widehat \CW}(u) := \N{P_{\widehat \CW}(u\otimes u)}_F^2.
\end{equation}

\begin{rem}
\label{rem:subspace_notation}
While we can think of $\CW$ and $\widehat \CW$ as being the subspaces introduced in Section \ref{sec:algorithm},
the following results actually hold for any $\CW \subseteq \textrm{Sym}(\bbR^{D\times D})$, which is spanned by $K$
rank-one matrices $\{w_i\otimes w_i : i \in [K]\}$, and any small perturbation $\widehat \CW \approx \CW$
with $\widehat \CW \subseteq \textrm{Sym}(\bbR^{D\times D})$. In other words, specific properties of the entangled weights $v_{i}^{\lind{\ell}}(x^\ast)$
or the subspace approximation are not used in what follows. Furthermore, some results hold for an arbitrary subspace of $\textrm{Sym}(\bbR^{D\times D})$,
in which case we use the notation $\widetilde \CW\subseteq \textrm{Sym}(\bbR^{D\times D})$.
\end{rem}

We begin by noting that a generalized version of \eqref{eq:opt_functional_sec_5} has been recently analyzed
in the context of symmetric tensor decompositions \cite{kileel2019subspace}. The author's main result,
applied to our case, shows that iteration \eqref{eq:weight_recovery_iteration} converges almost surely to a local
maximizer and that global maximizers correspond to spanning rank-one elements under fairly mild conditions.

\begin{thm}[{\cite[Theorem 5.1]{kileel2019subspace}}]
\label{thm:citing_joe_et_al}
Let $\widetilde \CW \subseteq \textrm{Sym}(\bbR^{D\times D})$. Let
$\gamma > 0$ be a fixed parameter such that $\Phi_{\widetilde\CW}(u) + \frac{\gamma}{4}\N{u}_2^2$
is strictly convex on $\bbR^{D}$ (e.g. $\gamma \geq \frac{1}{2}$ works)
and consider the iteration
\begin{align}
\label{eq:iteration_in_theorem}
u_j = P_{\bbS^{D-1}}(u_{j-1} + 2\gamma P_{\widetilde\CW}(u_{j-1}\otimes u_{j-1})u_{j-1}),\qquad j \in \bbN_{>0}.
\end{align}
\begin{enumerate}[label = \textbf{R\arabic*}]
  \item \label{enum:R1} For any $u_0 \in \bbS^{D-1}$ the iteration \eqref{eq:iteration_in_theorem} is well defined and converges monotonically
  to a constrained stationary point of $\Phi_{\widetilde\CW}$ at a power rate.
  \item \label{enum:R2} For a full Lebesgue measure subset of initializations $u_0 \in \bbS^{D-1}$, \eqref{eq:iteration_in_theorem}
  converges to a constrained local maximizer of $\Phi_{\widetilde\CW}$.
  \item \label{enum:R3} If  $\widetilde\CW= \CW = \Span{w_i \otimes w_i : i \in [K]}$ for $K < \frac{1}{2}(D-1)D$ and $w_1,\ldots,w_K$ are sampled from any absolutely continuous probability
  distribution on $\bbS^{D-1}$, constrained global maximizers of $\Phi_{\widetilde\CW}$ are precisely $\pm w_i$. Moreover,
  each $\pm w_i$ is exponentially attractive, which means that initializations $u_0 \in \bbS^{D-1}$ sufficiently
  close to $\pm w_i$ converge to $\pm w_i$ with an exponential rate.
\end{enumerate}
\end{thm}
\begin{proof}
Note that $\Phi_{\widetilde\CW}(u) = \sum_{i=1}^{K} (u^\top W_i u)^2$ for a basis $\{ W_i : i \in [\dim(\widetilde\CW)]\}$ of
$\widetilde\CW$, which implies that $\Phi_{\widetilde\CW}(u)$ is a homogeneous polynomial of degree 4.
Hence, \cite[Section 5]{kileel2019subspace} applies.
\end{proof}

\begin{rem}[Constrained extreme points]
\label{rem:constrained_extreme_points}
In Theorem \ref{thm:citing_joe_et_al} and the remainder of this section, a constrained stationary point
or local maximizer is meant in the sense of Riemannian optimization with respect to $\bbS^{D-1}$. That is, for any sufficiently smooth
function $f : \bbR^D \rightarrow \bbR$, $u \in \bbS^{D-1}$ is a constrained stationary
point over $\bbS^{D-1}$ if and only if
\begin{align*}
\nabla_{\bbS^{D-1}}f(u):=(\Id_D - u u^\top)\nabla f(u) = 0,
\end{align*}
where $\nabla f$ is the standard gradient; and $u$ is a constrained local maximizer if and only if
\begin{align*}
\nabla^2_{\bbS^{D-1}}f(u):=(\Id_D - u u^\top)\left(\nabla^2 f(u) - \langle u, \nabla f(u)\rangle \Id_D\right)(\Id_D - u u^\top) \preccurlyeq 0,
\end{align*}
where $\preccurlyeq$ indicates the order induced by the cone of positive semidefinite matrices
and $\nabla^2f$ is the standard Hessian matrix.
$\nabla_{\bbS^{D-1}}f$ and $\nabla^2_{\bbS^{D-1}}f$ are also called the Riemannian gradient
and Hessian. We refer to \cite{absil2009optimization} for further details
on Riemannian optimization.
\end{rem}

\noindent
Based on Theorem \ref{thm:citing_joe_et_al} we can assume that local maximizers of
\eqref{eq:opt_functional_sec_5} can be generated via iteration \eqref{eq:iteration_in_theorem}
initialized by $u_0 \sim \Uni{\bbS^{D-1}}$.
The main objective of this section is therefore to show that
local maximizers of the perturbed objective $\Phi_{\widehat \CW}$ are, under suitable conditions
on  $\CW$ and $\widehat \CW$, close to global maximizers of the unperturbed objective $\Phi_{\CW}$, i.e.,
close to the spanning elements $\{w_i : i \in [K]\}$ and thus to the entangled weights at $x^\ast$.

Most of our results are carried out under the assumptions
that length-normalized spanning vectors $\{w_i : i \in [K]\} \subset \bbS^{D-1}$, respectively length-normalized entangled network  weights,
satisfy

\begin{equation}
\label{eq:frame_condition_1}
(1-\nu)\N{x}_2^2 \leq \sum_{i=1}^{K}\langle w_i, x\rangle^2 \leq (1+\nu)\N{x}_2^2\quad  \textrm{ for all } x  \in \R^D, \textrm{ for some } \nu < 1.
\end{equation}
An immediate consequence of \eqref{eq:frame_condition_1} is that $\{w_i \otimes w_i: i \in [K]\}$ is a linearly independent system
\cite[Lemma 23]{fornasier2019robust} and thus the matrices $w_i \otimes w_i$ form a basis
of $\Span{w_i \otimes w_i : i \in [K]}$. Furthermore, the perturbed subspace $\widehat \CW$ should be close to $\CW$,
which is quantified by the assumption $\textN{P_{\CW} - P_{\widehat \CW}}_{2\rightarrow 2}\leq \delta \ll 1$ or
$$
\N{\left(P_{\CW} - P_{\widehat \CW}\right)Z} \leq \delta \N{Z}\qquad \textrm{ for all } Z \in \bbR^{D \times D} \textrm{ and } \delta \ll 1.
$$
In this case $\{P_{\widehat \CW}(w_i \otimes w_i) : i \in [K]\}$ is a basis of $\widehat \CW$
because $P_{\widehat \CW}$ is one-to-one from $\CW$ to $\widehat \CW$.

The remainder of this section is split into three parts. The next section derives
preliminary technical results such as the gradient and Hessian of
$\Phi_{\widetilde \CW}$ for arbitrary $\widetilde \CW \subseteq \textrm{Sym}(\bbR^{D\times D})$, and optimality conditions of the program \eqref{eq:opt_functional_sec_5}.
Afterwards, we generalize
\ref{enum:R3} in Theorem \ref{thm:citing_joe_et_al} under the
frame condition \eqref{eq:frame_condition_1} by showing that any constrained local maxima $u \in \bbS^{D-1}$ resulting in
$\Phi_{\widehat \CW}(u) \approx 1$ (note that $1$ is the global maximum)
is close to one of the entangled weight vectors. In the third part we further
complement Theorem \ref{thm:citing_joe_et_al}
by showing that any constrained local maximizer of $\Phi_{\widehat \CW}$ in a level set above
$C_{\nu}\delta$, where $C_{\nu}$ is a constant depending only on $\nu$, has objective value close to $1$.
As a consequence, apart from spurious local maximizers
in the level set $\{u \in \bbS^{D-1} : \Phi_{\widehat \CW}(u) \leq C_{\nu}\delta\}$, any constrained local maximizer
is close to one of the spanning vectors of $\CW$, and thus close to an entangled weight $v_{i}^{\lind{\ell}}(x^\ast)$, see Theorem \ref{thm:summary_theorem} below.

\begin{thm}
\label{thm:summary_theorem}
Let $\{w_i : i\in [K]\}\subset \bbS^{D-1}$ satisfy \eqref{eq:frame_condition_1} and denote $\CW := \{w_i \otimes w_i : i \in [K]\}$.
Let $\widehat \CW \subseteq \textrm{Sym}(\bbR^{D\times D})$ satisfy $\N{P_{\widehat \CW} - P_{\CW}}_{2\rightarrow 2} < \delta$ and assume
that $\nu$ and $\delta$ are small enough to satisfy
\begin{align}
\label{eq:assumption_nu_delta_summary}
4\nu + \nu^2 + 11 \delta < 1\quad \textrm{ and }\quad \delta  < \frac{1}{22}\left(1-\frac{3\nu}{1+\nu}\right)^2\\
\end{align}
For each $i \in [K]$ there exists a local maximizer $u_i^*$ of $\Phi_{\widehat \CW}$ with $\Phi_{\widehat \CW}(u_i^*) \geq 1-\delta$ within the cap

\begin{align*}
U_i := \left\{u \in \bbS^{D-1}: \langle u, w_i\rangle \geq \sqrt{(1-3\delta)\frac{1-\nu}{1+\nu}}\right\}.
\end{align*}
Furthermore, for any constrained local maximizer $u \in \bbS^{D-1}$ of $\Phi_{\widehat \CW}$
with $\Phi_{\widehat \CW}(u) > 7\frac{1+\nu}{1-\nu}\delta$ and basis expansion
$P_{\widehat\CW}(u\otimes u) = \sum_{i=1}^{K}\sigma_i \hat P_{\widehat \CW}(w_i \otimes w_i)$
ordered according to $\sigma_1 \geq \ldots \geq \sigma_K$, we have
\begin{equation}
\label{eq:summary_approximation_result}
\min_{s \in \{-1,1\}}\N{u-sw_1}_2 \leq \frac{\sqrt{2\nu\sum_{i=2}^{K}\sigma_i^2} + 2\delta}{(1-\nu) (1 - 6\frac{\nu}{1+\nu} - 18\delta) - \sqrt{6\frac{\nu}{1+\nu} + 18\delta} - 2\delta}.
\end{equation}
\end{thm}
\begin{proof}
The theorem is a straight-forward consequence of Corollary \ref{cor:nearly_rank_one_mat}, Theorem \ref{thm:stability_local_extrema_vector_frame},
and Proposition \ref{prop:existence_of_local_maximizer} below. Namely, Proposition \ref{prop:existence_of_local_maximizer} implies
the existence of a local maximizer $u_i^*$ in $U_i$ and thus the first part of the statement.
By Theorem \ref{thm:stability_local_extrema_vector_frame} we have
$\Phi_{\widehat \CW}(u) \geq 1 - 6\frac{\nu}{1+\nu} -18\delta$, and inserting this into Corollary \ref{cor:nearly_rank_one_mat},
the second part follows.
\end{proof}

\subsection{Preliminaries: Optimality conditions and a useful reformulation of $\Phi_{\CW}$}
\label{subsec:opt_grad_hess_functional}
\begin{lem}[Gradient and Hessian of $\Phi_{\widetilde \CW}$]
\label{lem:gradient_and_hessian}
Let $\widetilde \CW \subseteq \textrm{Sym}(\bbR^{D\times D})$ and take $u \in \bbS^{D-1}$.
The gradient and Hessian of $\Phi_{\widetilde\CW} : \bbR^D \rightarrow \bbR$ satisfy
\begin{align}
\label{eq:gradient_phi}
\nabla \Phi_{\widetilde\CW}(u) &= 4 P_{\widetilde\CW}(u \otimes u)u,\\
\label{eq:hessian_phi}
v^\top \nabla^2 \Phi_{\widetilde\CW}(u) v &= 8\N{P_{\widetilde\CW}(u \otimes v)}_F^2 + 4v^\top P_{\widetilde\CW}(u\otimes u)v\quad \textrm{for all }\quad v \in \bbS^{D-1}.
\end{align}
\end{lem}
\begin{proof}
To compute the gradient we first note that
$$
\partial_{u_j} \Phi_{\widetilde\CW}(u) = 2\langle P_{\widetilde\CW}(u\otimes u), \partial_{u_j} P_{\widetilde\CW}(u\otimes u)\rangle =
2\langle P_{\widetilde\CW}(u\otimes u), \partial_{u_j} (u\otimes u)\rangle
$$
Furthermore, we have $ \partial_{u_j} (u\otimes u) = e_j \otimes u + u \otimes e_j$, where $e_j$ is the
$j$-th standard basis vector. This implies the result by
$$
\partial_{u_j} \Phi_{\widetilde\CW}(u) = 2\langle P_{\widetilde\CW}(u\otimes u), e_j \otimes u + u \otimes e_j\rangle = 4e_j ^\top P_{\widetilde\CW}(u\otimes u) u.
$$
For \eqref{eq:hessian_phi} we use an arbitrary orthonormal basis $\{W_i : i \in [K]\}$ of $\widetilde\CW$
and write $\nabla \Phi_{\widetilde\CW}(u) = 4\sum_{i=1}^{N} \langle u\otimes u, W_i\rangle W_i u$.
Differentiating again with respect to $u_j$, we obtain the rows of the Hessian as
\begin{align*}
&\partial_{u_j}\left(\sum_{i=1}^{N} \langle u\otimes u, W_i\rangle W_i\right) u =
\left(\sum_{i=1}^{N} \langle u\otimes u, W_i\rangle W_i\right) e_j + \left(\sum_{i=1}^{N} \langle (u\otimes e_j + e_j \otimes u), W_i\rangle W_i\right) u\\
&\qquad
= P_{\widetilde\CW}(u\otimes u) e_j + 2\sum_{i=1}^{N} (W_i u)_j W_i u,
\end{align*}
which implies $\nabla^2 \Phi_{\widetilde\CW}(u) = 4 P_{\widetilde\CW}(u\otimes u) + 8 \sum_{i=1}^{N} W_i u \otimes (W_i u)$.
Multiplying with $v$ from the left and the right, it follows that
\begin{align*}
v^\top \nabla^2 \Phi_{\widetilde\CW}(u) v = 4 v^\top P_{\widetilde\CW}(u\otimes u)v + 8 \sum_{i=1}^{N} \langle u\otimes v, W_i\rangle^2
=4 v^\top P_{\widetilde\CW}(u\otimes u)v + 8 \N{ P_{\widetilde\CW}(u\otimes v)}_F^2.
\end{align*}
\end{proof}

\begin{lem}[Optimality conditions]
\label{lem:optimality_conditions}
Let $\widetilde \CW \subseteq \textrm{Sym}(\bbR^{D\times D})$. A vector $u \in \bbS^{D-1}$ is
a constrained stationary point of $\Phi_{\widetilde\CW}$ if and only if
\begin{align}
\label{eq:stationary_point}
P_{\widetilde\CW}(u\otimes u) u = \N{P_{\widetilde\CW}(u\otimes u)}_F^2 u
\end{align}
Furthermore, $u \in \bbS^{D-1}$ is a constrained local maximum of $\Phi_{\widetilde\CW}$ if and only if
\begin{align}
\label{eq:second_order_optimality}
\N{P_{\widetilde\CW}(u\otimes u)}_F^2 \geq 2\N{P_{\widetilde\CW}(u \otimes v)}_F^2 + v^\top P_{\widetilde\CW}(u\otimes u)v\qquad \textrm{for all } v \in \bbS^{D-1}\textrm{ with } v \perp u.
\end{align}
\end{lem}
\begin{proof}
Recall from Lemma \ref{lem:gradient_and_hessian} that we have $\nabla \Phi_{\widetilde\CW}(u) = 4 P_{\widetilde\CW}(u\otimes u)u$.
Following the definition of constrained stationary points in Remark \ref{rem:constrained_extreme_points},
$u \in \bbS^{D - 1}$ is a constrained stationary point if and only if
\begin{align*}
(\Id_D - u\otimes u) \nabla \Phi_{\widetilde\CW}(u) = 0\quad \textrm{ or }\quad  4 P_{\widetilde\CW}(u\otimes u)u - 4 \N{P_{\widetilde\CW}(u\otimes u)}_F^2 u = 0.
\end{align*}
Similarly, $u \in \bbS^{D-1}$ is a constrained local maximum according to
Remark \ref{rem:constrained_extreme_points} if and only if
\begin{align}
\label{eq:aux_second_order_conditions}
v^\top (\Id_D-u\otimes u)\left(\nabla^2 \Phi_{\widetilde\CW}(u) - \langle u, \nabla \Phi_{\widetilde\CW} (u)\rangle\Id_D\right)(\Id_D-u\otimes u)v \leq 0\quad \textrm{ for any } v \in \bbS^{D-1}.
\end{align}
Since $\Id_D-u\otimes u$ is the orthogonal projection onto $\Span{u}^\perp$
 \eqref{eq:aux_second_order_conditions} is actually equivalent to
\begin{align*}
v^\top \left(\nabla^2 \Phi_{\widetilde\CW}(u) - \langle u, \nabla \Phi_{\widetilde \CW} (u)\rangle\Id_D\right) v \leq 0\quad \textrm{ for any } v \in \bbS^{D-1} \textrm{ with } v \perp u.
\end{align*}
Using $\langle u, \nabla \Phi_{\widetilde\CW}(u)\rangle = 4\N{P_{\widetilde\CW}(u\otimes u)}_F^2$, this is equivalent to
$v^\top \nabla^2 \Phi_{\widetilde\CW}(u)v - 4\N{P_{\widetilde\CW}(u\otimes u)}_F^2 \leq 0$ or by the Hessian formula \eqref{eq:hessian_phi} in Lemma \ref{lem:gradient_and_hessian},
\begin{align*}
4\N{P_{\widetilde\CW}(u\otimes u)}_F^2 &\geq 8\N{P_{\widetilde\CW}(u \otimes v)}_F^2 + 4v^\top P_{\widetilde\CW}(u\otimes u)v.
\end{align*}
\end{proof}

\begin{rem}
\label{rem:derived_second_order_optimality}
The second-order optimality condition \eqref{eq:second_order_optimality} implies for any $v \in \bbS^{D-1}$
\begin{align}
\label{eq:second_order_optimality_exended}
\N{P_{\widetilde\CW}(u\otimes u)}_F^2 \geq 2\N{P_{\widetilde\CW}(u \otimes v)}_F^2 - 2 \N{P_{\widetilde\CW}(u\otimes u)}_F^2 \langle u, v\rangle^2 + v^\top P_{\widetilde\CW}(u\otimes u)v.
\end{align}
This can be seen by using \eqref{eq:second_order_optimality} for $\tilde v = \frac{v - \langle u,v\rangle u}{\N{v - \langle u,v\rangle u}_2}$
and reordering the terms.
\end{rem}

\noindent
The last result in this section will be used later in the proofs about properties of local maximizers.
It allows for rewriting $\Phi_{\CW}(u)$ as a quadratic form through the Gram matrix of $w_i \otimes w_i$'s.
\begin{lem}
\label{lem:constrained_quadratic_program}
Let $\{w_i : i \in [K]\} \subset \bbS^{D-1}$, assume $\{w_i \otimes w_i : i \in [K]\}$
are linearly independent, and denote $\CW:=\Span{w_i \otimes w_i : i \in [K]}$.
Denote $G_{ij} := \langle w_i, w_j\rangle^2$ as the Grammian matrix of
$\{w_i \otimes w_i : i \in [K]\}$. Then we have
\begin{align}
\label{eq:reformulate_objective}
\Phi_{\CW}(u) =  \beta_u G^{-1}\beta_u\qquad \textrm{for}\qquad (\beta_u)_i := \langle u, w_i\rangle^2.
\end{align}
\end{lem}
\begin{proof}
We first write the Frobenius norm as an optimal program
\begin{align*}
\N{P_{\CW}(u\otimes u)}_F = \max_{\substack{A \in \CW\\ \N{A}_F^2 = 1}} \langle A, P_{\CW}(u\otimes u)\rangle = \max_{\substack{A \in \CW\\ \N{A}_F^2 = 1}} u^\top A u
\end{align*}
and express $A = \sum_{i=1}^{K}\tau_i w_i \otimes w_i$ in terms of the basis coefficients $\tau_i$. Then using
\begin{align*}
\N{A}_F^2 &= \textrm{Tr}\left(\sum_{i=1}^{K}\tau_i w_i \otimes w_i \sum_{j=1}^{K}\tau_j w_j \otimes w_j\right)
= \sum_{i=1}^{K}\sum_{j=1}^{K}\tau_i \tau_j G_{ij} = \tau^\top G \tau,
\end{align*}
we can formulate the initial program as an optimal program over coefficients $\tau_i$ by
\begin{align*}
\N{P_{\CW}(u\otimes u)}_F = \max_{\tau^\top G \tau \leq 1}\sum_{i=1}^{K}\tau_i \langle u, w_i\rangle^2 = \max_{\tau^\top G \tau \leq 1}\langle \tau, \beta\rangle.
\end{align*}
This is a linear program with quadratic constraints. The KKT conditions
imply \eqref{eq:reformulate_objective}.
\end{proof}

\subsection{Nearly rank-one matrices approximate spanning rank-one elements}
\label{subsec:nearly_rank_one}
We begin with a general statement that shows that any unit Frobenius norm matrix with dominant eigenvalue
close to $1$ in a subspace $\CW=\Span{w_i\otimes w_i : i \in [K]}$ is close to one of the
spanning elements $\{w_i : i \in [K]\}$. By combining this with
the optimality conditions in Lemma \ref{lem:optimality_conditions} afterwards, we show
that $u \in \bbS^{D-1}$ with $\Phi_{\widehat \CW}(u) \approx 1$ is close to one of the $w_i$'s.
\begin{thm}
\label{thm:nearly_rank_one_mat}
Let $\{w_i : i \in [K]\} \subset \bbS^{D-1}$ satisfy frame condition \eqref{eq:frame_condition_1}
and denote $\CW = \Span{w_i \otimes w_i : i \in [K]}$. Let $\widehat \CW\subseteq \textrm{Sym}(\bbR^{D\times D})$ satisfy
$\N{P_{\widehat \CW} - P_{\CW}}_{2\rightarrow 2} \leq \delta < \frac{1}{2}$
and consider a basis $\{\widehat W_i := P_{\widehat \CW}(w_i \otimes w_i) : i \in [K]\}$.
For $u \in \bbS^{D-1}$ we represent $P_{\widehat \CW}(u\otimes u)$ as
$P_{\widehat \CW}(u\otimes u) = \sum_{i=1}^{K}\sigma_i \widehat W_i=\sum_{j=1}^{D}\lambda_j u_j\otimes u_j$, where
 $(\lambda_j, u_j)$ are eigenpairs, with both $\sigma_i$'s and $\lambda_j$'s sorted in descending order.
If $\lambda_1 > \frac{\lambda_2 - 2\delta}{1-\nu}$, then
\begin{align}
\label{eq:result_nearly_rank_one_mat}
\min_{s \in \{-1,1\}}\N{u_1 - sw_{1}}_2 \leq \sqrt{2}\frac{\N{\sigma_{2\ldots K}}_2\sqrt{\nu} + 2\delta}{(1-\nu)\lambda_1 - \lambda_2 - 2\delta},
\end{align}
where $\sigma_{2\ldots K} := (0, \sigma_2,\ldots,\sigma_K) \in \bbR^{K}$.
\end{thm}
\begin{proof}
Let $Z := \sum_{j=1}^{K}\sigma_j w_j \otimes w_j \in \CW$ be the unique element in $\CW$ such that $P_{\widehat \CW}(u\otimes u) = P_{\widehat \CW}(Z)$.
First notice that $\N{Z} \leq \N{P_{\CW}(Z) - P_{\widehat\CW}(Z)} + \N{P_{\widehat\CW}(Z)} \leq \delta \N{Z} + 1 $ implies
$\N{Z} \leq (1-\delta)^{-1}$. Therefore, we have
\begin{align*}
\lambda_1  &= \langle P_{\widehat \CW}(u\otimes u), u_1\otimes u_1\rangle = \langle Z, u_1\otimes u_1\rangle +
 \langle P_{\widehat \CW}(u\otimes u) - Z, u_1\otimes u_1\rangle\\
 &\leq \langle Z, u_1\otimes u_1\rangle + \N{P_{\widehat \CW}(Z) - P_{\CW}(Z)} \leq \sum_{k=1}^{K} \sigma_i \langle w_i, u_1\rangle^2 + \frac{\delta}{1-\delta}
 \leq \sigma_1 (1 + \nu) + 2\delta,
\end{align*}
which implies $\sigma_1 \geq \frac{\lambda_1 - 2\delta}{1+\nu} \geq \lambda_1 - \nu \lambda_1 - 2\delta$.
Define now $Q = \Id_D - u_1 \otimes u_1$. Choosing $s \in \{-1,1\}$ such that $s\langle w_1, u_1\rangle \geq 0$ we can bound the
squared left hand side of \eqref{eq:result_nearly_rank_one_mat} by
\begin{align*}
\N{u_1-sw_1}_2^2 = 2\left(1-s\langle w_1,u\rangle\right)\leq 2(1-\langle w_1, u\rangle^2) = 2\N{Qw_1}_2^2 = 2\N{Q (w_1\otimes w_1)}_F^2
\end{align*}
Denote $W_1 :  = w_1\otimes w_1$ and consider the auxiliary matrix $\sigma_1 W_1$. We can view $W_1$ as the orthogonal projection
onto the space spanned by eigenvectors of $\sigma_1 W_1$ associated to eigenvalues in $(\infty,\sigma_1]$.
Therefore, using the Davis-Kahan theorem in the form of \cite[Theorem 7.3.1]{bhatia2013matrix}, we obtain
\begin{align*}
\N{u_1-s w_1}_2 \leq \sqrt{2}\N{QW_1}_F \leq \sqrt{2}\frac{\N{\left(\sigma_1 W_1 - P_{\widehat \CW}(u\otimes u)\right)W_1}_F}{\sigma_1 - \lambda_2}
\end{align*}
To further bound the numerator, we first use the decomposition
\begin{align*}
\N{(\sigma_1 W_1 - P_{\widehat \CW}(u\otimes u))W_1}_F &\leq \N{(\sigma_1 W_1 - Z)W_1}_F +  \N{(Z - P_{\widehat \CW}(Z))W_1}_F \\
&\leq \N{(\sigma_1 W_1 - Z)W_1}_F  + \N{Z - P_{\widehat \CW}(Z)}_2\N{W_1}_F\\
&\leq \N{(\sigma_1 W_1 - Z)W_1}_F  + \frac{\delta}{1-\delta} \leq \N{(\sigma_1 W_1 - Z)W_1}_F + 2\delta,
\end{align*}
and then bound the first term by using the frame property \eqref{eq:frame_condition_1}
\begin{align*}
\N{(\sigma_1 W_1 - Z)W_1}_F = \N{\sum_{i=2}^{K}\sigma_i \langle w_i, w_1\rangle w_i\otimes w_1}_F
\leq \sum_{i=2}^{K}\SN{\sigma_i} \SN{\langle w_i, w_1\rangle} \leq \N{\sigma_{2\ldots K}}_2 \sqrt{\nu}.
\end{align*}
Combining the previous three estimates with $\sigma_1 \geq \lambda_1 - \nu \lambda_1 -2 \delta$, we obtain
\begin{align*}
\N{s w_1 - u_1}_2 \leq \sqrt{2}\frac{\N{\left(\sigma_1 W_1 - P_{\widehat \CW}(u\otimes u)\right)W_1}_F}{\sigma_1 - \lambda_2} \leq
\sqrt{2}\frac{\N{\sigma_{2\ldots K}}_2\sqrt{\nu} +2\delta}{(1-\nu)\lambda_1 - \lambda_2 - 2\delta}
\end{align*}
\end{proof}

\begin{cor}
\label{cor:nearly_rank_one_mat}
Assume the setting of Theorem \ref{thm:nearly_rank_one_mat} and let $u \in \bbS^{D-1}$ be a constrained
local maximum of $\Phi_{\widehat \CW}$. We have
\begin{align}
\label{eq:result_nearly_rank_one_mat2}
\min_{s \in \{-1,1\}}\N{u - sw_{1}}_2 \leq \sqrt{2}\frac{\N{\sigma_{2\ldots K}}_2\sqrt{\nu} + 2\delta}{(1-\nu) \Phi_{\widehat \CW}(u) - \sqrt{\Phi_{\widehat \CW}(u)(1 - \Phi_{\widehat \CW}(u))} - 2\delta}
\end{align}
whenever the denominator is positive.
\end{cor}
\begin{proof}
Recall from Lemma \ref{lem:optimality_conditions}
that the optimality conditions of $\Phi_{\widehat \CW}(u)$ imply
$u$ is the most dominant eigenvector of $P_{\widehat \CW}(u\otimes u)$ corresponding to
eigenvalue $\lambda_1 = \textN{P_{\widehat \CW}(u\otimes u)}_F^2 = \Phi_{\widehat \CW}(u)$. Furthermore,
the second largest eigenvalue is bounded by
\begin{align*}
\lambda_2 \leq \sqrt{\sum_{i=2}^{D}\lambda_i^2} = \sqrt{\N{P_{\widehat \CW}(u\otimes u)}_F^2 - \lambda_1^2} =  \sqrt{\Phi_{\widehat \CW}(u)(1 - \Phi_{\widehat \CW}(u))}.
\end{align*}
\end{proof}

\subsection{Properties and existence of local maximizers}
\label{subsec:properties_of_local_maximizers}
We now turn to the analysis of constrained local maximizers of the functional $\Phi_{\widehat \CW}$.
Our main result in this section is Theorem \ref{thm:stability_local_extrema_vector_frame} below, which shows
that any constrained local maximizer in the level set $\{u \in \bbS^{D-1} : \Phi_{\widehat \CW}(u) > C_\nu\delta\}$
results in function values $\Phi_{\widehat \CW}(u) \approx 1$. Additionally, we prove the existence
of a local maximizer near any of the spanning elements $w_i$ in Proposition \ref{prop:existence_of_local_maximizer},
which is nontrivial if we only have access to the noisy subspace projection $P_{\widehat\CW}$,
and we show the existence of spurious local maximizers with $\Phi_{\widehat \CW}(u) \leq \delta^2$
in Lemma \ref{lem:perturbations_spurious_local_max}.

\begin{thm}
\label{thm:stability_local_extrema_vector_frame}
Let $\{w_i : i\in [K]\}\subset \bbS^{D-1}$ satisfy \eqref{eq:frame_condition_1} and denote $\CW := \{w_i \otimes w_i : i \in [K]\}$.
Let $\widehat \CW\subseteq \textrm{Sym}(\bbR^{D\times D})$ satisfy $\textN{P_{\widehat \CW} - P_{\CW}}_{2\rightarrow 2} < \delta$ and assume
that $\nu$ and $\delta$ are small enough to satisfy
\begin{align}
\label{eq:assumption_nu_delta}
\frac{3\nu}{1+\nu} + 11 \delta < 1\quad \textrm{ and }\quad \delta  < \frac{1}{22}\left(1-\frac{3\nu}{1+\nu}\right)^2
\end{align}
For any constrained local maximizer $u \in \bbS^{D-1}$ of $\Phi_{\widehat \CW}$
we have
\begin{equation}
\label{eq:result_thm_stability_vector_frame}
\Phi_{\widehat \CW}(u) \leq 9 \frac{1+\nu}{1-\nu}\delta \qquad \textrm{ or }\qquad \Phi_{\widehat \CW}(u)\geq 1 - 6\frac{\nu}{1+\nu} - 18\delta.
\end{equation}
\end{thm}
\begin{proof}
The proof is more involved and given at the end of this section for improved readibility.
\end{proof}

\begin{prop}
\label{prop:existence_of_local_maximizer}
Let $\{w_i : i \in [K]\}\subset \bbS^{D-1}$ satisfy \eqref{eq:frame_condition_1} and denote $\CW := \{w_i \otimes w_i : i \in [K]\}$. Let $\widehat \CW\subseteq \textrm{Sym}(\bbR^{D\times D})$
be a subspace with $\N{P_{\widehat \CW} - P_{\CW}}_{2\rightarrow 2} \leq \delta$, and assume $\nu$ and $\delta$ satisfy
$$
4\nu+\nu^2 + 6\delta\leq 1.
$$
For each $i \in [K]$ there exists a local maximizer $u_i^*$ of $\Phi_{\widehat \CW}$ with $\Phi_{\widehat \CW}(u_i^*) \geq 1-\delta$ within the cap
$$
U_i := \left\{u \in \bbS^{D-1}: \langle u, w_i\rangle \geq \sqrt{(1-3\delta)\frac{1-\nu}{1+\nu}}\right\}.
$$
\end{prop}
\begin{proof}
Let $i \in [K]$ arbitrary. We can prove the statement essentially by showing the inequality
$\textN{P_{\widehat \CW}(u\otimes u)}_F^2 < \textN{P_{\widehat\CW}(w_i\otimes w_i)}_F^2$ for all $u \in \partial U_i$ because the compactness
of $U_i$ implies the existence of a global maximizer on $U_i$, and with
$\textN{P_{\widehat \CW}(u\otimes u)}_F^2 < \N{P_{\CW}(w_i\otimes w_i)}_F^2$ for all $u \in \partial U_i$,
the global maximizer is contained in $U_i\setminus \partial U_i$. Hence, it must be
a local maximizer of $u\mapsto \textN{P_{\widehat \CW}(u\otimes u)}_F^2$ over $\bbS^{D-1}$.

\noindent
First note that the perturbed objective at $w_i$ is bounded from below by
\begin{align*}
\N{P_{\widehat \CW}(w_i \otimes w_i)}_F^2 &= \N{P_{\CW}(w_i \otimes w_i)}_F^2 - \N{P_{\CW}(w_i \otimes w_i)}_F^2 - \N{P_{\widehat \CW}(w_i \otimes w_i)}_F^2\\
&= 1 - \SN{\N{P_{\CW}(w_i \otimes w_i)}_F^2 + \N{P_{\widehat \CW}(w_i \otimes w_i)}_F^2}\\
&=1 - \SN{\langle w_i \otimes w_i, (P_{\CW}-P_{\widehat \CW})(w_i \otimes w_i)\rangle} \geq 1-\delta.
\end{align*}
In the last equation we used the self-adjointness of the orthoprojector and in the inequality we applied Cauchy-Schwarz inequality.
Next we establish an upper bound for $\textN{P_{\widehat \CW}(u\otimes u)}_F^2$ for $u \in U_i$.
We note that with the Grammian matrix $G_{ij} := \langle w_i, w_j\rangle^2$ of $\{w_i \otimes w_i : i \in [K]\}$ and the vectors $\beta_i := \langle u, w_i\rangle^2$,
the unperturbed objective can be written as $\textN{P_{\CW}(u\otimes u)}_F^2 = \beta^\top G^{-1}\beta$ according to
Lemma \ref{lem:constrained_quadratic_program}. Furthermore, Gershgorin's circle Theorem implies for any eigenvalue $\sigma$ of $G$ that there exists
$j \in [K]$ such that
\begin{align*}
\SN{1-\sigma} = \SN{G_{jj}-\sigma} \leq \sum_{\ell \neq j}^{K}\langle w_\ell, w_j\rangle^2 \leq \nu.
\end{align*}
Therefore, $\N{G^{-1}}_2 \leq (1-\nu)^{-1}$ and the noisefree objective can be bounded by
$$
\N{P_{\CW}(u\otimes u)}_F^2 = \beta^\top G^{-1}\beta \leq \N{G^{-1}}_2 \N{\beta}_2^2 \leq \frac{1}{1-\nu}\N{\beta}_2^2
\leq \frac{1}{1-\nu}\N{\beta}_{\infty} \N{\beta}_1 \leq \frac{1+\nu}{1-\nu}\N{\beta}_{\infty}.
$$
In the last inequality we used frame condition \eqref{eq:frame_condition_1}, which implies $\N{\beta}_1 = \sum_{i=1}^{K}\langle u,w_i\rangle^2 \leq 1+\nu$.
To bound $\N{\beta}_{\infty}$ note first that
\begin{align*}
\max_{j\neq i}\beta_j \leq \sum_{j\neq i}\beta_j
= \sum_{j=1}^{K}\langle u, w_j\rangle^2  - \langle u, w_i\rangle^2 \leq 1+\nu - (1-3\delta)\frac{1-\nu}{1+\nu}
\leq \frac{3\nu  + \nu^2+ 3\delta}{1+\nu},
\end{align*}
which implies with the assumption $4\nu+\nu^2 + 6\delta \leq 1$
\begin{align*}
\beta_i = \langle u, w_i\rangle^2 \geq (1-3\delta)\frac{1-\nu}{1+\nu} \geq \frac{1-3\delta-\nu}{1  +\nu} \geq
\frac{1-6\delta-\nu+3\delta}{1  +\nu} \geq \max_{j\neq i}\beta_j.
\end{align*}
Hence, $\N{\beta}_{\infty} = \langle u, w_i\rangle^2$. It follows that for any $u \in U_i$ we get
\begin{align*}
\N{P_{\widehat \CW}(u \otimes u)}_F^2 &\leq \N{P_{\CW}(u \otimes u)}_F^2  + \SN{\N{P_{\widehat \CW}(u \otimes u)}_F^2 -\N{P_{\CW}(u \otimes u)}_F^2 }\\
&\leq \frac{1+\nu}{1-\nu}\N{\beta}_{\infty} + \delta \leq \frac{1+\nu}{1-\nu}\langle u, w_i\rangle^2 + \delta .
\end{align*}
Comparing $\N{P_{\widehat \CW}(w_i \otimes w_i)}_F^2$ and $\N{P_{\widehat \CW}(u \otimes u)}_F^2$ for $u \in \partial U_i$, where
$\langle u, w_i\rangle^2 = \frac{1-\nu}{1+\nu}(1-3\delta)$, yields
\begin{align*}
\N{P_{\widehat \CW}(u \otimes u)}_F^2 \leq  \frac{1 + \nu}{1-\nu}\langle u, w_i\rangle^2  + \delta = 1-2\delta < 1-\delta\leq \N{P_{\widehat \CW}(w_i \otimes w_i)}_F^2.
\end{align*}
\end{proof}

\noindent
Before proving Theorem \ref{thm:stability_local_extrema_vector_frame}, we also show
that spurious local maximizers in the level set $\{u \in \bbS^{D-1} : \Phi_{\widehat \CW}(u) \leq \delta^2\}$
can not be avoided under our deterministic noise model. Fortunately however, the objective
value $\Phi_{\widehat \CW}(u)$ acts as a certificate for whether
we found a spurious local maximizer or a vector $u$ that is close to one of the
spanning elements of $\CW$, and thus we can discard spurious solution in practice by checking
the objective value. We also add that iteration \eqref{eq:iteration_in_theorem} seems to avoid
spurious local maximizers in practice when initializing with $u_0 \sim \Uni{\bbS^{D-1}}$.

\begin{lem}[Perturbations induce spurious local maximizer]
\label{lem:perturbations_spurious_local_max}
Let $w \in \bbS^{D-1}$ and  $\CW = \Span{w \otimes w}$. There exists a subspace $\widehat \CW\subset \textrm{Sym}(\bbR^{D\times D})$ with
$\N{P_{\CW} - P_{\widehat \CW}}_{2\rightarrow 2} \leq 2\sqrt{\delta}$ so that
$\Phi_{\widehat \CW}$ has a constrained local maximizer with objective value  $\delta$.
\end{lem}
\begin{proof}
Choose $u \in \bbS^{D-1}$ with $u \perp w$ and define $M = \sqrt{1-\delta} w \otimes w - \sqrt{\delta} u \otimes u$
with the corresponding subspace $\widehat \CW := \Span{M}$. Taking arbitrary $A \in \bbR^{D\times D}$ with
$\N{A}_2 = 1$, the subspace perturbation between $\CW$ and $\widehat \CW$ can be upper bounded by
\begin{align*}
\N{P_{\CW}(A)-P_{\widehat \CW}(A)} &= \N{(1-\sqrt{1-\delta})\langle w\otimes w, A\rangle w\otimes w - \sqrt{\delta}\langle u\otimes u, A\rangle u\otimes u}\\
&\leq 1-\sqrt{1-\delta} + \sqrt{\delta} \leq 2\sqrt{\delta}.
\end{align*}
We now check that $u$ is a local maximizer of $\Phi_{\widehat \CW}$. The corresponding matrix appearing in $\Phi_{\widehat{\mathcal{W}}}(u)$ is
\begin{align*}
P_{\widehat \CW}(u\otimes u) = \langle M, u\otimes u\rangle M = -\sqrt{\delta} M = - \sqrt{\delta}\sqrt{1-\delta}w\otimes w + \delta u \otimes u,
\end{align*}
which shows that $u$ is an eigenvector of $P_{\widehat \CW}(u\otimes u)$ to eigenvalue $\delta = \N{P_{\widehat \CW}(u\otimes u)}_F^2$.
In other words, $u$ satisfies \eqref{eq:stationary_point} and is thus a stationary point.
Taking now any $q \perp u$, we have
\begin{align*}
2\N{P_{\widehat \CW}(u\otimes q)}_F^2 + q^\top P_{\widehat \CW}(u\otimes u)q
&=2 \langle M, u\otimes q\rangle^2 -\sqrt{\delta}\sqrt{1-\delta}\langle w, q\rangle^2 =-\sqrt{\delta}\sqrt{1-\delta}\langle w, q\rangle^2 < \delta,
\end{align*}
which shows that $u$ also satisfies constrained second-order optimality as in \eqref{eq:second_order_optimality}.
Hence, $u$ is a local maximizer with $\Phi_{\widehat \CW}(u) = \N{P_{\widehat{\CW}}(u \otimes u)}_F^2 = \delta$.
\end{proof}

\paragraph{Proof of Theorem \ref{thm:stability_local_extrema_vector_frame}}
Let us first show the following auxiliary result.

\begin{lem}
\label{lem:technical_relation}
Assume the settings of Theorem \ref{thm:nearly_rank_one_mat} and Theorem \ref{thm:stability_local_extrema_vector_frame}. Let $P_{\widehat\CW}(u\otimes u) = \sum_{i=1}^{K}\sigma_i
\widehat W_i$ for some $u \in \bbS^{D-1}$ and denote $\lambda_1 := \lambda_1(P_{\widehat\CW}(u\otimes u))$.
Furthermore, assume $1-\delta > \lambda_1 > 3\frac{1+\nu}{1-\nu}\delta$. For any index $j$ with
$\SN{\sigma_j} = \N{\sigma}_{\infty}$ we have $\sigma_j \geq 0$ and
\begin{align}
\label{eq:bound_fraction_inf_dot_1}
\sigma_j \leq \frac{1}{1-\nu}\langle u, w_j\rangle^2 + \frac{2}{1-\nu}\delta.
\end{align}
\end{lem}
\begin{proof} Denote $M=P_{\widehat\CW}(u\otimes u)=\sum_{i=1}^{K}\sigma_i
\widehat W_i$ and
define $Z = \sum_{i=1}^{K}\sigma_i w_i \otimes w_i$ as the unique matrix in $\CW$ with $P_{\widehat \CW}(Z)=M$. (Recall that $\widehat W_i = P_{\widehat\CW}(w_i\otimes w_i)$.)
We first note that the statement is trivial if $\N{\sigma}_1 = 0$ since this implies
$\sigma_i = 0$ for all $i \in [K]$ and thus $M = 0$. So without loss of generality we may
assume $\N{\sigma}_{\infty} \neq 0$ in the following. We will first show
$\sigma_j > 0$ by contradiction, so let us assume $\sigma_j < 0$.
Using the frame condition \eqref{eq:frame_condition_1} and $\{w_i : i\in [K]\} \subset \bbS^{D-1}$
to show the estimate
\begin{align}
\label{eq:bound_dot_product_1}
\SN{\langle Z, w_j \otimes w_j\rangle - \sigma_j} =
\SN{\sum_{i\neq j}\sigma_i \langle w_i, w_j \rangle^2}\leq \N{\sigma}_{\infty}
\sum_{i\neq j} \langle w_i, w_j \rangle^2 \leq \N{\sigma}_{\infty} \nu,
\end{align}
and combining this with $\N{Z}_2 \leq (1-\delta)^{-1}\lambda_1(M) < 1$,
we can bound $\N{\sigma}_{\infty}$ by
\begin{equation}
\begin{aligned}
\label{eq:aux_bound_1}
\N{\sigma}_{\infty} &= \SN{\sigma_j} \leq \SN{\sigma_j - \langle u, w_j\rangle^2}\\
&\leq \SN{\sigma_j - \langle Z, w_j \otimes w_j\rangle} +
\SN{\langle Z, w_j \otimes w_j\rangle - \langle P_{\widehat\CW}(Z), w_j \otimes w_j\rangle}\\
&\qquad\qquad+ \SN{\langle P_{\widehat\CW}(u\otimes u), w_j \otimes w_j\rangle -  \langle P_{\CW}(u\otimes u), w_j \otimes w_j\rangle}\\
&\leq \nu\N{\sigma}_{\infty} + \N{Z}_2 \delta + \delta \leq \nu \N{\sigma}_{\infty} + 2\delta.
\end{aligned}
\end{equation}
Hence, $\N{\sigma}_{\infty}\leq \frac{2}{1-\nu}\delta$, which further implies
a bound on $\lambda_1$ by
\begin{align*}
\lambda_1 &= \max_{\N{v}_2 =1} v^\top P_{\widehat\CW}(u\otimes u)v = \max_{\N{v}_2 = 1}\sum_{i=1}^{K}\sigma_i \langle \widehat W_i, v \otimes v\rangle\\
&\leq \max_{\N{v}_2 = 1}\sum_{i=1}^{K} \sigma_i\langle w_i, v\rangle^2 +
\max_{\N{v}_2 = 1}\sum_{i=1}^{K}\sigma_i \langle \widehat W_i - w_i \otimes w_i, v\otimes v\rangle \\
&\leq \N{\sigma}_{\infty}(1+\nu) + \max_{\N{v}_2 = 1} \langle P_{\widehat \CW}(Z) - Z, v \otimes v\rangle \leq 2\frac{1+\nu}{1-\nu}\delta + \delta \leq 3\frac{1+\nu}{1-\nu}\delta.
\end{align*}
This contradicts the assumption of the statement and therefore $\sigma_j > 0$.
For the estimate   \eqref{eq:bound_fraction_inf_dot_1} we can reuse, from Equation \eqref{eq:aux_bound_1},
\begin{align*}
\SN{\sigma_j - \left\langle u, w_j\right\rangle^2} \leq  \N{\sigma}_{\infty} \nu + 2\delta = \sigma_j\nu + 2\delta,
\end{align*}
which implies
\begin{align*}
  (1-\nu)\sigma_j \leq \left\langle u, w_j\right\rangle^2 + 2\delta.
\end{align*}
\end{proof}

\paragraph{Proof of Theorem \ref{thm:stability_local_extrema_vector_frame}}
Let $\hat W_i := P_{\widehat \CW}(w_i\otimes w_i)$ and denote by $P_{\widehat \CW}(u \otimes u) = \sum_{i=1}^{K} \sigma_{i}\hat W_i$
the basis expansion of $P_{\widehat \CW}$ ordered according to $\sigma_1\geq \ldots \geq \sigma_K$. Since $P_{\widehat \CW}$
is a bijection from $\CW$ to $\widehat  \CW$ for $\delta < 1$, the matrix $Z := \sum_{i=1}^{K}\sigma_i w_i \otimes w_i$
is the unique element in $\CW$ with $P_{\widehat \CW}(Z) = P_{\widehat \CW}(u\otimes u)$.
The proof of Theorem \ref{thm:stability_local_extrema_vector_frame} leverages
the second-order optimality condition \eqref{eq:aux_second_order_conditions}
in Remark \ref{rem:derived_second_order_optimality} for $v = w_{j^*}$, where
$j^*$ is any index with $\textN{\sigma}_{\infty} = \textSN{\sigma_j^{*}}$, to construct a quadratic inequality
for $\langle w_{j^*}, u \rangle^2$, which can only be satisfied by $\langle w_{j^*}, u \rangle^2 \approx 1$ or $\langle w_{j^*}, u \rangle^2 \approx 0$.
These inequalities, combined with the fact that $\textN{\sigma}_{\infty} = \textSN{\sigma_{j^*}}$,
imply $\lambda_1 := \lambda_1(P_{\widehat\CW }(u \otimes u)) \approx 1$ or $\lambda_1 \approx 0$.
By the optimality conditions in Lemma \ref{lem:optimality_conditions}, we have
$\Phi_{\widehat \CW_f}(u) = \lambda_1$ and thus the same bounds are transferred to
the objective value.

So let $j^*$ be any index with $\textN{\sigma}_{\infty} = \textSN{\sigma_{j^*}}$. We first notice
that the statement is trivially true whenever
\begin{align*}
\lambda_1\leq 3\frac{1+\nu}{1-\nu}\delta \qquad \textrm{or}\qquad \lambda_1 \geq 1- \delta,
\end{align*}
which implies we can concentrate on the cases $1-\delta > \lambda_1 \geq 3\frac{1+\nu}{1-\nu}\delta$ in
the following. Under this condition, Lemma \ref{lem:technical_relation} implies
$\textSN{\sigma_{j^*}} = \sigma_{j^*} = \max_{i \in [K]}\sigma_i = \sigma_1$ (note the
ordering $\sigma_1 \geq \ldots \geq \sigma_K$). Furthermore,
using Lemma \ref{lem:technical_relation} and additionally the frame condition \eqref{eq:frame_condition_1},
we can estimate $\lambda_1$ in terms of $\langle u, w_1\rangle^2$ according to
\begin{equation}
\label{eq:bound_lambda_1}
\begin{aligned}
\lambda_1 &= u^\top P_{\widehat\CW}(u\otimes u) u = \langle Z, u\otimes u\rangle + \langle P_{\widehat \CW}(Z) - Z, u\otimes u\rangle \leq
\sum_{i=1}^{K}\sigma_i \langle w_i,u\rangle^2 + \N{P_{\widehat \CW}(Z) - Z}\\
&\leq \sigma_1 (1+\nu) + \N{Z}_2\delta \leq \frac{1+\nu}{1-\nu}\langle u, w_1\rangle^2 + 3\frac{1+\nu}{1-\nu}\delta,
\end{aligned}
\end{equation}
where we used $\N{Z}_2 \leq (1-\delta)^{-1}\lambda_1 < 1$ and $1 \leq \frac{1+\nu}{1-\nu}$ in the last inequality.
Using the perturbation estimates
\begin{align}
\label{eq:easy_perturbation_estimates}
\SN{\N{P_{\widehat \CW}(u\otimes w_j)}_F^2 - \N{P_{\CW}(u\otimes w_j)}_F^2} &= \SN{\langle P_{\widehat \CW}(u\otimes w_j) - P_{\CW}(u\otimes w_j), u\otimes w_j\rangle} \leq \delta,\\
\SN{w_j^\top P_{\widehat \CW}(u\otimes u)w_j - \langle w_j, u\rangle^2} &=\SN{\langle P_{\widehat \CW}(u\otimes u)-P_{\CW}(u\otimes u), w_j\otimes w_j\rangle}\leq \delta,
\end{align}
which hold for any $j \in [K]$, the optimality condition \eqref{eq:aux_second_order_conditions} with $v = w_1$ implies
\begin{align*}
\lambda_1 &\geq 2\N{P_{\widehat \CW}(u\otimes w_1)}_F^2 - 2\lambda_1 \langle w_1, u\rangle^2 + w_1^\top P_{\widehat \CW}(u \otimes u)w_1\\
&\geq 2\N{P_{\CW}(u\otimes w_1)}_F^2- 2\lambda_1 \langle w_1, u\rangle^2 + \langle w_1, u\rangle^2 - 2\delta.
\end{align*}
Furthermore by $\N{P_{\CW}(u\otimes w_1)}_F^2 \geq \N{P_{\CW}(u\otimes w_1)}^2 \geq \langle u, w_1\rangle^2$
and the bound \eqref{eq:bound_lambda_1} for  $\lambda_1$, this becomes
\begin{align*}
\frac{1+\nu}{1-\nu}\langle u, w_1\rangle^2 &\geq 3 \langle u, w_1\rangle^2 - 2 \frac{1+\nu}{1-\nu}\langle u, w_1\rangle^4 - 6  \frac{1+\nu}{1-\nu}\langle u, w_1\rangle^2 \delta - 3\frac{1+\nu}{1-\nu}\delta- 2\delta.
\end{align*}
After dividing by $\frac{1+\nu}{1-\nu}$ and simplifying the terms we obtain
\begin{align*}
\langle u, w_1\rangle^2  &\geq 3\frac{1-\nu}{1+\nu}\langle u, w_1\rangle^2 - 2\langle u, w_1\rangle^4 - 11\delta,\\
\textrm{hence}\qquad 0 &\geq \left(1-\frac{3\nu}{1+\nu}\right)\langle u, w_1\rangle^2 - \langle u, w_1\rangle^4 - 11/2\delta\\
&= (1-c_1)\langle u, w_1\rangle^2  - \langle u, w_1\rangle^4 - c_2,
\end{align*}
with constants $c_1 := \frac{3\nu}{1+\nu}$ and $c_2 := 11/2\delta$.
This quadratic inequality for $\langle u, w_1\rangle^2$ has the solutions
\begin{align*}
\langle u, w_1\rangle^2 \leq c_2 = 11/2\delta,\quad \textrm{ and }\quad \langle u, w_1\rangle^2 \geq 1 - c_1 - c_2 - (c_1 + 2c_2)^2 \geq 1-2c_1 - 3c_2,
\end{align*}
provided that $\delta < \frac{1}{22}(1-c_1)^2$ and $c_1 + 2c_2 < 1$ as implied by the condition \eqref{eq:assumption_nu_delta} in
the statement. In the first case, where $\langle u, w_1\rangle^2  \leq 11/2\delta$,
the estimate for $\lambda_1$ in \eqref{eq:bound_lambda_1} gives
$$
\lambda_1 \leq \frac{1+\nu}{1-\nu}\langle u, w_1\rangle^2 + 3\frac{1+\nu}{1-\nu}\delta \leq 11/2 \frac{1+\nu}{1-\nu}\delta + 3\frac{1+\nu}{1-\nu}\delta \leq 9\frac{1+\nu}{1-\nu}\delta.
$$
On the other hand, the second case implies $\langle u, w_1\rangle^2 \geq 1 - 2c_1 - 3c_2$ and therefore
\begin{align*}
\lambda_1 &\geq w_1^\top P_{\widehat\CW}(u \otimes u)w_1 = w_1^\top P_{\CW}(u \otimes u)w_1 + w_1^\top\left(P_{\widehat \CW}(u \otimes u) - P_{\CW}(u \otimes u)\right)w_1\\
&\geq \langle u, w_1\rangle^2 - \delta \geq 1 - 2c_1 - 3c_2 - \delta \geq 1 - 6\frac{\nu}{1+\nu} - 35/2\delta.
\end{align*}
\qed

\section{Conclusion and outlook}
\label{sec:conclusion}

We provided a novel reparametrization of feedforward deep neural networks in terms of so-called entangled certain weights, which
uniquely identify the network apart from a few undetermined shift and scaling parameters.
The entangled weights are formed by linear mixtures of weights of successive layers and they are natural components of rank-1 symmetric decompositions of Hessians of the network.
By sampling approximate Hessians in a concentrated manner, it is possible to build the subspace spanned by $2$-tensors of entangled weights and to identify the entangled weights by means of the subspace power method. We provided proofs of stable recovery of entangled weights given $\mathcal O(D^2 \times m)$ input-output samples of the network, where $D$ is the input dimension and $m$ is the total number of the neurons of the network. Empirically, our approach works well for generic networks whose total number of neurons $m$ is less than $D \times m_L$, where $m_L$ is the number of outputs. In order to give a context to our results, we showed how the stable recovery of entangled weights allows  near exact identification of generic networks of pyramidal shape. We provided numerical demonstrations of exact identification for networks with up to $L=5$ hidden layers. However, this is by no means a limitation as recovering larger and deeper networks is exclusively matter of additional  computational efforts. To our knowledge there are no results of the kind in the literature. In our algorithmic pipeline for full and stable recovery of networks, some steps are remaining to be mathematically proven and further investigations are needed:
\begin{itemize}
\item  With Algorithms  \ref{alg:clustering} - \ref{alg:detect_last_layer} we provided the assignment of entangled weights to respective layers for networks of depth $L \leq 3$. These procedures need to have knowledge of the number of entangled weights per layer. The identification of the network architecture from the knowledge of the entangled weights and their assignment to layers remain open problems for deeper networks.
\item The use of gradient descent for the recovery of scaling and shift parameters by   minimization of \eqref{eq:new_GD_recover_network_functional} proves to be a robust and reliable method. We did not present in this paper theoretical guarantees of this last empirical risk minimization. Nevertheless its empirical success can be explained - at least locally - by (nested) linearizations of the mean-squared error \eqref{eq:new_GD_recover_network_functional} around the correct parameters and by showing that the linearization is actually uniquely and stably solvable for some probabilistic models of the parameters, see \cite{nguyen2020global,oymak2019moderate,DBLP:journals/corr/SoltanolkotabiJ17} for related techniques. We postpone the detailed analysis to a follow up paper.
\end{itemize}
Learned deep neural networks generalize very well despite being trained with a number of training samples that is significantly lower than the number of parameters. This surprising phenomenon goes against traditional wisdom, which attributes overparametrization with overfitting and thus poor generalization. Evidence suggests that the implicit bias of optimization methods is  in fact towards low-complexity networks (e.g., low-rank weights). The proposed robust identification pipeline can leverage the intrinsic low complexity of trained nonlinear networks to design novel algorithms for their compression. Moreover, it may be exploited as a new training method, which directly pursues and outputs networks of minimal parameters, rather than training overparametrized networks and then operating a compression.

\appendix

\bibliographystyle{abbrv}


\end{document}

%% file: Macros.tex
\newcommand{\vertiii}[1]{{\left\vert\kern-0.25ex\left\vert\kern-0.25ex\left\vert #1
    \right\vert\kern-0.25ex\right\vert\kern-0.25ex\right\vert}}

\newcommand{\R}{\mathbb{R}}

\newcommand{\supp}{{\rm supp}}

\newcommand{\argmin}{\arg \min}

\newcommand{\opvec}{\operatorname{vec}}

\newcommand{\opdiag}{\operatorname{diag}}

\DeclarePairedDelimiterX\set[1]\lbrace\rbrace{#1}

\providecommand*{\Uni}[1]{\operatorname{Unif}({#1})}

\providecommand{\Dim}{\operatorname{dim}}            
\providecommand{\dim}{\Dim}

\providecommand*{\Span}[1]{\operatorname{Span}\left\{{#1}\right\}}     
\providecommand{\supp}{\Supp}

\providecommand{\range}{\operatorname{range}}                
\providecommand{\rank}{\operatorname{rank}}                        

\providecommand{\argmin}{\operatorname*{argmin}}  
\providecommand{\Id}{\Op{Id}}                     

\providecommand{\diag}{\operatorname{diag}}

\providecommand{\CD}{{\cal D}}

\providecommand{\CN}{{\cal N}}
\providecommand{\CO}{{\cal O}}

\providecommand{\CV}{{\cal V}}
\providecommand{\CW}{{\cal W}}

\providecommand{\bbE}{\mathbb{E}}

\providecommand{\bbN}{\mathbb{N}}

\providecommand{\bbP}{\mathbb{P}}

\providecommand{\bbR}{\mathbb{R}}
\providecommand{\bbS}{\mathbb{S}}

\providecommand*{\N}[1]{\left\|{#1}\right\|} 
\providecommand*{\textN}[1]{\|{#1}\|} 

\newcommand*{\SN}[1]{\left|{#1}\right|}      
\newcommand*{\textSN}[1]{|{#1}|}      

\newcommand*{\Op}[1]{\mathsf{#1}} 




















%% file: ms.bbl
\begin{thebibliography}{10}

\bibitem{absil2009optimization}
P.-A. Absil, R.~Mahony, and R.~Sepulchre.
\newblock {\em Optimization algorithms on matrix manifolds}.
\newblock Princeton University Press, 2009.

\bibitem{Albertini93uniquenessof}
F.~Albertini, E.~D. Sontag, and V.~Maillot.
\newblock Uniqueness of weights for neural networks.
\newblock In {\em Artificial Neural Networks with Applications in Speech and
  Vision}, pages 115--125. Chapman and Hall, 1993.

\bibitem{Anandkumar2014GuaranteedNT}
A.~Anandkumar, R.~Ge, and M.~Janzamin.
\newblock Guaranteed non-orthogonal tensor decomposition via alternating
  rank-$1 $ updates.
\newblock {\em arXiv preprint arXiv:1402.5180}, 2014.

\bibitem{arora2019convergence}
S.~Arora, N.~Cohen, N.~Golowich, and W.~Hu.
\newblock A convergence analysis of gradient descent for deep linear neural
  networks.
\newblock {\em arXiv preprint arXiv:1810.02281}, 2018.

\bibitem{NIPS2019_8960}
S.~Arora, N.~Cohen, W.~Hu, and Y.~Luo.
\newblock Implicit regularization in deep matrix factorization.
\newblock In {\em Advances in Neural Information Processing Systems}, pages
  7413--7424, 2019.

\bibitem{bah2019}
B.~Bah, H.~Rauhut, U.~Terstiege, and M.~Westdickenberg.
\newblock Learning deep linear neural networks: Riemannian gradient flows and
  convergence to global minimizers.
\newblock {\em arXiv preprint arXiv:1910.05505}, 2019.

\bibitem{Berner2020}
J.~Berner, P.~Grohs, and A.~Jentzen.
\newblock Analysis of the generalization error: Empirical risk minimization
  over deep artificial neural networks overcomes the curse of dimensionality in
  the numerical approximation of black--scholes partial differential equations.
\newblock {\em SIAM Journal on Mathematics of Data Science}, 2(3):631–657,
  Jan 2020.

\bibitem{bhatia2013matrix}
R.~Bhatia.
\newblock {\em Matrix analysis}, volume 169.
\newblock Springer Science \& Business Media, 2013.

\bibitem{BLUM1992117}
A.~L. Blum and R.~L. Rivest.
\newblock Training a 3-node neural network is np-complete.
\newblock {\em Neural Networks}, 5(1):117 -- 127, 1992.

\bibitem{Blcskei2019OptimalAW}
H.~B{\"o}lcskei, P.~Grohs, G.~Kutyniok, and P.~Petersen.
\newblock Optimal approximation with sparsely connected deep neural networks.
\newblock {\em SIAM Journal on Mathematics of Data Science}, 1:8--45, 2019.

\bibitem{BUHMANN1999103}
M.~D. Buhmann and A.~Pinkus.
\newblock Identifying linear combinations of ridge functions.
\newblock {\em Advances in Applied Mathematics}, 22(1):103 -- 118, 1999.

\bibitem{10.1145/1970392.1970395}
E.~J. Cand{\`e}s, X.~Li, Y.~Ma, and J.~Wright.
\newblock Robust principal component analysis?
\newblock {\em Journal of the ACM}, 58(3):1--37, 2011.

\bibitem{CHUI1992131}
C.~K. Chui and X.~Li.
\newblock Approximation by ridge functions and neural networks with one hidden
  layer.
\newblock {\em Journal of Approximation Theory}, 70(2):131 -- 141, 1992.

\bibitem{cloninger2020relu}
A.~Cloninger and T.~Klock.
\newblock Relu nets adapt to intrinsic dimensionality beyond the target domain.
\newblock {\em arXiv preprint arXiv:2008.02545}, 2020.

\bibitem{daubechies2019}
I.~Daubechies, R.~DeVore, S.~Foucart, B.~Hanin, and G.~Petrova.
\newblock Nonlinear approximation and (deep) relu networks, 2019.

\bibitem{devore2020neural}
R.~DeVore, B.~Hanin, and G.~Petrova.
\newblock Neural network approximation.
\newblock {\em arXiv preprint arXiv:2012.14501}, 2020.

\bibitem{du2019gradient}
S.~Du, J.~Lee, H.~Li, L.~Wang, and X.~Zhai.
\newblock Gradient descent finds global minima of deep neural networks.
\newblock In {\em International Conference on Machine Learning}, pages
  1675--1685, 2019.

\bibitem{elbraechter2020dnn}
D.~Elbr{\"a}chter, P.~Grohs, A.~Jentzen, and C.~Schwab.
\newblock Dnn expression rate analysis of high-dimensional pdes: Application to
  option pricing.
\newblock {\em arXiv preprint arXiv:1809.07669}, 2020.

\bibitem{Fefferman1994ReconstructingAN}
C.~Fefferman.
\newblock Reconstructing a neural net from its output.
\newblock {\em Revista Matematica Iberoamericana}, 10:507--555, 1994.

\bibitem{fiedlerlearning}
C.~Fiedler.
\newblock Learning deep neural networks with very few samples.
\newblock Master's thesis, Technical University Munich, 2019.

\bibitem{fornasier2019robust}
M.~Fornasier, T.~Klock, and M.~Rauchensteiner.
\newblock Robust and resource efficient identification of two hidden layer
  neural networks.
\newblock {\em to appear in Constr. Approx.}, arXiv preprint arXiv:1907.00485.

\bibitem{Fornasier2012}
M.~Fornasier, K.~Schnass, and J.~Vybiral.
\newblock Learning functions of few arbitrary linear parameters in high
  dimensions.
\newblock {\em Foundations of Computational Mathematics}, 12(2):229--262, Apr.
  2012.

\bibitem{fornasier2018identification}
M.~Fornasier, J.~Vyb\'{\i}ral, and I.~Daubechies.
\newblock Identification of shallow neural networks by fewest samples.
\newblock {\em to appear in Information and Inference}, arXiv preprint
  arXiv:1804.01592.

\bibitem{gale2019state}
T.~Gale, E.~Elsen, and S.~Hooker.
\newblock The state of sparsity in deep neural networks.
\newblock {\em arXiv preprint arXiv:1902.09574}, 2019.

\bibitem{tropp_gittens}
A.~Gittens and J.~A. Tropp.
\newblock Tail bounds for all eigenvalues of a sum of random matrices.
\newblock {\em arXiv preprint arXiv:1104.4513}, 2011.

\bibitem{Goodfellow2015ExplainingAH}
I.~Goodfellow, J.~Shlens, and C.~Szegedy.
\newblock Explaining and harnessing adversarial examples.
\newblock In {\em International Conference on Learning Representations}, 2015.

\bibitem{elbrachter2020deep}
P.~Grohs, D.~Perekrestenko, D.~Elbr{\"a}chter, and H.~B{\"o}lcskei.
\newblock Deep neural network approximation theory.
\newblock {\em arXiv preprint arXiv:1901.02220}, 1, 2020.

\bibitem{DBLP:journals/corr/HannunCCCDEPSSCN14}
A.~Hannun, C.~Case, J.~Casper, B.~Catanzaro, G.~Diamos, E.~Elsen, R.~Prenger,
  S.~Satheesh, S.~Sengupta, A.~Coates, et~al.
\newblock Deep speech: Scaling up end-to-end speech recognition.
\newblock {\em arXiv preprint arXiv:1412.5567}, 2014.

\bibitem{7780459}
K.~{He}, X.~{Zhang}, S.~{Ren}, and J.~{Sun}.
\newblock Deep residual learning for image recognition.
\newblock In {\em 2016 IEEE Conference on Computer Vision and Pattern
  Recognition (CVPR)}, pages 770--778, 2016.

\bibitem{doi:10.1002/sce.37303405110}
D.~O. Hebb.
\newblock {\em The organization of behavior: a neuropsychological theory}.
\newblock John Wiley; Chapman \& Hall, 1949.

\bibitem{hinton1988learning}
G.~E. Hinton and J.~L. McClelland.
\newblock Learning representations by recirculation.
\newblock In {\em Neural information processing systems}, pages 358--366, 1988.

\bibitem{1189626}
G.-B. Huang.
\newblock Learning capability and storage capacity of two-hidden-layer
  feedforward networks.
\newblock {\em IEEE Transactions on Neural Networks}, 14(2):274--281, 2003.

\bibitem{isaac2013pleasures}
R.~Isaac.
\newblock {\em The Pleasures of Probability}.
\newblock Undergraduate Texts in Mathematics. Springer New York, 2013.

\bibitem{anandkumar15}
M.~Janzamin, H.~Sedghi, and A.~Anandkumar.
\newblock Beating the perils of non-convexity: Guaranteed training of neural
  networks using tensor methods.
\newblock {\em arXiv preprint arXiv:1506.08473}, 2015.

\bibitem{JUDD1988177}
S.~Judd.
\newblock On the complexity of loading shallow neural networks.
\newblock {\em Journal of Complexity}, 4(3):177 -- 192, 1988.

\bibitem{kileel2019subspace}
J.~Kileel and J.~M. Pereira.
\newblock Subspace power method for symmetric tensor decomposition and
  generalized pca.
\newblock {\em arXiv preprint arXiv:1912.04007}, 2019.

\bibitem{NIPS2012_4824}
A.~Krizhevsky, I.~Sutskever, and G.~E. Hinton.
\newblock Imagenet classification with deep convolutional neural networks.
\newblock In {\em Advances in Neural Information Processing Systems},
  volume~25, pages 1097--1105. Curran Associates, Inc., 2012.

\bibitem{7fa6b6a5cde14bcfbd7ab3a8f19d0d56}
Y.~Lecun.
\newblock Une procedure d'apprentissage pour reseau a seuil asymmetrique (a
  learning scheme for asymmetric threshold networks).
\newblock In {\em Proceedings of Cognitiva 85, Paris, France}, pages 599--604,
  1985.

\bibitem{Lin20}
K.-C. Lin.
\newblock {\em Nonlinear Sampling Theory and Efficient Signal Recovery}.
\newblock PhD thesis, University of Maryland, 2020.

\bibitem{MeiE7665}
S.~Mei, A.~Montanari, and P.-M. Nguyen.
\newblock A mean field view of the landscape of two-layer neural networks.
\newblock {\em Proceedings of the National Academy of Sciences},
  115(33):E7665--E7671, 2018.

\bibitem{mhaskar2020function}
H.~Mhaskar and T.~Poggio.
\newblock Function approximation by deep networks.
\newblock {\em Communications on Pure \& Applied Analysis}, 19(8), 2020.

\bibitem{mhaskar2016deep}
H.~N. Mhaskar and T.~Poggio.
\newblock Deep vs. shallow networks: An approximation theory perspective.
\newblock {\em Analysis and Applications}, 14(06):829--848, 2016.

\bibitem{nature15}
V.~Mnih, K.~Kavukcuoglu, D.~Silver, A.~A. Rusu, J.~Veness, M.~G. Bellemare,
  A.~Graves, M.~Riedmiller, A.~K. Fidjeland, G.~Ostrovski, S.~Petersen,
  C.~Beattie, A.~Sadik, I.~Antonoglou, H.~King, D.~Kumaran, D.~Wierstra,
  S.~Legg, and D.~Hassabis.
\newblock Human-level control through deep reinforcement learning.
\newblock {\em Nature}, 518(7540):529--533, 2015.

\bibitem{MondMont18}
M.~Mondelli and A.~Montanari.
\newblock On the connection between learning two-layer neural networks and
  tensor decomposition.
\newblock In {\em The 22nd International Conference on Artificial Intelligence
  and Statistics}, pages 1051--1060. PMLR, 2019.

\bibitem{moroshko2020implicit}
E.~Moroshko, S.~Gunasekar, B.~Woodworth, J.~D. Lee, N.~Srebro, and D.~Soudry.
\newblock Implicit bias in deep linear classification: Initialization scale vs
  training accuracy.
\newblock {\em arXiv preprint arXiv:2007.06738}, 2020.

\bibitem{NIPS2014_5430}
P.~Netrapalli, N.~U~N, S.~Sanghavi, A.~Anandkumar, and P.~Jain.
\newblock Non-convex robust pca.
\newblock In {\em Advances in Neural Information Processing Systems},
  volume~27, pages 1107--1115. Curran Associates, Inc., 2014.

\bibitem{neyshabur2015}
B.~Neyshabur, R.~Tomioka, and N.~Srebro.
\newblock In search of the real inductive bias: On the role of implicit
  regularization in deep learning.
\newblock In {\em International Conference on Learning Representations}, 2015.

\bibitem{nguyen2020global}
Q.~Nguyen and M.~Mondelli.
\newblock Global convergence of deep networks with one wide layer followed by
  pyramidal topology.
\newblock {\em arXiv preprint arXiv:2002.07867}, 2020.

\bibitem{pmlr-v48-oord16}
A.~V. Oord, N.~Kalchbrenner, and K.~Kavukcuoglu.
\newblock Pixel recurrent neural networks.
\newblock volume~48 of {\em Proceedings of Machine Learning Research}, pages
  1747--1756. PMLR, 2016.

\bibitem{oord2016wavenet}
A.~v.~d. Oord, S.~Dieleman, H.~Zen, K.~Simonyan, O.~Vinyals, A.~Graves,
  N.~Kalchbrenner, A.~Senior, and K.~Kavukcuoglu.
\newblock Wavenet: A generative model for raw audio.
\newblock {\em arXiv preprint arXiv:1609.03499}, 2016.

\bibitem{oymak2019moderate}
S.~Oymak and M.~Soltanolkotabi.
\newblock Towards moderate overparameterization: global convergence guarantees
  for training shallow neural networks.
\newblock {\em IEEE Journal on Selected Areas in Information Theory}, 2020.

\bibitem{PETERSEN2018296}
P.~Petersen and F.~Voigtlaender.
\newblock Optimal approximation of piecewise smooth functions using deep relu
  neural networks.
\newblock {\em Neural Networks}, 108:296 -- 330, 2018.

\bibitem{rolnick2020reverseengineering}
D.~Rolnick and K.~Kording.
\newblock Reverse-engineering deep relu networks.
\newblock In {\em International Conference on Machine Learning}, pages
  8178--8187. PMLR, 2020.

\bibitem{10.5555/104279.104293}
D.~E. Rumelhart, G.~E. Hinton, and R.~J. Williams.
\newblock {\em Learning Internal Representations by Error Propagation}, page
  318–362.
\newblock MIT Press, Cambridge, MA, USA, 1986.

\bibitem{shaham2018provable}
U.~Shaham, A.~Cloninger, and R.~R. Coifman.
\newblock Provable approximation properties for deep neural networks.
\newblock {\em Applied and Computational Harmonic Analysis}, 44(3):537--557,
  2018.

\bibitem{shevchenko2020landscape}
A.~Shevchenko and M.~Mondelli.
\newblock Landscape connectivity and dropout stability of sgd solutions for
  over-parameterized neural networks.
\newblock In {\em International Conference on Machine Learning}, pages
  8773--8784. PMLR, 2020.

\bibitem{silver2017mastering}
D.~Silver, J.~Schrittwieser, K.~Simonyan, I.~Antonoglou, A.~Huang, A.~Guez,
  T.~Hubert, L.~Baker, M.~Lai, A.~Bolton, Y.~Chen, T.~Lillicrap, F.~Hui,
  L.~Sifre, G.~van~den Driessche, T.~Graepel, and D.~Hassabis.
\newblock Mastering the game of go without human knowledge.
\newblock {\em Nature}, 550:354--, Oct. 2017.

\bibitem{DBLP:journals/corr/SoltanolkotabiJ17}
M.~Soltanolkotabi, A.~Javanmard, and J.~D. Lee.
\newblock Theoretical insights into the optimization landscape of
  over-parameterized shallow neural networks.
\newblock {\em IEEE Transactions on Information Theory}, 65(2):742--769, 2018.

\bibitem{soudry2018implicit}
D.~Soudry, E.~Hoffer, M.~S. Nacson, S.~Gunasekar, and N.~Srebro.
\newblock The implicit bias of gradient descent on separable data.
\newblock {\em The Journal of Machine Learning Research}, 19(1):2822--2878,
  2018.

\bibitem{stewart1991perturbation}
G.~W. Stewart.
\newblock Perturbation theory for the singular value decomposition.
\newblock Technical report, 1991.

\bibitem{sun2019optimization}
R.~Sun.
\newblock Optimization for deep learning: theory and algorithms.
\newblock {\em arXiv preprint arXiv:1912.08957}, 2019.

\bibitem{SUSSMANN1992589}
H.~J. Sussmann.
\newblock Uniqueness of the weights for minimal feedforward nets with a given
  input-output map.
\newblock {\em Neural Networks}, 5(4):589 -- 593, 1992.

\bibitem{42503}
C.~Szegedy, W.~Zaremba, I.~Sutskever, J.~Bruna, D.~Erhan, I.~Goodfellow, and
  R.~Fergus.
\newblock Intriguing properties of neural networks.
\newblock In {\em International Conference on Learning Representations}, 2014.

\bibitem{NIPS2017_7181}
A.~Vaswani, N.~Shazeer, N.~Parmar, J.~Uszkoreit, L.~Jones, A.~N. Gomez,
  {\L}.~Kaiser, and I.~Polosukhin.
\newblock Attention is all you need.
\newblock In {\em Advances in Neural Information Processing Systems 30}, pages
  5998--6008. Curran Associates, Inc., 2017.

\bibitem{HDP18}
R.~Vershynin.
\newblock {\em High-dimensional probability: An introduction with applications
  in data science}, volume~47.
\newblock Cambridge university press, 2018.

\bibitem{doi:10.1137/20M1314884}
R.~Vershynin.
\newblock Memory capacity of neural networks with threshold and rectified
  linear unit activations.
\newblock {\em SIAM Journal on Mathematics of Data Science}, 2(4):1004--1033,
  2020.

\bibitem{DBLP:journals/corr/abs-1906-06994}
V.~Vla{\v{c}}i{\'c} and H.~B{\"o}lcskei.
\newblock Neural network identifiability for a family of sigmoidal
  nonlinearities.
\newblock {\em arXiv preprint arXiv:1906.06994}, 2019.

\bibitem{DBLP:journals/corr/abs-2006-11727}
V.~Vla{\v{c}}i{\'c} and H.~B{\"o}lcskei.
\newblock Affine symmetries and neural network identifiability.
\newblock {\em Advances in Mathematics}, 376:107485, 2020.

\bibitem{wedin1972perturbation}
P.-{\AA}. Wedin.
\newblock Perturbation bounds in connection with singular value decomposition.
\newblock {\em BIT Numerical Mathematics}, 12(1):99--111, 1972.

\bibitem{werbos74}
P.~Werbos.
\newblock {\em Beyond Regression: New Tools for Prediction and Analysis in the
  Behavioral Sciences}.
\newblock Harvard University, 1975.

\bibitem{woodworth2020kernel}
B.~Woodworth, S.~Gunasekar, J.~D. Lee, E.~Moroshko, P.~Savarese, I.~Golan,
  D.~Soudry, and N.~Srebro.
\newblock Kernel and rich regimes in overparametrized models.
\newblock In {\em Proceedings of Thirty Third Conference on Learning Theory},
  volume 125, pages 3635--3673. PMLR, 2020.

\bibitem{NIPS2010_4005}
H.~Xu, C.~Caramanis, and S.~Sanghavi.
\newblock Robust pca via outlier pursuit.
\newblock In {\em Advances in Neural Information Processing Systems},
  volume~23, pages 2496--2504. Curran Associates, Inc., 2010.

\bibitem{ye2020good}
M.~Ye, C.~Gong, L.~Nie, D.~Zhou, A.~Klivans, and Q.~Liu.
\newblock Good subnetworks provably exist: Pruning via greedy forward
  selection.
\newblock In {\em Proceedings of the 37th International Conference on Machine
  Learning}, volume 119, pages 10820--10830. PMLR, 2020.

\bibitem{DBLP:conf/nips/YunSJ19}
C.~Yun, S.~Sra, and A.~Jadbabaie.
\newblock Small relu networks are powerful memorizers: a tight analysis of
  memorization capacity.
\newblock In {\em Advances in Neural Information Processing Systems},
  volume~32, pages 15558--15569. Curran Associates, Inc., 2019.

\bibitem{zhang2016understanding}
C.~Zhang, S.~Bengio, M.~Hardt, B.~Recht, and O.~Vinyals.
\newblock Understanding deep learning requires rethinking generalization.
\newblock In {\em International Conference on Learning Representations}, 2017.

\bibitem{zhong2017recovery}
K.~Zhong, Z.~Song, P.~Jain, P.~L. Bartlett, and I.~S. Dhillon.
\newblock Recovery guarantees for one-hidden-layer neural networks.
\newblock In {\em Proceedings of the 34th International Conference on Machine
  Learning}, volume~70, pages 4140--4149. PMLR, 2017.

\end{thebibliography}
